\documentclass[11pt]{article}

\usepackage[margin=1in]{geometry}
\usepackage{amsmath}
\usepackage{amsthm}
\usepackage{natbib}
\usepackage{hyperref}

\usepackage[utf8]{inputenc} 
\usepackage[T1]{fontenc}    
\usepackage{xcolor}         
\usepackage{url}            
\usepackage{booktabs}       
\usepackage{cellspace}
\usepackage{amsfonts}       
\usepackage{nicefrac}       
\usepackage{microtype}      
\usepackage{enumitem}
\usepackage{mathtools}
\usepackage{tikz}
\usetikzlibrary{arrows.meta,positioning,fit,shapes.geometric}
\hypersetup{hidelinks}

\newcommand{\cE}{\mathcal{E}}
\newcommand{\cF}{\mathcal{F}}

\newcommand{\cH}{\mathcal{H}}

\newcommand{\cL}{\mathcal{L}}

\newcommand{\cX}{\mathcal{X}}
\newcommand{\cY}{\mathcal{Y}}

\newcommand{\E}{\mathbb{E}}

\newcommand{\I}{\mathbb{I}}

\newcommand{\N}{\mathbb{N}}

\newcommand{\Pb}{\mathbb{P}}
\newcommand{\bbR}{\mathbb{R}}

\usepackage{bm}

\newcommand{\lrb}[1]{\left(#1\right)}
\newcommand{\brb}[1]{\bigl(#1\bigr)}
\newcommand{\Brb}[1]{\Bigl(#1\Bigr)}

\newcommand{\lsb}[1]{\left[#1\right]}
\newcommand{\bsb}[1]{\bigl[#1\bigr]}
\newcommand{\Bsb}[1]{\Bigl[#1\Bigr]}
\newcommand{\bbsb}[1]{\biggl[#1\biggr]}
\newcommand{\lcb}[1]{\left\{#1\right\}}
\newcommand{\bcb}[1]{\bigl\{#1\bigr\}}
\newcommand{\Bcb}[1]{\Bigl\{#1\Bigr\}}

\newcommand{\lce}[1]{\left\lceil#1\right\rceil}

\newcommand{\labs}[1]{\left\lvert#1\right\rvert}
\newcommand{\babs}[1]{\bigl\lvert#1\bigr\rvert}
\newcommand{\Babs}[1]{\Bigl\lvert#1\Bigr\rvert}
\newcommand{\bbabs}[1]{\biggl\lvert#1\biggr\rvert}
\newcommand{\lno}[1]{\left\lVert#1\right\rVert}

\DeclareMathOperator*{\argmax}{argmax}

\newcommand{\dif}{\mathrm{d}}

\newcommand{\eps}{\varepsilon}
\newcommand{\gft}{\mathrm{gft}}

\newcommand{\wt}{\widetilde}

\newcommand{\ceq}{\coloneqq}

\newcommand{\KL}{D_{\mathrm{KL}}}

\usepackage{todonotes}

\usepackage{comment}
\newcommand{\BlackBox}{\ensuremath{\blacksquare}}

\usepackage{amssymb}
\usepackage[ruled, noline,noend]{algorithm2e}
\usepackage{graphicx}
\usepackage{subcaption}
\usepackage{xfrac}
\usepackage{multirow}
\usepackage{pifont}

\usepackage[capitalize]{cleveref}
\usepackage{aliascnt}

\newtheorem{theorem}{Theorem}
\crefname{theorem}{theorem}{theorems}
\Crefname{theorem}{Theorem}{Theorems}

\newenvironment{proofsketch}{%
  \par\noindent{\bf Proof Sketch\ }%
}{%
  \hfill\BlackBox\\[2mm]%
}

\newaliascnt{lemma}{theorem}
\newtheorem{lemma}[lemma]{Lemma}
\aliascntresetthe{lemma}
\crefname{lemma}{lemma}{lemmas}
\Crefname{lemma}{Lemma}{Lemmas}

\newaliascnt{corollary}{theorem}

\aliascntresetthe{corollary}
\crefname{corollary}{corollary}{corollaries}
\Crefname{corollary}{Corollary}{Corollaries}

\newaliascnt{proposition}{theorem}
\newtheorem{proposition}[proposition]{Proposition}
\aliascntresetthe{proposition}
\crefname{proposition}{proposition}{propositions}
\Crefname{proposition}{Proposition}{Propositions}

\newaliascnt{fact}{theorem}

\aliascntresetthe{fact}
\crefname{fact}{fact}{facts}
\Crefname{fact}{Fact}{Facts}

\newtheorem{definition}{Definition}
\crefname{definition}{definition}{definitions}
\Crefname{definition}{Definition}{Definitions}

\newtheorem{remark}{Remark}
\crefname{remark}{remark}{remarks}
\Crefname{remark}{Remark}{Remarks}

\crefname{notation}{notation}{notations}
\Crefname{notation}{Notation}{Notations}

\newtheorem{assumption}{Assumption}
\crefname{assumption}{assumption}{assumptions}
\Crefname{assumption}{Assumption}{Assumptions}

\crefname{example}{example}{examples}
\Crefname{example}{Example}{Examples}

\newtheorem{axiom}{Axiom}
\crefname{axiom}{axiom}{axioms}
\Crefname{axiom}{Axiom}{Axioms}

\usepackage{thmtools}
\usepackage{thm-restate}
\usepackage{lastpage}

\title{Repeated Bilateral Trade: The Quest for Fairness}

\author{
Fran\c{c}ois Bachoc\\
University of Lille\\
Institut Universitaire de France (IUF)\\
Lille, France\\
\texttt{francois.bachoc@univ-lille.fr}
\and
Roberto Colomboni\\
School of Mathematics\\
University of Bristol\\
Fry Building, Woodland Road\\
Bristol BS8 1UG, United Kingdom\\
\texttt{roberto.colomboni@bristol.ac.uk}
\and
Emilie Kaufmann\\
Univ. Lille, CNRS, Inria, Centrale Lille, UMR 9189-CRIStAL\\
F-59000 Lille, France\\
\texttt{emilie.kaufmann@univ-lille.fr}
}

\date{}

\begin{document}
\maketitle

\begin{abstract}%
We study repeated bilateral trade from a fairness perspective.
At each round, a fresh seller-buyer pair arrives, and the platform posts a price before observing the traders' valuations.
Trade occurs only if both agents accept the price.
Rather than maximizing only the gain from trade, we consider platforms that seek balanced divisions of the generated surplus.
We show that natural fairness desiderata lead to a one-parameter Rawls-to-Nash family of fair-gain objectives, obtained by aggregating the seller's and buyer's net gains through nonpositive H\"older means.
Unlike the standard gain-from-trade objective and the Rawlsian fair-gain objective studied in prior work, our proposed objectives induce a new statistical structure in which expected rewards are recovered from threshold feedback through a two-dimensional singular-kernel integral identity. This leads to a nonstandard pure-exploration problem whose natural estimators are rectangular double sums with row-column dependence and singular weights.
Assuming independent i.i.d.\ seller and buyer valuation sequences with arbitrary unknown marginals, we characterize the optimal learning rates for the whole Rawls-to-Nash family of fair-gain objectives, giving matching fixed-confidence sample-complexity and regret bounds up to polylogarithmic factors.
\end{abstract}

\vspace{0.5em}
\noindent\textbf{Keywords:} bilateral trade, fairness, online learning, regret minimization, pure exploration\par



\section{Introduction}

Online bilateral trade studies how a platform can learn, over repeated interactions, to intermediate between sellers and buyers.
At each round, the platform posts a price before observing the traders' valuations and trade occurs if and only if both agents accept the price.
The learner then observes only the two threshold decisions, namely the seller's and buyer's acceptance or rejection bits.
The standard learning objective in this model is the \emph{gain from trade} \citep{cesa2021bilateralEC,cesa2024bilateralMOR},
which measures the net increase in market value: it is the sum of the seller's and buyer's realized gains when trade occurs, and zero otherwise.
This is the right quantity for measuring efficiency, but it is blind to how the surplus is divided.
Indeed, for a realized pair of valuations $s \le b$, every price $p \in [s,b]$ generates the same gain from trade, $b-s$, even though the seller gain $p-s$ and the buyer gain $b-p$ may be arbitrarily imbalanced.
Consequently, a learner that optimizes only for efficiency may learn prices that induce systematically unfair trades.

This observation motivates learning objectives that evaluate not only whether surplus is generated, but also how it is divided.
Inspired by Rawlsian fairness, \citet{bachoc2024fairNeurIPS} proposed an alternative learning objective, termed \emph{fair gain from trade} in their work, equal to the minimum of the seller's and buyer's realized gains.
By optimizing the worst-off participant in each successful trade, this objective captures the Rawlsian max-min notion of fairness.
However, this Rawlsian criterion is only one possible way to aggregate the two realized gains into a scalar fair-gain score.

In this paper, rather than selecting another fairness metric from the literature and then studying the corresponding online-learning problem, we first ask which scalar scores are structurally justified for evaluating a two-person gain vector.
Starting from basic desiderata---an equal-gain-equivalent interpretation on a common welfare scale, covariance under rescaling of gains, and zero fairness for one-takes-all splits---we show that the admissible fairness objectives are precisely the finite-order nonpositive H\"older means of the seller's and buyer's gains, forming a family parameterized by a \emph{nonpositive} real number $\lambda$.

Interestingly, this class interpolates between Rawlsian and Nash fairness and connects several standard fairness notions.
The Rawlsian criterion studied by \citet{bachoc2024fairNeurIPS} is recovered as the limiting endpoint $\lambda \to -\infty$.
The opposite endpoint, $\lambda = 0$, is the geometric mean, corresponding to Nash fairness rooted in Nash's bargaining solution \citep{nash1950bargaining} and widely used through the Nash social welfare objective \citep{caragiannis2019unreasonable,branzei2017nash}.
The harmonic mean, obtained at $\lambda = -1$, has appeared as a fairness criterion in resource-allocation problems \citep{dashti2013harmonic} and is closely related to Jain-type fairness considerations \citep{lan2010axiomatic}.
More broadly, H\"older means arise naturally in axiomatic treatments of fairness in networks \citep{lan2010axiomatic}, in fair division of goods and chores \citep{eckart2024fairness}, and as equal-gain equivalents induced by fair utility scales in the classical proportional-fairness literature on communication networks \citep{kelly1998rate,mo2002fair}.

Having identified this class of fairness candidates, we use it as the starting point for this paper.
Our goal is to study the online-learning problem induced by the Rawls-to-Nash family of fair gain-from-trade objectives, obtained by applying nonpositive H\"older means to the seller's and buyer's realized gains in repeated bilateral trade.

\subsection{Problem Formulation}
\label{s:setting}

The repeated bilateral trade interaction protocol is described in Interaction~Protocol~\ref{algo:interaction_protocol}.

{
\renewcommand{\algorithmcfname}{Interaction Protocol}
\begin{algorithm}[H]
\DontPrintSemicolon
\caption{Repeated Bilateral Trade}
\label{algo:interaction_protocol}
\For{$t=1,2,\dots$}{
    A seller and a buyer with private valuations $S_t,B_t\in[0,1]$ join the platform\;
    The platform posts a price $P_t\in[0,1]$\;
    Trade occurs if and only if $S_t\le P_t\le B_t$\;
    The learner observes the acceptance bits
    $Y_t\coloneqq \I\lcb{S_t\le P_t}$
    and
    $Z_t\coloneqq \I\lcb{P_t\le B_t}$\;
}
\end{algorithm}
}

\subsubsection{Fair-Gain Objectives} If trade occurs, then the seller gain is $P_t-S_t$ and the buyer gain is $B_t-P_t$, while both gains are zero otherwise.
To promote fair divisions of the surplus between sellers and buyers, rather than aggregating the two gains through their sum, we consider the following family of reward objectives.

Formally, writing $x_+ \ceq \max\lrb{x,0}$ for $x \in \bbR$, we define the \emph{$\lambda$-fair gain from trade} by
\[
    \gft_\lambda \colon [0,1]^3 \to \bbR,
\qquad
    (p,s,b) \mapsto \I\lcb{s \le p \le b} M_\lambda\brb{(p-s)_+,(b-p)_+},
\]
where $M_\lambda$ denotes the $\lambda$-H\"older mean, on $\bbR_+^2$, namely
\[
    M_\lambda(x,y)
\coloneqq
    \begin{cases}
    \min(x,y), & \lambda=-\infty,\\[2pt]
    \lrb{\frac{x^\lambda+y^\lambda}{2}}^{1/\lambda}, & -\infty < \lambda < 0,\ 0<x,\ 0<y,\\[6pt]
    \sqrt{xy}, & \lambda=0,
    \end{cases}
\]
and, for $\lambda \in (-\infty,0)$, we set $M_\lambda(x,y)=0$ whenever $xy=0$.\footnote{Note that the boundary definitions are chosen so that, with these conventions, one can show that
$
    (\lambda,x,y)\mapsto M_\lambda(x,y)
$
is continuous on the whole $[-\infty,0]\times[0,+\infty)^2$, where $[-\infty,0]$ is endowed with its natural extended-real topology.}

The family of $\lambda$-fair gain-from-trade objectives provides a natural fairness-aware refinement of the usual gain-from-trade objective.
For every realized pair $(s,b)$ with $s\le b$, the reward $\gft_\lambda(p,s,b)$ is symmetric around the midpoint $p=(s+b)/2$ and is maximized there, where both traders receive the same gain $(b-s)/2$.
The family also satisfies a Pigou-Dalton transfer principle (also known as the Robin-Hood principle): holding the total surplus fixed, transferring gain from the trader with the larger gain to the trader with the smaller gain cannot decrease the score, as long as their order is not reversed.
We stress once more that this family recovers several notions of fairness: $\gft_{-\infty}$ coincides with the notion called \emph{fair gain from trade} in \cite{bachoc2024fairNeurIPS}, which corresponds to Rawlsian fairness; $\gft_0$ corresponds to Nash fairness; and $\gft_{-1}$ corresponds to a harmonic-mean notion of fairness.

Beyond these qualitative and interpolating properties, the family also has an axiomatic justification.
In Appendix~\ref{app:why_Holder?}, we start from a generic fair-gain score assigned to the seller-buyer gain pair and impose three desiderata: an equal-gain-equivalent interpretation through a common social-welfare scale, scale covariance, and zero fairness for one-takes-all splits.
These axioms characterize the finite-order H\"older means with $\lambda\in(-\infty,0]$, with the Rawlsian objective recovered as the limiting case $\lambda\to-\infty$.

\subsubsection{Valuation Model}
Throughout the paper, unless otherwise specified, we assume that the sequence of sellers' valuations $S,S_1,S_2,\dots$ is an i.i.d.\ sequence of $[0,1]$-valued random variables, that the sequence of buyers' valuations $B,B_1,B_2,\dots$ is an i.i.d.\ sequence of $[0,1]$-valued random variables, that the two sequences are independent of each other, and that their distributions are arbitrary and unknown to the learner.
In what follows, we denote by $F_S$ the unknown sellers' cdf and by $F_B^\circ$ the unknown buyers' survival function, i.e., for any $p \in [0,1]$,
\[
    F_S(p) \coloneqq \Pb\lsb{S\le p},
\qquad
    F_B^\circ(p) \coloneqq \Pb\lsb{p \le B}.
\]

These assumptions are common in the online bilateral-trade literature \citep{cesa2024bilateralMOR,bachoc2024fairNeurIPS,bacchiocchi2025NeurIPS}.
The boundedness assumption on valuations and prices is standard in regret-minimization settings, and can be viewed as a normalization.
The i.i.d.\ assumption is natural when modeling large and stable markets \citep{bolic2024onlineAAMAS}, in which the distribution of arriving valuations does not change significantly over time.
The independence assumption between seller and buyer valuations is also natural in settings where the two sides of the market are formed by separate populations and traders join the platform without coordination.
On a technical level, removing the independence assumption between seller and buyer valuations is known to lead to unlearnability in several related settings, unless one allows for more general mechanisms, different assumptions, or relaxed benchmarks; see \Cref{s:related_works}.
As shown in \Cref{thm:linear}, analogous unlearnability phenomena also arise in our setting.

\subsubsection{Learning Objectives}
For every $\lambda\in[-\infty,0]$ and $p\in[0,1]$, we define the \emph{expected} $\lambda$-gain from trade and its maximum\footnote{We note that $G_\lambda^\star$ is indeed well defined as $G_\lambda$ is continuous on the compact set $[0,1]$; see \Cref{lem:G_is_Holder}.} as
\[
    G_\lambda(p)
\coloneqq
    \E\bsb{\gft_\lambda(p,S,B)},
\qquad
    G_\lambda^\star
\coloneqq
    \max_{q\in[0,1]} G_\lambda(q).
\]
At a high level, for a fixed fairness level $\lambda\in[-\infty,0]$, the learner's goal is to design pricing strategies that identify or sequentially post prices maximizing the expected $\lambda$-fair gain from trade $G_\lambda$.
We formalize this goal through two learning objectives: pure exploration, in which one cares about estimating the best price(s), and regret minimization.

In all cases, the learner's pricing strategy sequentially selects the price $P_t$ based on the feedback observed in previous trades,
$
    (P_1,Y_1,Z_1),\dots,(P_{t-1},Y_{t-1},Z_{t-1})
$,
and possibly on external randomization.

\paragraph{Pure exploration}
We first consider a fixed-confidence pure-exploration objective.
The learner posts prices until a stopping rule, fixed as part of the algorithm, declares that exploration is over.
The stopping rule is distribution-independent: at each round it can depend only on the feedback and on the learner's randomization observed so far; under a given valuation distribution (of $(S_t,B_t)$), it induces a random stopping time $\tau$.
At time $\tau$, the algorithm outputs a recommendation $\widehat p_\lambda\in[0,1]$, depending only on the stopped history and possibly on the desired fairness level $\lambda$.
We say that the resulting pure-exploration algorithm---consisting of the pricing strategy, the stopping rule, and the recommendation rule---is $(\varepsilon,\delta)$-PAC (Probably Approximately Correct) at fairness level $\lambda$ if
\begin{equation}
\label{eq:pointwise_PAC}
    \Pb\bsb{
        G_\lambda^\star-G_\lambda(\widehat p_\lambda)
        \le \eps
    }
    \ge
    1-\delta.
\end{equation}
We further study a stronger \emph{uniform-in-fairness-level PAC} requirement, useful when the platform wants to choose the fairness level only after collecting data through exploratory posted-price interactions.
The same exploration phase is used for all fairness levels: after the stopping rule declares that exploration is over, the algorithm outputs a recommendation map
\[
    \lambda\mapsto \widehat p_\lambda,
\]
where each $\widehat p_\lambda$ depends only on the stopped history and on the queried value of $\lambda$.
Given an accuracy parameter $\eps>0$ and a confidence parameter $\delta\in(0,1)$, the requirement is that this single map returns an $\eps$-accurate recommendation simultaneously for every $\lambda\in[-\infty,0]$, with high probability:
\begin{equation}
\label{eq:uniform_PAC}
    \Pb\lsb{
        \sup_{\lambda\in[-\infty,0]}
        \brb{
            G_\lambda^\star-G_\lambda(\widehat p_\lambda)
        }
        \le \eps
    }
    \ge
    1-\delta.
\end{equation}
For a fixed valuation distribution, the expected exploration length is the expectation of the stopping time induced by that distribution.
The sample complexity of a PAC algorithm is the worst-case expected exploration length over the considered class of valuation distributions.

We note that while our proposed algorithms use deterministic exploration budgets, which correspond to the special case $\tau\equiv n$, for some $n \in \N$, our PAC lower bound is proved in the more general fixed-confidence model, i.e., it applies to arbitrary
distribution-independent stopping rules.

\paragraph{Regret minimization}
In the \emph{regret-minimization setting} performance is evaluated directly on the sequence of posted prices.
Specifically, given a horizon $T \in \N$ and a pricing strategy $\alpha$ producing the prices $P_1,\dots,P_T$, we define its expected regret at fairness level $\lambda$ by
\[
    R_\lambda(\alpha,T)
\coloneqq
    \max_{p \in [0,1]} \E \lsb{\sum_{t=1}^T
        \brb{\gft_{\lambda}(p,S_t,B_t)-\gft_{\lambda}(P_t,S_t,B_t)}}
=
    \E\lsb{
        \sum_{t=1}^T
        \brb{G_\lambda^\star-G_\lambda(P_t)}
        }.
\]
In this case, the learner aims to design policies whose regret is sublinear in $T$. 

The goal of this paper is to characterize the optimal learning rates for these problems and to provide explicit algorithms that achieve them.

\subsection{Our Contributions}

We report here an overview of our results.
Throughout the paper, $\widetilde O$ suppresses universal constants and logarithmic factors in the relevant displayed parameters.
In particular, in PAC statements $\widetilde{O}(\eps^{-2})$ may hide logarithmic factors in $1/\eps$ and $1/\delta$, while in regret statements $\widetilde{O}(T^{2/3})$ hides logarithmic factors in $T$.

\begin{itemize}
    \item \emph{A Rawls-to-Nash Fairness Family.}
    We propose a new family $\lrb{\gft_\lambda}_{\lambda\in(-\infty,0]}$ of fair gain-from-trade objectives, and show that it arises from basic desiderata for fairness in bilateral trade; see Appendix~\ref{app:why_Holder?}.
    \item\emph{A $\lambda$-Free Exploration Strategy.}
    We introduce a sample-collection procedure (\Cref{algo:collect_samples}) that is independent of the fairness level $\lambda$.
    From the collected threshold feedback up to some fixed deterministic budget $n$, we construct recommenders $\widehat{p}_{\lambda}$ with optimization error $\widetilde O(n^{-1/2})$ with high probability: a finite-$\lambda$ estimator with explicit dependence on $1-\lambda$, and a Rawls estimator for $\lambda=-\infty$ (see \Cref{prop:finite_lambda_high_probability_optimization,prop:rawls_high_probability_optimization}).
    By using the finite-$\lambda$ recommender only when $1-\lambda\le \sqrt n$ and switching otherwise to the Rawls recommender, we obtain a single rule valid for every $\lambda\in[-\infty,0]$ whose error bound is agnostic to $\lambda$ (\Cref{prop:lambda_free_exploration_optimization}). An appropriate choice of $n$ yields an $(\varepsilon,\delta)$-PAC algorithm at fairness level $\lambda$ with sample complexity $\widetilde{O}(\varepsilon^{-2})$.
    \item\emph{Uniform PAC Learning.}
    We also study the setting in which the fairness level is chosen only after exploration.
    Using a grid over fairness levels, we turn our $\lambda$-free error bound into a single high-probability event that supports $\eps$-optimal recommendations simultaneously for all $\lambda\in[-\infty,0]$
    for exploration budget $\widetilde{O}(\eps^{-2})$ (\Cref{thm:uniform_pac_grid}).
    A matching fixed-confidence lower bound, valid already for any fixed $\lambda$ and for arbitrary distribution-independent stopping rules with a.s.\ finite stopping times, shows that the $\eps^{-2}\log(1/\delta)$ scale is unavoidable up to logarithmic factors (\Cref{thm:pac_lower_bound}).
    \item\emph{Regret Minimization.}
    We turn our $\lambda$-free exploration strategy into an explore-then-commit strategy, $\mathrm{ETC}_\lambda(T)$.
    The exploration length depends only on $T$, while the committed recommendation is selected using the pair $(\lambda,T)$.
    The resulting regret is $\widetilde O(T^{2/3})$ (\Cref{thm:regret_etc}) and the $T^{2/3}$ dependence is optimal up to logarithmic factors (\Cref{thm:regret_lower_bound}).
    \item \emph{Necessity of Independence.}
    We show that the independence assumption between seller and buyer valuations is essential under Interaction~Protocol~\ref{algo:interaction_protocol}.
    If the pairs $(S_t,B_t)_{t\in\N}$ are i.i.d.\ over time but seller and buyer valuations may be arbitrarily dependent within each round, then every algorithm suffers linear regret (\Cref{thm:linear}).
    \item \emph{Reusable Technical Tools.}
    Along the way, we develop two tools that may be useful beyond the present setting.
    First, we prove a rectangular McDiarmid inequality for row-column dependent double sums which can be useful beyond bilateral trade whenever an estimator is formed by recombining two independently sampled populations, so that each observation from one population is paired with many observations from the other (\Cref{prop:rectangular:bounded:differences}).
    Second, we formulate a modular lower-bound template for not necessarily finite online-learning problems with a revealing-action structure in the sense of \cite{cesa2006prediction}.
    The template is designed to be used as a black-box ingredient for proving $T^{2/3}$ regret lower bounds in problems where information can be acquired only through actions that are themselves suboptimal; see Appendix~\ref{app:lower_bound_template} and \Cref{thm:t23-template}.
\end{itemize}


To help the reader, given the complex structure of the results in the paper, we offer a proof-dependency graph for the upper bound results in \Cref{fig:upper_bound_graph}.

\begin{figure}[t]
    \centering
    \includegraphics[width=\textwidth]{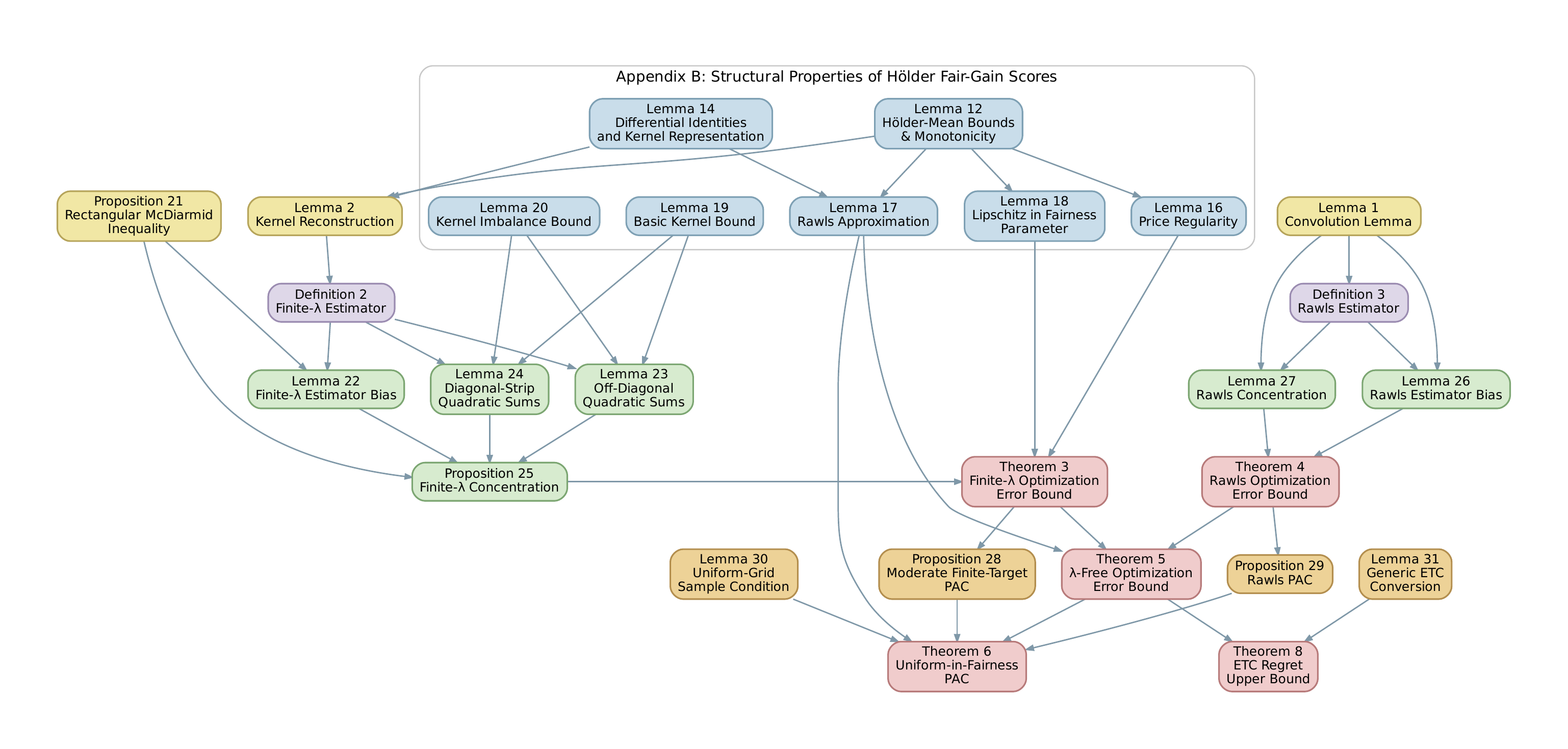}
    \caption{Proof-dependency graph for the main upper bounds.}
    \label{fig:upper_bound_graph}
\end{figure}

Next, we explain the main technical obstacles behind these results and the ideas used to overcome them.

\subsection{Techniques and Challenges}
\label{s:techniques_and_challenges}
One central difficulty of the problem is observability: after posting a price $p$, the learner receives only the two acceptance bits $\I\{S_t \le p\}$ and $\I\{p \le B_t\}$ as feedback, and, unlike in bandit problems, such feedback is \emph{not} sufficient to reconstruct the realized reward $\gft_\lambda(p,S_t,B_t)$, nor to estimate $G_\lambda(p)$ directly from repeated observations at that same price.
This is not merely a technical inconvenience, but a real learnability barrier: as shown in \Cref{thm:linear}, even if seller and buyer valuations remain i.i.d.\ over time, allowing them to be arbitrarily dependent within each round makes the problem unobservable.

Under independence of seller and buyer valuations, we overcome the observability barrier through a key structural result of the paper: a new \emph{Kernel Reconstruction Lemma} (\Cref{lem:kernel_reconstruction_lemma}),  which rewrites the expected $\lambda$-fair gain from trade in terms of the seller cdf, the buyer survival function, and a kernel $H_\lambda$.
This lemma turns an apparently unobservable reward into a form that can be estimated from the available feedback, and provides the right representation of the problem for the subsequent estimation analysis.

On the other hand, even though \Cref{lem:kernel_reconstruction_lemma} resolves the observability issue, it does not reduce the problem to the same estimation task as in \cite{bachoc2024fairNeurIPS}.
In the Rawlsian case studied there, the reconstruction collapses to a one-dimensional convolution.
After discretization, estimation is made at a regular grid of prices and
the estimator at a fixed grid point is an average along an anti-diagonal of the grid; although the same observations are reused across different grid points, the summands entering any single fixed-grid estimate are bounded and built from disjoint time indices, leading to a sum of bounded and independent random variables over which one can apply the ordinary Hoeffding inequality.

In the Rawls-to-Nash family, the reconstruction is instead a two-dimensional integral against a singular kernel.
Once discretized, the estimator of $G_\lambda$ at a fixed grid point (see \Cref{def:estimator_G}) takes the form of a rectangular double sum over seller-side and buyer-side observations.
This creates a row-column dependence structure: each seller-side observation is paired with many buyer-side observations, and vice versa.
Moreover, because the kernel is singular, the corresponding weights can grow as the mesh is refined.
Thus the estimator is not an empirical average of independent bounded variables, but a dependent double sum with large kernel weights.
Consequently, the concentration argument used in the Rawlsian case no longer applies directly.

One might try to avoid this dependence by abandoning the double-sum structure and using more direct strategies.
For instance, one could estimate the seller cdf and the buyer survival function pointwise on a grid and then plug these estimates into the kernel reconstruction formula.
This, however, would waste samples on estimating the marginal functions themselves, rather than the specific reward functional needed for optimization.
Alternatively, one could try to estimate the reconstructed integral directly through independent Monte Carlo samples.
This avoids the dependence induced by the double sum, but runs into the other difficulty created by the kernel reconstruction: the kernel $H_\lambda$ has boundary singularities, so the resulting kernel-weighted summands are not uniformly bounded.
A Hoeffding-type argument would then require truncating the summands, introducing a bias-variance tradeoff that leads to suboptimal rates.
Heuristically, both approaches correspond to an $\eps^{-4}$-type PAC scale, and hence to an explore-then-commit regret rate around $T^{4/5}$, rather than the optimal $\eps^{-2}$ and $T^{2/3}$ scales.
Thus, neither the Rawlsian concentration argument nor a na\"ive independent-sampling alternative captures the correct statistical structure of the problem.

The resolution is to keep the rectangular estimator, but analyze its dependence structure directly.
To do this, we first prove a dedicated rectangular McDiarmid inequality for row-column dependent double sums (\Cref{prop:rectangular:bounded:differences}). Applying this inequality requires bounding a quadratic sum of deterministic envelopes, that is specific to the geometry of the kernel.

In our case, this step turns out to be technically delicate.
We split the grid
of indices
into a diagonal strip and an off-diagonal region.
Inside the diagonal strip, the kernel weights can be large, but the strip has only constant width, so its contribution can be controlled by basic kernel bounds (\Cref{lem:diag_quadratic}).
In the off-diagonal region, the kernel is less singular but spread over many cells: there we use a kernel imbalance bound (\Cref{cor:kernel_imbalance_bound}) together with summability bounds for the resulting row and column envelopes (\Cref{lem:offdiag_quadratic}) to show that the row and column quadratic sums remain of the right order.
This diagonal-strip/off-diagonal decomposition is what allows the rectangular McDiarmid inequality to deliver a fixed-grid estimation error of order $n^{-1/2}$, up to logarithmic factors, despite both dependence and singular weights.

The next estimation issue is discretization in prices.
The estimators are defined on a finite grid of prices, whereas the benchmark optimizes over the continuum $[0,1]$.
Thus, even a sharp grid-level estimate is insufficient unless the reward functions have enough regularity in the posted price.
For finite $\lambda$, the ordinary Lipschitz control available in \cite{bachoc2024fairNeurIPS} is no longer available because the $\lambda$-H\"older mean has singular derivatives near the boundary of the trade interval.
Nevertheless, we prove that every objective $G_\lambda$ is $\frac12$-H\"older in the price (\Cref{lem:G_is_Holder}), which is just enough to transfer the $n^{-1/2}$-type grid-level optimization bounds to the continuum.

At first sight, the fixed-grid estimation problem may appear to be satisfactorily resolved.
A closer look at the finite-$\lambda$ guarantee, however, reveals an undesirable dependence on $1-\lambda$.
The term worsens as $\lambda$ moves toward the Rawlsian regime and, when $|\lambda|$ is large relative to the exploration budget $n$, can make the bound trivial.
The source of the dependence is the control of the singular finite-$\lambda$ kernel.
Yet, one would expect that such deterioration is not intrinsic to the objective family: at $\lambda=-\infty$, the reconstruction collapses to the one-dimensional Rawls convolution, and an estimator inspired by \cite{bachoc2024fairNeurIPS}
(see \Cref{def:estimator_G_rawls})
again achieves an $n^{-1/2}$-type guarantee (see \Cref{prop:rawls_high_probability_optimization}).

The resolution is to stop forcing a single estimator to handle all values of $\lambda$.
For a fixed exploration length $n$, the algorithm uses the finite-$\lambda$ estimator only in the moderate regime $1-\lambda\le\sqrt n$, where the explicit $1-\lambda$ term can be absorbed into the statistical error.
In the complementary Rawls-like tail, it switches to the Rawls recommender.
This switch is justified by a deterministic approximation inequality showing that $G_\lambda$ is uniformly close to $G_{-\infty}$ when $\lambda$ is very negative (\Cref{lem:rawls_approx}).
Thus, the finite-$\lambda$ estimator is used exactly where its concentration bound is effective, while the Rawls estimator controls the ill-conditioned tail.
This yields a fixed-$\lambda$ optimization guarantee whose leading rate is independent of the fairness level.

A related but distinct challenge arises when the fairness level is chosen after exploration, rather than fixed in advance.
This leads us to study PAC guarantees that hold uniformly over $\lambda\in[-\infty,0]$.
Here, the issue is no longer only to estimate one fixed objective $G_\lambda$, for some $\lambda$ fixed in advance.
Rather, after a single exploration phase, the learner must produce $\lambda$-dependent recommendations whose guarantees hold \emph{simultaneously} for all fairness levels.
We handle this by discretizing the fairness parameter and proving a single high-probability event on which all grid recommendations are accurate.
The key step that makes this discretization valid is a Lipschitz regularity property of the H\"older means with respect to the fairness parameter (\Cref{lem:M_lipschitz_in_alpha}).
This lets us transfer accuracy from a finite fairness grid to every $\lambda$ in the moderate range ($1-\lambda \le \sqrt{n}$), while the Rawls-approximation argument handles the far negative tail.

The same exploration routine then yields regret guarantees through the explore-then-commit strategy of \Cref{algo:etc} (\Cref{thm:regret_etc}), with a tuning that depends only on the horizon.

The regret lower bound (\Cref{thm:regret_lower_bound}) poses a different kind of challenge.
One cannot simply write down an abstract partial-monitoring game and import a known lower bound, because both the reward and the feedback must be induced by bilateral-trade instances: seller and buyer valuation distributions, posted prices, and threshold feedback.
The hard instance
(where two close distributions of $(S,B)$ are possible)
must therefore be engineered so that the induced $G_{\lambda}$ have separated optima while the feedback reveals the underlying distribution only through prices that are themselves suboptimal.
We construct such instances in Appendix~\ref{app:regret_lower_bound}.
The construction creates a revealing-action structure: to learn on which side the optimum lies, the learner must repeatedly use informative prices and pay a fixed per-round opportunity cost.
The resulting lower bound is proved through an axiomatic template in Appendix~\ref{app:lower_bound_template}, which isolates the information-theoretic mechanism behind the $T^{2/3}$ rate (\Cref{thm:t23-template}).
The same hard instances, constructed in Appendix~\ref{app:regret_lower_bound} for \Cref{thm:regret_lower_bound}, also yield the fixed-confidence PAC lower
bound of \Cref{thm:pac_lower_bound}.
The reduction, formalized in \Cref{thm:pac-template}, turns any PAC algorithm with a
distribution-independent stopping rule into a test between the two environments, while the stopped-history KL bound in \Cref{lem:pac-kl} controls the available information by the expected stopping time.


\subsection{Related Work}
\label{s:related_works}
Starting from the impossibility result of \citet{myerson1983efficient}, a large literature has studied approximately efficient or sample-based mechanisms for bilateral trade and related two-sided settings; see, e.g., \citet{BlumrosenM16,Colini-Baldeschi16,Colini-Baldeschi17,brustle2017approximating,babaioff2020bulow,colini2020approximately,blumrosen2021almost,dutting2021efficient,DengMSW21,kang2022fixed,deng2025approximately,hajiaghayi2025gains}.
This line of work is offline in nature: the valuation distributions are known, sampled from, or accessed before the mechanism is chosen, and the goal is to design or approximate a truthful mechanism rather than to learn from sequential market feedback.

To the best of our knowledge, the closest prior work that investigated fairness in this offline mechanism-design literature is \citet{babaioff2025efficiencyEC}.
In an offline Bayesian mechanism-design setting with \emph{known} valuation distributions, they study the efficiency loss induced by imposing a Kalai-Smorodinsky-type requirement on truthful mechanisms.
Their requirement is imposed on expected utilities: the mechanism equalizes the ratios $U_S(M)/U_S^\star$ and $U_B(M)/U_B^\star$, where $U_S(M)$ (resp. $U_B(M)$) is the seller's (resp. buyer's) expected utility under mechanism $M$, and $U_S^\star$ (resp. $U_B^\star$) is the maximal expected utility attainable by optimizing for the seller (resp. buyer) alone.
Thus, the normalizing benchmarks are unilateral counterfactual claims: each trader is compared to what could be achieved by optimizing for that trader while disregarding the other trader's payoff, although the gains being divided are competing shares of the same surplus.
The resulting constraint equalizes relative concessions from these one-sided claims.
They then compare the resulting mechanisms with the second-best gain-from-trade benchmark, namely the maximum expected gain from trade achievable by truthful mechanisms.
By contrast, our objectives assign a common gain-equivalent fairness score to the realized seller-buyer gains in each trade, and take its expectation under arbitrary \emph{unknown} valuation distributions as the learning objective: by interacting with the traders posting prices and from threshold feedback alone, the platform must learn how to maximize the objective uniformly over the valuation distributions.

Online bilateral trade as a regret-minimization problem was initiated by \citet{cesa2021bilateralEC} and further developed in \citet{cesa2024bilateralMOR}.
These works depart from the classical Bayesian mechanism-design viewpoint where the seller and buyer distributions are assumed to be known, and instead study repeated posted-price mechanisms in which the learner must infer through sequential interactions with traders the distributional features relevant to the efficiency objective, which is again the \emph{gain from trade} from the offline literature: if a price $p$ is posted to a seller-buyer pair with valuations $S,B$, trade occurs when $S\le p\le B$, and the realized gain from trade is
$
    (p-S)+(B-p)=B-S
$
on the trade event, and $0$ otherwise.
The benchmark against which the learner is compared is then the ex-ante optimal fixed posted price for this objective, namely the price maximizing the expected gain from trade under the true unknown valuation distributions.
They obtain minimax-optimal worst-case regret guarantees against this benchmark, uniformly over classes of independent seller and buyer valuation distributions satisfying a bounded-density assumption.

Subsequent work studied several variants of the problem, including smoothed-adversary models with weak budget-balance constraints \citep{cesa2023smoothT34COLT,cesa2024smoothT34JMLR}, $\alpha$-regret in adversarial settings \citep{azar2024alphaAI}, relaxations of budget-balance requirements \citep{bernasconi2024globalSTOC,chen2026tightICALP,lunghi2026betterSODA,lunghi2026strongerArXiv}, profit-maximization objectives for general incentive-compatible mechanisms \citep{digregorio2025profitFOCS} and for fixed-price mechanisms under minimal one-bit feedback \citep{castiglioni2026sampleSTOC}, and contextual settings \citep{gaucher2025featureICLR,coccia2026nonparametricICLR,cosson2026contextualArXiv}.

A number of related online-learning models study neighboring two-sided market problems.
These include online brokerage, where traders may act as sellers or buyers depending on market conditions, under both gain-from-trade and trading-volume objectives \citep{bolic2024onlineAAMAS,bachoc2025parcnxtICML,bachoc2025nonparcnxtAISTATS,cesari2025volumeICLR,zhao2026bilateralArXiv}, as well as the repeated mediated newsvendor problem \citep{bolic2025newsvendorNeurIPS}, where traded quantities vary over a continuum.
While these models may differ in market structure, learning objective, and feedback, they share with bilateral trade the broader theme of learning to intermediate between two sides of a market under limited information.

Fairness considerations in online bilateral trade were introduced by \citet{bachoc2024fairNeurIPS}, where the authors proposed the \emph{fair gain from trade} objective.
Unlike the usual gain from trade, which sums the traders' realized net gains, their fair objective uses the Rawlsian score given by the \emph{minimum} of the seller's and buyer's realized gains.
They showed that optimizing the usual gain-from-trade objective is not a reliable surrogate for fairness: there are stochastic instances on which algorithms with no regret for gain from trade incur linear regret for fair gain from trade.
This separation motivates studying fair objectives in bilateral trade as a learning problem in its own right, rather than expecting fairness to emerge as a by-product of efficiency maximization.

On the technical side, the main positive result of \citet{bachoc2024fairNeurIPS} under independent seller and buyer valuations is based on their Convolution Lemma, which we recall for completeness in \Cref{lem:convolution_lemma}.
For the Rawlsian objective, this identity expresses the expected fair gain from trade in terms of the seller cdf and the buyer survival function, thereby making the unobserved reward estimable from threshold feedback.
They use this lemma to devise an explore-then-commit-type algorithm: during the exploration phase, the learner collects threshold-feedback observations at grid prices, recombines independent observations to build discrete convolution estimates, and then commits to an estimated best price.
This yields $\widetilde O(T^{2/3})$ regret, with a matching lower bound up to logarithmic factors.
Interestingly, the bounded-density assumptions needed for the gain-from-trade guarantees of \citet{cesa2024bilateralMOR} can be removed for the Rawlsian fair objective due to its Lipschitz regularity.

Our work compares with \citet{bachoc2024fairNeurIPS} along three main dimensions.
First, we move from the single Rawlsian objective to the full Rawls-to-Nash family of H\"older-mean fairness objectives, including the harmonic and Nash objectives as special cases.
Second, we provide an axiomatic foundation for this objective family through an equal-gain-equivalent interpretation, scale covariance, and zero fairness for one-takes-all splits, which in particular identifies the Rawlsian objective as a limiting endpoint of the family.
Third, at the technical level, the problem becomes significantly more complex.
Indeed, while the Rawlsian reconstruction of \citet{bachoc2024fairNeurIPS} is one-dimensional, the finite-order objectives require a two-dimensional reconstruction through a singular kernel.
This substantially complicates the statistical structure of the estimation problem: the natural estimators become rectangular double sums of row-column dependent terms, with singular-kernel weights that explode as the discretization mesh shrinks, rather than one-dimensional anti-diagonal sums of uniformly bounded independent terms as in \citet{bachoc2024fairNeurIPS}
(recall \Cref{s:techniques_and_challenges} for more details).
Interestingly, unlike in the gain-from-trade analysis of \citet{cesa2024bilateralMOR} but similarly to the Rawlsian analysis of \citet{bachoc2024fairNeurIPS}, bounded-density assumptions play no role in our upper bounds: the relevant reward functions are $1/2$-H\"older uniformly over the fairness level.
In this sense, we believe that our work can be seen as a substantial extension of the Rawlsian theory of \citet{bachoc2024fairNeurIPS}.


\section[The Kernel Reconstruction Lemma]{The Kernel Reconstruction Lemma}

The main obstacle is that the platform never observes the realized fair reward.
For Rawlsian fairness, \cite{bachoc2024fairNeurIPS} overcame this through a one-dimensional convolution identity, which we recall here, and whose proof can be found for completeness in Appendix~\ref{app:kernel_reconstruction}.

\begin{restatable}[Convolution Lemma]{lemma}{convolutionlemma}
\label{lem:convolution_lemma}
    For every $p,s,b \in [0,1]$,
	\begin{align*}
		\gft_{-\infty}(p,s,b)
		&=
		\int_0^{\min(p,1-p)} \I\{s \le p-u  \} \I\{p+u\le b\} \dif u.
	\end{align*}
In particular, if $S,B$ are two $[0,1]$-valued independent random variables, with cdf $F_S$ and survival function $F_B^\circ$, respectively, for every $p \in [0,1]$, it holds
\begin{align*}
		G_{-\infty}(p)
        =
        \E\bsb{\gft_{-\infty}(p,S,B)}
		&=
		\int_0^{\min(p,1-p)} F_S(p-u) F_B^\circ(p+u) \dif u.
	\end{align*}
\end{restatable}

For finite $\lambda$, the previous formula no longer holds, and in order to reconstruct from the corresponding threshold feedback the $\lambda$-fair gain from trade we need to take a 2-dimensional integral involving a singular kernel, whose definition we now provide.
\begin{definition}
    For $\lambda \in (- \infty,0]$ and $x,y \ge 0$, define
\[
	H_\lambda(x,y) \coloneqq 
	\begin{cases}
		\frac{1-\lambda}{4}\cdot\lrb{\frac{M_\lambda(x,y)}{M_0(x,y)}}^{1-2\lambda} \frac{1}{M_0(x,y)}, & \textrm{if $x>0$ and $y>0$},\\[3pt]
		0, & \textrm{otherwise}.
	\end{cases}
\]
\end{definition}
We call $H_\lambda$ singular because it is a nonnegative integrable kernel,
\emph{unbounded} near the corner where both kernel arguments vanish.
Having the singular kernel $H_\lambda$, we can now provide an analogous reconstruction formula to \Cref{lem:convolution_lemma}, but for the finite $\lambda$ case.

\begin{restatable}[Kernel Reconstruction Lemma]{lemma}{kernelreconstructionlemma}
\label{lem:kernel_reconstruction_lemma}
    Let $p,s,b \in [0,1]$.
    For each $\lambda \in (- \infty,0]$ it holds that
	\[
	\gft_\lambda(p,s,b)
	=
    \int_{0}^p\int_{p}^1
	H_\lambda(p-v,w-p) \I\{s\le v\}\I\{ w \le b\} \dif w\dif v.
	\]
    In particular, if $S,B$ are two $[0,1]$-valued independent random variables, with cdf $F_S$ and survival function $F_B^\circ$, respectively, for every $p \in [0,1]$ and every $\lambda \in (-\infty,0]$, it holds
    \begin{align*}
		G_{\lambda}(p)
        =
        \E\bsb{\gft_{\lambda}(p,S,B)}
		&=
		\int_{0}^p\int_{p}^1
	    H_\lambda(p-v,w-p) F_S(v)F_B^\circ(w) \dif w\dif v.
	\end{align*}
\end{restatable}

We defer the proof of this result to Appendix~\ref{app:kernel_reconstruction}.
We remark that the Rawlsian convolution identity is recovered as the singular limit of this two-dimensional reconstruction: as $\lambda\to-\infty$, the kernel measures concentrate on the Rawlsian anti-diagonal and converge weakly to the corresponding line measure.
This singular-convergence picture is made precise in \Cref{remark:kernel_approximation} and \Cref{fig:kernel-heatmaps} in Appendix~\ref{app:kernel_reconstruction}.

\section{A Universal Exploration Procedure}
\label{s:pac}

\subsection{Estimators}

We now present the exploration procedure used to gather information, which, interestingly, is the same for any fairness level $\lambda$.
Given an exploration length $n$, the learner posts a randomized grid of prices, observes the two threshold bits, and stores the resulting sample.

\begin{algorithm}[H]
\DontPrintSemicolon
\textbf{Input:} exploration length $n \ge 4$\;
\For{$t=1,\dots,n$}{
Sample and post price
$
    P_t
\sim
    \mathrm{Unif}\lrb{\lsb{\frac{t-1}{n},\frac{t}{n}}}
$\;
Observe
$
    Y_t
\coloneqq
    \I\{S_t\le P_t\}
$
and
$
    Z_t
\coloneqq
    \I\{P_t\le B_t\}
$.
}
\Return $\lrb{P_t,Y_t,Z_t}_{t \in [n]}$\;
\caption{\textrm{Exploration Procedure}}
\label{algo:collect_samples}
\end{algorithm}

Based on the feedback collected by Algorithm~\ref{algo:collect_samples}, we now define the two estimators suggested by the reconstruction formulas above. The finite-$\lambda$ estimator is obtained by discretizing the two-dimensional kernel formula of \Cref{lem:kernel_reconstruction_lemma}, while the Rawls estimator is obtained by discretizing the one-dimensional convolution formula of \Cref{lem:convolution_lemma}. The finite-$\lambda$ estimator is defined on the grid $
    \lcb{\frac{2}{n},\dots,\frac{n-2}{n}}
$ (of values of $p$)
whereas the Rawls estimator is defined on the grid
$
    \lcb{\frac{1}{n},\dots,\frac{n-1}{n}}
$.
We also define the corresponding recommended prices, selected by maximizing these estimators on their respective grids.

\begin{definition}[Finite-$\lambda$ Estimator]
\label{def:estimator_G}
For every $\lambda\in(-\infty,0]$ and every $k\in\{2,\dots,n-2\}$, define
\[
    \widehat{G}_{\lambda,k}(n)
\coloneqq
    \frac{1}{n^2}
    \sum_{i=1}^{k-1}
    \sum_{j=k+2}^{n}
    H_\lambda\lrb{\frac{k}{n}-P_i,P_j-\frac{k}{n}} Y_iZ_j.
\]
We then define
\[
    \widehat{k}_\lambda(n)
\in
    \argmax_{k\in\{2,\dots,n-2\}}
    \widehat{G}_{\lambda,k}(n),
\qquad
    \widehat{p}_\lambda(n)
\coloneqq
    \frac{\widehat{k}_\lambda(n)}{n}.
\]
\end{definition}

\begin{definition}[Rawls Estimator]
\label{def:estimator_G_rawls}
For every $k\in\{1,\dots,n-1\}$, define
\[
    \widehat{G}_{-\infty,k}(n)
\coloneqq
    \frac{1}{n}
    \sum_{a=0}^{\min\lrb{k,n-k}-1}
    Y_{k-a}Z_{k+a+1}.
\]
We then define
\[
    \widehat{k}_{-\infty}(n)
\in
    \argmax_{k\in\{1,\dots,n-1\}}
    \widehat{G}_{-\infty,k}(n),
\qquad
    \widehat{p}_{-\infty}(n)
\coloneqq
    \frac{\widehat{k}_{-\infty}(n)}{n}.
\]
\end{definition}

Neither estimator is exactly unbiased for its grid target, although the source of the bias is different in the two cases.

The Rawls estimator uses the full collection of anti-diagonal cells suggested by \Cref{lem:convolution_lemma}.
Nevertheless, as illustrated in \Cref{lem:rawls_bias}, it has a small one-sided cell-averaging bias because
in the (Rawls) convolution lemma, the same integration variable $u$ appears twice, and for the estimator, two different independent random prices correspond to these two occurrences of $u$. We refer to the proof of \Cref{lem:rawls_bias} for details.

For finite $\lambda$, the situation is different.
The full rectangular discretization of \Cref{lem:kernel_reconstruction_lemma}, with $i=1,\dots,k$ and $j=k+1,\dots,n$
(instead of $i=1,\dots,k-1$ and $j=k+2,\dots,n$ in \Cref{def:estimator_G}), would be unbiased for $G_\lambda(k/n)$.
Indeed, conditionally on the posted prices, independence across rounds gives
$
    \E\bsb{Y_iZ_j\mid P_i,P_j}
    =
    F_S(P_i)F_B^\circ(P_j)
$, for any $i\le k,\ j\ge k+1$, and averaging over the randomized grid prices 
provides the integral in \Cref{lem:kernel_reconstruction_lemma}
restricted to
the rectangle
$
    \lsb{\frac{i-1}{n},\frac{i}{n}}
    \times
    \lsb{\frac{j-1}{n},\frac{j}{n}}
$.
Since these rectangles form a partition of
$
    [0,k/n]\times[k/n,1]
$,
up to null boundaries, the full rectangular estimator would have expectation $G_\lambda(k/n)$.

However, this full rectangular estimator is hard to tame: under some valuation distributions allowed by our model, it may even have infinite variance.
The problem arises because this alternative estimator contains boundary cells where the arguments of the singular kernel $H_\lambda$ can be arbitrarily close to zero, and hence the corresponding kernel-weighted summands have unbounded range and can have non-integrable second moments.

For this reason, we use the truncated estimator in \Cref{def:estimator_G}: by giving up exact unbiasedness, we pay a one-sided bias of order $n^{-1/2}$ (as quantified in \Cref{lem:finite_bias}), and in return obtain a concentration radius of the same algebraic order, up to logarithmic factors and some dependence on $\lambda$.
Specifically, after the truncation, the kernel-weighted summands admit finite deterministic envelopes whose row-column quadratic sums can be controlled (see \Cref{lem:diag_quadratic,lem:offdiag_quadratic}) in a way that allows us to apply our rectangular McDiarmid inequality (\Cref{prop:rectangular:bounded:differences}), which in turn yields the desired concentration bound in \Cref{prop:moderate_concentration}.

\subsection{Optimization Error Bounds}

In this section, we present the optimization error bounds that we obtain using the recommended prices defined in \Cref{def:estimator_G,def:estimator_G_rawls}.
Here the fairness level $\lambda$, the sample size $n$, and the confidence parameter $\delta$ are given.

\begin{restatable}[Finite-$\lambda$ Optimization Error Bound]{theorem}{finitekernelhighproboptimization}
\label{prop:finite_lambda_high_probability_optimization}
For every $\lambda\in(-\infty,0]$, every $\delta\in(0,1)$, and every $n\in\N$ with $n\ge4$, if Algorithm~\ref{algo:collect_samples} is run with exploration length $n$, then
\[
    \Pb\lsb{\;
        G_\lambda^\star
        -
        G_\lambda\brb{\widehat p_\lambda(n)}
        \le
        \lrb{62(1+\ln n)
    +
    15\frac{1-\lambda}{\sqrt n}}\sqrt{\frac{\ln(2n/\delta)}{n}}^{\phantom{\lambda}}
    }
    \ge
    1-\delta.
\]
\end{restatable}
\begin{proofsketch}
The proof uses a bias bound for $\widehat{G}_{\lambda,k}(n)$ (\Cref{lem:finite_bias}) and a control of its fluctuations (\Cref{prop:moderate_concentration}) relying on a McDiarmid concentration inequality for rectangular double sums (\Cref{prop:rectangular:bounded:differences}). In this latter inequality, a row- (resp. column-) sum of squared column- (resp. row-) sums needs to be bounded. It is divided into a diagonal strip term
(\Cref{lem:diag_quadratic})
and an off-diagonal one (\Cref{lem:offdiag_quadratic}). The diagonal term yields the $(1-\lambda)/\sqrt n$ contribution.
We defer the full proof to Appendix~\ref{app:finite:lambda:optim:guarantee}.
\end{proofsketch}

We now present the optimization error bound for the Rawls objective.

\begin{restatable}[Rawls Optimization Error Bound]{theorem}{rawlshighproboptimization}
\label{prop:rawls_high_probability_optimization}
For every $\delta\in(0,1)$ and every $n\in\N$ with $n\ge4$, if Algorithm~\ref{algo:collect_samples} is run with exploration length $n$, then
\[
    \Pb\lsb{\;
        G_{-\infty}^\star
        -
        G_{-\infty}\brb{\widehat p_{-\infty}(n)}
        \le
        2\sqrt{\frac{\ln(2n/\delta)}{n}}^{\phantom{\lambda}}
    }
    \ge
    1-\delta.
\]
\end{restatable}
\begin{proofsketch}
The proof is divided into similar (but much simpler) main steps as the proof of \Cref{prop:finite_lambda_high_probability_optimization}.
In particular,
we address a single sum with summands bounded by one, so
we do not need the rectangular McDiarmid inequality but simply Hoeffding's inequality.
We defer the proof to Appendix~\ref{app:rawls:optim:guarantees}.
\end{proofsketch}

In the error rate for the finite-$\lambda$ case, notice the term $
    \frac{1-\lambda}{\sqrt n}.
$
Thus, for a fixed exploration budget $n$, the finite-$\lambda$ guarantee is uniform in $\lambda$ in the moderate regime
$
    1-\lambda\le\sqrt n
$.
Outside this regime, the explicit dependence on $1-\lambda$ may make the finite-$\lambda$ bound ineffective.

Nevertheless, as $\lambda\to-\infty$, the objective $G_\lambda$ converges uniformly to $G_{-\infty}$.
We make this convergence quantitative and use it to handle the far-negative regime $\lambda\le 1-\sqrt n$: in this range, the Rawls estimator $\widehat G_{-\infty,k}(n)$ can be used as a proxy for $G_\lambda(k/n)$.
We provide the corresponding hybrid price recommendation, together with its optimization-error bound, next.
Importantly, this bound depends only on $n$ and $\delta$, while holding uniformly over the whole range $\lambda\in[-\infty,0]$.

\begin{restatable}[$\lambda$-Free Optimization Error Bound]{theorem}{lambdafreeoptimization}
 
\label{prop:lambda_free_exploration_optimization}
Fix $\lambda\in[-\infty,0]$, $\delta\in(0,1)$, and $n\ge4$.
After running Algorithm~\ref{algo:collect_samples} with exploration length $n$, return
\[
    \widetilde p_\lambda(n)
    \coloneqq
    \begin{cases}
    \widehat p_\lambda(n),
    &
    \lambda\in(-\infty,0]\ \textrm{and}\ 1-\lambda\le\sqrt n,\\
    \widehat p_{-\infty}(n),
    &
    \textrm{otherwise.}
    \end{cases}
\]
Then
\[
    \Pb\lsb{\;
        G_\lambda^\star
        -
        G_\lambda\brb{\widetilde p_\lambda(n)}
        \le
        77\lrb{1+\ln n}
        \sqrt{\frac{\ln(2n/\delta)}{n}}^{\phantom{\lambda}}
    }
    \ge
    1-\delta.
\]
\end{restatable}

\begin{proofsketch}
The proof splits the parameter range into three regimes.
If $\lambda\in(-\infty,0]$ and $1-\lambda\le\sqrt n$, the algorithm uses the finite-$\lambda$ recommendation $\widehat p_\lambda(n)$, and the claim follows directly from \Cref{prop:finite_lambda_high_probability_optimization}; the moderate-regime condition absorbs the extra $(1-\lambda)/\sqrt n$ term into the $1+\ln n$ factor.
If $\lambda=-\infty$, the algorithm uses the Rawls recommendation $\widehat p_{-\infty}(n)$, and the result follows from \Cref{prop:rawls_high_probability_optimization}.
Finally, if $\lambda\in(-\infty,0]$ but $1-\lambda>\sqrt n$, the finite-$\lambda$ guarantee may be ineffective, so the algorithm again uses the Rawls recommendation.
In this far-negative regime, the objective $G_\lambda$ is uniformly close to the Rawls objective $G_{-\infty}$: \Cref{lem:rawls_approx} gives an approximation error of order $n^{-1/2}$.
Combining this approximation with the Rawls optimization guarantee yields an $O(\sqrt{\ln(n/\delta)/n})$ bound, which is dominated by the displayed radius.
We defer the full proof to Appendix~\ref{app:lambda:free:guarantee}.
\end{proofsketch}

\begin{remark}[Sample Complexity for a Fixed Fairness Level] It follows from \Cref{prop:lambda_free_exploration_optimization} that using \Cref{algo:collect_samples} with an exploration budget $n$ that satisfies 
\[n \geq \frac{5929(1+\ln n )^2\ln(2n/\delta)}{\varepsilon^2}\]
together with the recommendation $\widetilde{p}_{\lambda}(n)$ yields an $(\varepsilon,\delta)$-PAC algorithm at fairness level $\lambda$. We now show that a similar condition on $n$ is in fact sufficient for $(\varepsilon,\delta)$-PAC guarantees that are uniform in the fairness level. In both cases a sufficient condition is to select $n=\mathcal{O}\left(\frac{1}{\varepsilon^2}\left(\ln\left(\frac{1}{\delta\varepsilon^4}\right)\right)^3\right)$, see the proof of \Cref{thm:uniform_pac_grid}.

 Alternatively, for any $n\ge 4$, every fixed $\lambda \in [-\infty,0]$, we may also state the following bound on the probability that \Cref{algo:collect_samples} run with a fixed budget $n$ makes an $\varepsilon$-error in its recommendation: 
 \[
   \forall \varepsilon > 0, \ \  \Pb\Bsb{
        G_\lambda^\star
        -
        G_\lambda\brb{\widetilde p_\lambda(n)}
        > \varepsilon
        } \leq 2n\exp\left(-\frac{n\varepsilon^2}{5929(1+\ln n)^2}\right).
\]
This type of formulation is common in the literature on fixed-budget pure-exploration \citep{audibert2010best,Zhao22SRSR}. Our deterministic exploration strategy in fact yields both fixed-budget and fixed-confidence (PAC) guarantees.
\end{remark}

\subsection{Uniform PAC Guarantees}

In \Cref{prop:lambda_free_exploration_optimization}, for any fixed $\lambda$, the optimization-error bound holds with probability at least $1-\delta$.
Now, the goal is to slightly enlarge the bound on the optimization error, such that the enlarged bound holds {\it uniformly} in $\lambda$. This is possible thanks to the Lipschitzness in $\lambda$ of the function $M_{\lambda}$ (see \Cref{lem:M_lipschitz_in_alpha}). 

In the next theorem, we reformulate this uniform-in-$\lambda$ bound as a uniform PAC guarantee.
Specifically, from
a target optimization error $\eps$ and target probability $\delta$, we provide a sufficient sample size $n$ to achieve $(\eps,\delta)$-PAC guarantees that hold uniformly in the fairness level $\lambda$.
This can be useful when the learner does not know in advance to which level of fairness to commit before seeing the data.
The recommended prices in this theorem are obtained by applying the recommendation rule of \Cref{prop:lambda_free_exploration_optimization} on a uniform grid of fairness levels, with mesh size of order $\eps$.

\begin{restatable}[Uniform in Fairness PAC Guarantees]{theorem}{uniformpacgrid}
\label{thm:uniform_pac_grid}
Let $\eps\in(0,1)$ and $\delta\in(0,1)$.
Set
\[
    \Gamma_\eps
    \coloneqq
    \lcb{
        -j\frac{\eps}{4}
        :
        j=0,\dots,
        \left\lceil
        \frac{4\ln2}{\eps\ln(1+\eps)}
        \right\rceil
    }.
\]
For $\lambda\in[-\ln2/\ln(1+\eps),0]$, let $\lambda_\eps$ be a closest point to $\lambda$ in $\Gamma_\eps$, with ties broken toward $0$.\footnote{Equivalently, if $J_\eps=\lceil 4\ln2/(\eps\ln(1+\eps))\rceil$, one may set $j_\eps(\lambda)=\min\lcb{J_\eps,\max\lcb{0,\lceil -4\lambda/\eps-1/2\rceil}}$ and $\lambda_\eps=-j_\eps(\lambda)\eps/4$.}
After Algorithm~\ref{algo:collect_samples} has been run with exploration length $n \ge 4$, define the queried recommendation by
\[
    \widehat p_{\lambda,\eps}(n)
    \coloneqq
    \begin{cases}
    \widehat p_{-\infty}(n),
    &
    \lambda<-\frac{\ln2}{\ln(1+\eps)},\\[1mm]
    \widehat p_{\lambda_\eps}(n),
    &
    -\frac{\ln2}{\ln(1+\eps)}\le\lambda\le0.
    \end{cases}
\]
There exists a universal constant $C_{\rm unif}>0$ such that, if
\[
    n
    \ge
    \frac{C_{\rm unif}}{\eps^2}
    \lrb{\ln\frac{C_{\rm unif}}{\eps^4\delta}}^3,
\]
then
\[
    \Pb\lsb{
        \sup_{\lambda \in [-\infty,0]} \Brb{G_\lambda^\star
        -
        G_\lambda\brb{\widehat p_{\lambda,\eps}(n)}}
        \le
        \eps
    }
    \ge
    1-\delta.
\]
Moreover, one may take $C_{\rm unif}=864000$.\footnote{We remark that the numerical constant is not optimized.
It results from conservative simplifications made to obtain an explicit bound that is uniform over the full range of $\eps$ and $\delta$.
We leave its optimization to readers with the inclination for it.}
\end{restatable}
\begin{proofsketch}
The grid $\Gamma_\eps$ has cardinality at most $10/\eps^2$.
For every $\lambda\in[-\ln2/\ln(1+\eps),0]$, the closest grid point $\lambda_\eps$ satisfies
$
    \labs{\lambda-\lambda_\eps}\le \frac{\eps}{4}
$.
By the Lipschitz dependence of $M_\lambda$ on the fairness parameter $\lambda$ (see \Cref{lem:M_lipschitz_in_alpha}), this changes $G_\lambda(p)$ by at most $\eps/4$, uniformly over all prices $p$. Also, if $\lambda$ is smaller than $-\ln2/\ln(1+\eps)$, we can approximate $G_\lambda(p)$ by $G_{-\infty}(p)$, again with an error of order $\eps$ (\Cref{lem:rawls_approx}). 
Hence up to an additive error of order $\eps$, it is sufficient to control the optimization error uniformly over no more than $10/\eps^2 +1 $ values of $\lambda$. This is achieved thanks to our fixed-$\lambda$ PAC results and a union bound.
We defer the full proof to Appendix~\ref{app:PAC:guarantee}.
\end{proofsketch}

\subsection{PAC Lower Bound}

We now prove that the $\varepsilon^{-2}$ dependence in the preceding PAC upper
bounds is unavoidable, already for any fixed fairness level.

\begin{restatable}[PAC Lower Bound for Fair Bilateral Trade]{theorem}{paclowerbound}
\label{thm:pac_lower_bound}
Fix any fairness level $\lambda\in[-\infty,0]$.
There exists a universal constant $c_{\mathrm{PAC}}>0$ such that the following holds.
For every $\varepsilon>0$ and $\delta\in(0,1/4)$ satisfying
$
    \varepsilon
    \le
    \frac{4-\sqrt2-\sqrt3}{16\cdot 10^4}
$,
let a pure-exploration algorithm stop at an a.s. finite stopping time $\tau$
and output a recommendation $\widehat p_\lambda$.
If the algorithm is $(\varepsilon,\delta)$-PAC (as in~\eqref{eq:pointwise_PAC}), then there exists an independent i.i.d.\ bilateral-trade
instance $\nu$ such that
\[
    \E_\nu[\tau]
    \ge
    \frac{c_{\mathrm{PAC}}}{\varepsilon^2}
    \ln\frac{1}{4\delta}.
\]
Moreover, one may take
\[
    c_{\mathrm{PAC}}
    =
    \frac{1991(4-\sqrt2-\sqrt3)^2}{16384\cdot 10^8}.
\]
\end{restatable}
\begin{proofsketch}
The proof proceeds by turning any PAC algorithm with a stopping time into a test between two hard instances.
The hard instances are the same bilateral-trade instances used for the regret lower bound (\Cref{thm:regret_lower_bound} below).
The information-theoretic step combines Bretagnolle-Huber inequality with a KL bound for the stopped observation history, which controls the KL divergence by $c_{\mathrm{info}}\eta^2\E[\tau]$, where $\eta$ is the perturbation parameter.
We defer the full proof of \Cref{thm:pac_lower_bound} to
Appendix~\ref{app:proof_PAC_lower}.
\end{proofsketch}


\section{Regret Guarantees}
\label{s:regret_upper_bound}

We now convert the $\lambda$-free bound on the optimization error into regret guarantees.
The regret algorithm (\Cref{algo:etc}) receives the target fairness level $\lambda$ together with the horizon $T$.
It is an explore then commit (ETC) algorithm whose
exploration length depends only on $T$.
After exploration, the algorithm applies the same two-case rule as \Cref{prop:lambda_free_exploration_optimization}: it computes the finite-$\lambda$ recommendation when this has $\lambda$-free guarantees at the chosen budget, and otherwise computes the Rawls recommendation.

\begin{algorithm}[H]
\DontPrintSemicolon
\textbf{Input:} fairness level $\lambda \in [-\infty,0]$, horizon $T \ge 4$\;
Set
\[
    n_T
    \coloneqq
    \min\lcb{
        T,
        \max\lcb{
            4,
            \lce{64T^{2/3}\ln T}
        }
    }.
\]
Run Algorithm~\ref{algo:collect_samples} with exploration length $n_T$\;
\eIf{$\lambda\in(-\infty,0]$ and $1-\lambda\le\sqrt{n_T}$}{
    Compute $\widehat p_{\lambda}(n_T)$ and set $\widehat p_T\coloneqq\widehat p_{\lambda}(n_T)$\;
}{
    Compute $\widehat p_{-\infty}(n_T)$ and set $\widehat p_T\coloneqq\widehat p_{-\infty}(n_T)$\;
}
\For{$t=n_T+1,\dots,T$}{
    Post price $P_t\coloneqq \widehat p_T$
}
\caption{$\mathrm{ETC}_\lambda(T)$}
\label{algo:etc}
\end{algorithm}

The next theorem then provides the regret bound, of order $\widetilde{O}(T^{2/3})$, thus extending \cite{bachoc2024fairNeurIPS} from Rawls fairness to the entire Rawls-to-Nash family of fairness objectives.

\begin{restatable}[ETC Regret Upper Bound]{theorem}{regretupperbound}
\label{thm:regret_etc}
For every fairness level $\lambda\in[-\infty,0]$ and every horizon $T\ge 4$, Algorithm~\ref{algo:etc} satisfies
\[
    R_\lambda\brb{\mathrm{ETC}_\lambda(T),T}
    \le
    2n_T+1
    \le
    128T^{2/3}\ln T+3,
\]
where $n_T$ is the deterministic exploration length defined in Algorithm~\ref{algo:etc}.
\end{restatable}
\begin{proofsketch}
With confidence level $1/T$, \Cref{prop:lambda_free_exploration_optimization} gives a $\lambda$-free optimization guarantee for the price committed by Algorithm~\ref{algo:etc}.
The choice $n_T=\lce{64T^{2/3}\ln T}$ makes this optimization guarantee at most $n_T/T$.
The generic fixed-budget-to-ETC conversion (\Cref{lem:generic_etc_from_pac}) then gives regret at most $n_T+n_T+1=2n_T+1$.
We defer the full calculation to Appendix~\ref{app:etc_regret}.    
\end{proofsketch}
We remark that \Cref{algo:etc} assumes that the horizon $T$ is taken as an input parameter.
If the horizon is unknown, the standard doubling trick \citep{cesa2006prediction} gives an anytime version with the same regret guarantees up to a constant factor.

The upper bound in \Cref{thm:regret_etc} is tight up to logarithmic factors.

\begin{restatable}[Regret Lower Bound]{theorem}{regretlowerbound}
\label{thm:regret_lower_bound}
There exist two universal constants $T_0 \in \N$ and $c>0$ such that the following holds.
For every fairness level $\lambda\in[-\infty,0]$, every horizon $T\ge T_0$, and every randomized pricing policy $\alpha$ under Interaction~Protocol~\ref{algo:interaction_protocol}, there exists an independent i.i.d.\ bilateral-trade instance such that
\[
    R_\lambda(\alpha,T)
    \ge
    c T^{2/3}.
\]
For example, one may take $T_0=19$ and $c=10^{-6}$.
\end{restatable}
\begin{proofsketch}
Appendix~\ref{app:lower_bound_template} proves an abstract $\Omega(T^{2/3})$ lower-bound template for revealing-action problems.
Appendix~\ref{app:regret_lower_bound} constructs bilateral-trade instances satisfying that template for every $\lambda\in[-\infty,0]$, where these two ingredients are also combined to get the desired lower bound guarantees.
\end{proofsketch}


\section{Dropping Independence between Seller and Buyer Valuations}
\label{s:linear}

All positive results above use independence between seller and buyer valuations within each round.
The next theorem shows that this assumption is not a technical artifact of the estimators.
Without independence, the two threshold bits may be compatible with different joint valuation laws that induce different optimal fair prices, and no algorithm can distinguish the corresponding instances from feedback alone (even if $(S,B)$ has a bounded density).

\begin{restatable}[Linear Lower Bound without Independence]{theorem}{linearlowerbound}
\label{thm:linear}
There exist two universal constants $L,c_{\rm lin}>0$ such that the following holds.
For every horizon $T\in\N$ and every fairness level $\lambda\in[-\infty,0]$, for any randomized pricing policy $\alpha$ under Interaction~Protocol~\ref{algo:interaction_protocol}, there exists an i.i.d.\ bilateral-trade instance $(S_t,B_t)_{t\in[T]}$, whose common law admits a joint density bounded by $L$, such that
\[
    R_\lambda(\alpha,T)
    \ge
    c_{\rm lin}T.
\]
For example, one may take $L\ceq 10^8/3$ and $c_{\rm lin}\ceq 1/192$.
\end{restatable}
\begin{proofsketch}
The proof constructs two dependent joint distributions of $(S,B)$ with the same observable threshold-feedback process but with different optimal prices for the $\lambda$-fair objective.
Since every algorithm has the same law of observations on the two instances, it must play the same distribution over prices, and one of the two instances incurs a constant per-round regret.
The full construction and the verification of the bounded-density condition are given in Appendix~\ref{app:linear}.
\end{proofsketch}


\section{Computational Costs}

We now briefly discuss the computational costs to run the proposed exploration algorithm (Algorithm~\ref{algo:collect_samples}) and to return the corresponding recommendations. First, we notice that the exploration phase of Algorithm~\ref{algo:collect_samples} stores $n$ triples $(P_t,Y_t,Z_t)$ and therefore has time and memory cost $O(n)$.
After exploration, computing a single finite-$\lambda$ estimate $\widehat G_{\lambda,k}(n)$
(\Cref{def:estimator_G})
naively costs $O(k(n-k))$ arithmetic operations.
Maximizing $\widehat G_{\lambda,k}(n)$ over all $k\in\{2,\dots,n-2\}$ costs $ \sum_{k=2}^{n-2}(k-1)(n-k-1) = O(n^3) $ operations for a fixed value of $\lambda$.
The Rawls estimator is cheaper: computing all values $\widehat G_{-\infty,k}(n)$ naively costs $ \sum_{k=1}^{n-1}\min\{k,n-k\} = O(n^2) $. Thus the fixed-$\lambda$ and $\lambda$-free recommendation rules have $O(n^3)$ offline computational cost.
For the uniform-in-fairness $(\eps,\delta)$-PAC rule, if one wants to output all recommendations on the fairness grid $\Gamma_\eps$ with the required budget $n_{\eps,\delta}$ to achieve the guarantees, the naive cost is
$
    O\brb{|\Gamma_\eps|n_{\eps,\delta}^3} = O\lrb{
        \frac{1}{\eps^8}
        \lrb{
            \ln \frac{1}{\eps^4\delta}
        }^9
    }
$, while an on-demand implementation computes only the recommendation associated with the queried fairness level $\lambda_\eps$, where $\lambda_\eps$ is as in the footnote of \Cref{thm:uniform_pac_grid}, leading to a computational cost of
$
    O(n_{\eps,\delta}^3)
    =
    O\lrb{
        \frac{1}{\eps^6}
        \lrb{
            \ln \frac{C_{\rm unif}}{\eps^4\delta}
        }^9
    }
$.
For the regret algorithm, the same discussion applies with the deterministic exploration length $ n_T = \min\lcb{ T, \max\lcb{ 4, \left\lceil 64T^{2/3}\ln T \right\rceil } } $. Algorithm~\ref{algo:etc} computes only one recommendation after exploration: the finite-$\lambda$ recommendation when $1-\lambda\le\sqrt{n_T}$, and the Rawls recommendation otherwise. Thus, in the finite-$\lambda$ branch, the naive offline cost of computing the committed price is $O(n_T^3) = O(T^2(\ln T)^3)$ in the large-horizon regime where $n_T\simeq T^{2/3}\ln T$. In the Rawls branch, the corresponding cost is $O(n_T^2) = O(T^{4/3}(\ln T)^2)$ in the same regime. The remaining execution of the explore-then-commit policy has linear interaction cost in $T$, since after the committed price has been computed the policy posts the same price at every exploitation round. Analogous considerations extend to the anytime version.

\section{Further Directions}

The paper leaves several natural directions open.
First, the logarithmic factors---and of course the numerical constants---in our upper and lower bounds may not be optimal.
More importantly, the current analysis focuses on a single posted price for each seller/buyer pair.
Extending the Rawls-to-Nash analysis to richer market mechanisms remains an important direction, including weakly budget-balanced posted-price pairs, globally budget-balanced mechanisms, and bargaining mechanisms beyond posted prices.
Such mechanisms may provide additional feedback or flexibility, potentially allowing one to relax the independence assumption between seller and buyer valuations or the i.i.d.\ assumption over time, in analogy with what is possible for gain-from-trade objectives.
These extensions would also come with their own economic and statistical tradeoffs: different mechanisms may enlarge the class of learnable valuation sequences, but may also be less desirable from an economic perspective, and, as already observed in the gain-from-trade literature, additional robustness is likely to deteriorate the sample-complexity and regret guarantees.
We believe that understanding this tradeoff between mechanism-design flexibility, economic desiderata, and statistical learnability is a natural next step.

\section*{Acknowledgments and Disclosure of Funding}
The authors declare no competing interests related to this work.

Part of this work was carried out while RC was affiliated with Politecnico di Milano, where he received support from the FAIR (Future Artificial Intelligence Research) project, funded by the NextGenerationEU program within the PNRR-PE-AI scheme (M4C2, investment 1.3, line on Artificial Intelligence), from the MUR PRIN grant 2022EKNE5K (Learning in Markets and Society), and from the EU Horizon CL4-2022-HUMAN-02 RIA under grant agreement 101120237, project ELIAS (European Lighthouse of AI for Sustainability). The work was completed while RC was affiliated with the University of Bristol.
EK acknowledges the support of the French National Research Agency (ANR) in the framework of the PEPR IA FOUNDRY project (ANR-23-PEIA-0003).

\clearpage
\appendix
\addcontentsline{toc}{section}{Appendix}


\section[Why the Rawls-to-Nash Fairness Family]
{Why the Rawls-to-Nash Fairness Family?\\ A Common-Value-Scale Interpretation of Fair Gains}
\label{app:why_Holder?}

In this section, we provide an axiomatic justification of why one may want to adopt one of the $\gft_{\lambda}$ for $\lambda \le 0$ as a sensible metric to measure performance in bilateral trade.
Suppose that we want to assign a scalar score to a trade that captures both the size of the gains generated by the trade and the fairness of their division, which we call $\gft_{\mathrm{fair}}$.
When a trade happens at price $p$, if the seller has valuation $s$ and the buyer has valuation $b$, they obtain gains $p-s$ and $b-p$ respectively, while, if the trade fails to happen, they obtain nothing. 
Hence, we would like our metric to satisfy
\begin{align}
    \label{eq:gft_fair_formula}
    \gft_{\mathrm{fair}}(p,s,b)
    \coloneqq
    \begin{cases}
    M(p-s,b-p), & \textrm{if } s \le p \le b,\\[2pt]
    0, & \textrm{otherwise},
    \end{cases}
\end{align}
where $M(x,y)$ is a scalar score assigned to the seller-buyer gain pair $(x,y)= (p-s,b-p)$.

The function $M$ is meant to compress a two-person outcome into a single number, and since the pair of realized gains $(x,y)$ is summarized by the scalar $M(x,y)$, we want this scalar to have a clear interpretation.
The interpretation we choose is that $M(x,y)$ is the common gain $m$ such that the equal-gain outcome $(m,m)$ is judged equivalent, according to the fair-gain score, to the (possibly) unequal outcome $(x,y)$.

To make this summary score interpretable, we postulate the existence of a common social-welfare scale
\[
    \varphi \colon (0,\infty)\to \bbR,
\]
which assigns to each individual gain $r>0$ a welfare value $\varphi(r)$.
Fairness requires a common scale for the two parties: the same gain must have the same welfare value for the seller and for the buyer.

We assume that the welfare scale is additive: if two individuals are assigned welfare values $v$ and $w$ on the common scale, then the social-welfare value of the joint outcome is $v+w$.
Then, the role of $\varphi$ is precisely to transform raw gains into welfare values on a scale where interpersonal comparisons and additive aggregation are meaningful.

It is natural to assume that $\varphi$ is strictly increasing and continuous.
Strict monotonicity means that larger individual gains have larger welfare value on this scale.
Continuity means that small changes in positive realized gains should not induce discontinuous jumps in their welfare evaluation.

Now, suppose that the seller receives a positive gain $x$ and the buyer receives a positive gain $y$, so that the corresponding welfare values are $\varphi(x)$ and $\varphi(y)$.
By additivity in the social-welfare scale, the joint outcome $(x,y)$ on the common scale has social welfare
\[
    \varphi(x)+\varphi(y).
\]
Recalling that we wanted to define the fair-gain score $M(x,y)$ as the common gain $m$ such that the equal-gain outcome $(m,m)$ has the same social-welfare value on this scale as the unequal outcome $(x,y)$, we see that $m$ must satisfy
\[
    \varphi(m)+\varphi(m)
    =
    \varphi(x)+\varphi(y),
\]
and hence
\[
    M(x,y)
    =
    m
    =
    \varphi^{-1}\lrb{\frac{\varphi(x)+\varphi(y)}{2}}.
\]
We formalize this as the common-value-scale axiom.
\begin{axiom}[Common-Value-Scale]
    \label{axiom:common-value-scale}
    There exists a strictly increasing and continuous function $\varphi \colon (0,\infty)\to \bbR$ such that, for every pair of positive gains $x,y>0$, the score $M$ admits the following representation
    \[
        M(x,y)
        =
        \varphi^{-1}\lrb{\frac{\varphi(x)+\varphi(y)}{2}}.
    \]
    Furthermore, the score $M$ is continuous as a function from $\bbR_+^2$ to $[0,+\infty)$.
\end{axiom}
The representation above is imposed only on strictly positive gains, because $\varphi$ is assumed defined only on $(0,+\infty)$.
Since $\varphi$ is continuous and strictly increasing, this ensures automatically that $M$ is continuous on the interior of $[0,+\infty)^2$.
However, the representation does not determine the behavior of $M$ at boundary points such as $(x,0)$ or $(0,y)$.
For this reason, we imposed separately that $M$ is continuous on the whole quadrant $\bbR_+^2$.
Again, this is a natural regularity condition: small changes in realized gains should lead to small changes in the score, including when one of the gains is close to zero.

The common-value-scale axiom gives an interpretation of the score $M(x,y)$ as an equal-gain equivalent.
However, it does not yet specify how the score should react when the whole trade is rescaled.
Suppose that two trades generate gain pairs $(x,y)$ and $(tx,ty)$, for some $t>0$.
These two gain pairs induce the same split of the surplus: in both cases, the seller receives the same fraction of the total gain and the buyer receives the same fraction of the total gain.
The only difference is the size of the total gain.

Since $M$ is meant to be a fair-gain score, and not a pure fairness index, it should account for this change in scale.
We model this by requiring homogeneity of degree one: if the total surplus is multiplied by $t$ while the split remains unchanged, then the score is also multiplied by $t$.\footnote{Intuitively, the score is linear in the size of the cake, holding fixed how the cake is divided.}
This requirement also makes the score compatible with changes of units.
For example, measuring gains in cents rather than euros multiplies both gains by the same constant, and the score should be multiplied by the same constant without changing the underlying evaluation of the split.

We formalize this as the scale-covariance axiom.
\begin{axiom}[Scale Covariance]
    \label{axiom:scaling}
    For every $t\ge 0$ and every $x,y\ge 0$,
    \[
        M(tx,ty)
        =
        tM(x,y).
    \]
\end{axiom}

The scale-covariance axiom is what separates the size of the trade from the way the gains are split.
Indeed, for every $(x,y)\neq (0,0)$, applying scale covariance to the normalized pair
$
    \lrb{
        \frac{x}{x+y},
        \frac{y}{x+y}
    }
$
with scaling factor $x+y$ gives
\[
    M(x,y)
    =
    (x+y)
    M\lrb{
        \frac{x}{x+y},
        \frac{y}{x+y}
    }.
\]
Thus, $M$ admits the factorization
\[
    M(x,y)
    =
    (x+y)
    F\lrb{
        \frac{x}{x+y},
        \frac{y}{x+y}
    },
\]
where
\[
    F(u,v)\coloneqq M(u,v),
    \qquad (u,v)\in\Delta_2,
\]
and
\[
    \Delta_2\coloneqq \bcb{(u,v)\in\bbR_+^2:u+v=1}
\]
is the two-dimensional simplex.

This factorization cleanly separates the efficiency part of the score from the split-dependent part.
The first factor, $x+y$, is the total gain generated by the trade, and therefore captures the efficiency of the trade.
The second factor
$
    F\lrb{
        \frac{x}{x+y},
        \frac{y}{x+y}
    }
$
depends only on the normalized split of the total gain and it is unchanged if both gains are multiplied by the same positive constant.
In this precise sense, $F$ is a scale-free component of the score: it is oblivious to the size of the cake and sees only how the cake is divided.

Consequently, if $M$ is meant to combine efficiency and fairness, then the normalized factor $F=M|_{\Delta_2}$ is the only component in the factorization that can encode fairness considerations: $F$ depends only on the proportions in which the total gain is divided and is invariant under common rescalings of the gains.
The factor $x+y$, instead, captures the size of the surplus but contains no information about how this surplus is split.

In terms of the bilateral-trade variables, whenever $s\le p\le b$ and $b>s$, we have
\[
    x+y
    =
    (p-s)+(b-p)
    =
    b-s
    \eqqcolon
    \gft(p,s,b),
\]
where $\gft(p,s,b)$ is the well-known gain from trade defined in \cite{cesa2024bilateralMOR}.
Therefore, scale covariance yields
\[
    \gft_{\mathrm{fair}}(p,s,b)
    =
    \gft(p,s,b)
    M\lrb{
        \frac{p-s}{b-s},
        \frac{b-p}{b-s}
    }
    =
    \gft(p,s,b)
    F\lrb{
        \frac{p-s}{b-s},
        \frac{b-p}{b-s}
    }.
\]
Thus, the scale-covariance axiom leads to the revealing representation
\[
    \gft_{\mathrm{fair}}
    =
    \textrm{efficiency}
    \times
    \textrm{fairness of the split},
\]
and this representation clarifies where fairness requirements should be imposed, i.e., on the function $F$. 
In particular, consider the boundary point $(1,0)\in\Delta_2$.
Since $\Delta_2$ is the space of normalized splits, this point represents a division of the cake in which the first party receives the whole cake and the other receives none of it.
Therefore, if $F$ is meant to judge the fairness of a split in which one party takes all the cake and the other takes none of it, it is natural to assign zero fairness to this (completely unfair) split.
Thus, we impose
\[
    F(1,0)=0.
\]
Since $F=M|_{\Delta_2}$, this condition can equivalently be written as
\[
    M(1,0)=0.
\]

We formalize this requirement by asking the fairness score to vanish on one-takes-all normalized splits.
\begin{axiom}[Zero Fairness of the One-Takes-All Split]
    \label{axiom:unfair}
    The fairness score of a normalized split in which one party receives the whole surplus and the other receives none of it is zero:
    \[
        M(1,0)=0.
    \]
\end{axiom}

We now show that \Cref{axiom:common-value-scale,axiom:scaling,axiom:unfair} force the fair-gain score to belong to the family introduced above.
The common-value-scale axiom identifies the fairness score with a quasi-arithmetic mean of the seller's and buyer's gains, while scale covariance forces this mean to be homogeneous.
The next proposition makes the resulting characterization precise.

\begin{proposition}[Characterization of Admissible Fair-Gain Scores]
    \label{prop:characterization}
    Suppose that $M\colon \bbR_+^2\to[0,+\infty)$ satisfies the common-value-scale axiom, the scale-covariance axiom, and the zero fairness of the one-takes-all split axiom.
    Then there exists $\lambda\in (-\infty,0]$ such that, for every $x,y\ge 0$,
    \[
        M(x,y)
        =
        M_\lambda(x,y),
    \]
    where, for $\lambda<0$,
    \[
        M_\lambda(x,y)
        =
        \lrb{
            \frac{x^\lambda+y^\lambda}{2}
        }^{1/\lambda}
    \]
    on $(0,\infty)^2$, extended continuously to the boundary of $\bbR_+^2$, and where
    \[
        M_0(x,y)=\sqrt{xy}.
    \]
    Consequently, if $\gft_{\mathrm{fair}}$ satisfies \Cref{eq:gft_fair_formula}, there exists $\lambda \in (-\infty,0]$ such that, for any $p,s,b \in \bbR_+^3$,
    \[
        \gft_{\mathrm{fair}}(p,s,b)
        =
        \gft_\lambda(p,s,b).
    \]
\end{proposition}

\begin{proof}
    By the common-value-scale axiom, for positive gains $x,y>0$,
    \[
        M(x,y)
        =
        \varphi^{-1}\lrb{
            \frac{\varphi(x)+\varphi(y)}{2}
        }.
    \]
    Thus, on $(0,\infty)^2$, $M$ is a quasi-arithmetic mean generated by the strictly increasing function $\varphi$.
    By the scale-covariance axiom, this quasi-arithmetic mean is homogeneous of degree one.

    We use the standard characterization of homogeneous quasi-arithmetic means on the positive numbers;
    see, e.g., \citet[Chapter~3]{hardy1952inequalities}.
    It states that the generator of any such mean is affine-equivalent either to
    $
        t\mapsto t^\lambda
    $, for some $\lambda \neq 0$,
    or to
    $
        t\mapsto \ln t
    $.
    Equivalently, the generator can be written either as
    $
        \varphi(t)=a t^\lambda+b
    $,
    or as
    $
        \varphi(t)=a\ln t+b,
    $
    for some $\lambda\neq0$, $a\neq0$, and $b\in\bbR$, with the sign of $a$ chosen so that $\varphi$ is strictly increasing.
    In particular, when $\lambda<0$, one takes $a<0$.
    Since affine transformations of the generator do not change the induced quasi-arithmetic mean, it follows that, on $(0,\infty)^2$,
    \[
        M(x,y)
        =
        \begin{cases}
        \lrb{
            \dfrac{x^\lambda+y^\lambda}{2}
        }^{1/\lambda}, & \textrm{if }\lambda \in \mathbb{R}\backslash\{0\},\\[1.2ex]
        \sqrt{xy}, & \textrm{in the logarithmic case.}
        \end{cases}
    \]
    We identify the logarithmic case with the limiting order $\lambda=0$.

    It remains to use the zero fairness of the one-takes-all split axiom.
    If $\lambda>0$, then the continuous extension of the power mean to the boundary satisfies
    \[
        M_\lambda(1,0)
        =
        2^{-1/\lambda}
        >
        0,
    \]
    contradicting the axiom.
    Hence positive orders are excluded.
    For $\lambda=0$, the geometric mean satisfies $M_0(1,0)=0$.
    For $\lambda<0$, the continuous extension to the boundary also satisfies $M_\lambda(1,0)=0$.
    Therefore the admissible finite orders are exactly $\lambda \in (-\infty,0]$.
\end{proof}

\subsection{The Rawlsian Limit}
We observe that \Cref{prop:characterization} rules out the Rawlsian score introduced in \cite{bachoc2024fairNeurIPS}, which corresponds to
\[
    M(x,y)=\min\lrb{x,y}.
\]
The reason is that \Cref{axiom:common-value-scale} represents $M$ through a strictly increasing welfare scale $\varphi$, and this makes the induced quasi-arithmetic mean strictly increasing in each coordinate on $(0,+\infty)^2$.
By contrast, the Rawlsian score is only non-decreasing in each coordinate.
For example, if $x<y$, then increasing the larger coordinate $y$ does not change the value of $\min\lrb{x,y}$.

Nevertheless, the Rawlsian score arises naturally as the limiting case of increasingly inequality-averse H\"older means.
Indeed, for every $x,y\ge 0$,
$
    \lim_{\lambda\to-\infty} M_\lambda(x,y)
    =
    \min\lrb{x,y}.
$
Thus, defining
$
    M_{-\infty}(x,y)\coloneqq \min\lrb{x,y},$
we obtain the full closed family
$
    \lambda\in[-\infty,0]
$.

Consequently, the fair-gain scores defined through \Cref{eq:gft_fair_formula} and satisfying \Cref{axiom:common-value-scale,axiom:scaling,axiom:unfair}, together with their Rawlsian limiting case, are precisely the $\lambda$-fair gain from trade objectives already introduced in \Cref{s:setting}
\[
    \gft_{\lambda}(p,s,b)
    =
    \I\lcb{s \le p \le b} M_\lambda\brb{(p-s)_+,(b-p)_+},
\]
for some $\lambda\in[-\infty,0]$.


\section{Structural Properties of H\"older Fair-Gain Scores}
\label{app:holder_structural}

This appendix collects the structural and regularity properties of the H\"older fair-gain scores used throughout the reconstruction, discretization, and concentration arguments.


\subsection{Basic Properties of H\"older Means}

\begin{lemma}[Basic Properties, Bounds, and Order Monotonicity of $M_\lambda$]
\label{lem:holder_basic_bounds}
\label{lem:M-bounds}
For every $\lambda\in[-\infty,0]$, the map
\[
    M_\lambda:\bbR_+^2\to\bbR_+
\]
is symmetric, positively homogeneous, and coordinatewise nondecreasing.

Moreover, for every $-\infty\le \lambda\le \mu\le 0$ and every $x,y\ge0$,
\[
    M_\lambda(x,y)\le M_\mu(x,y).
\]
In particular, for every $\lambda\in[-\infty,0]$ and every $x,y\ge0$,
\[
    \min(x,y)\le M_\lambda(x,y)\le \sqrt{xy}.
\]
Finally, for every $a\ge0$ and every $\lambda\in[-\infty,0]$,
\[
    M_\lambda(a,a)=a,
    \qquad
    M_\lambda(a,0)=M_\lambda(0,a)=0.
\]
\end{lemma}

\begin{proof}
Symmetry and positive homogeneity are immediate from the definition of $M_\lambda$.
The identities $   M_\lambda(a,a)=a$ and
$M_\lambda(a,0)=M_\lambda(0,a)=0$
are also immediate from the definition and from our boundary convention.

Coordinatewise monotonicity is immediate for $\lambda=-\infty$ and $\lambda=0$.
For $\lambda\in(-\infty,0)$, the claim on $(0,+\infty)^2$ follows as follows: if $x$ increases, then $x^\lambda$ decreases, and since $z\mapsto z^{1/\lambda}$ is decreasing on $(0,+\infty)$, the value of $M_\lambda(x,y)$ cannot decrease.
The same argument applies to the second coordinate, and the boundary follows by continuity.

If $xy=0$, then
$    M_\lambda(x,y)=0$
for every $\lambda\in[-\infty,0]$,
and all the claimed inequalities are immediate.
Hence assume $x,y>0$.

We first prove order monotonicity for finite parameters.
Set
\[
    a\ceq \ln x,
    \qquad
    b\ceq \ln y,
\]
and define, for $r\in\bbR$,
\[
    h(r)\ceq \ln\lrb{\frac{e^{ra}+e^{rb}}{2}}.
\]
Then $h(0)=0$, and $h$ is convex because
\[
    h''(r)
    =
    \frac{
        a^2 e^{ra}+b^2 e^{rb}
    }{
        e^{ra}+e^{rb}
    }
    -
    \lrb{
        \frac{
            a e^{ra}+b e^{rb}
        }{
            e^{ra}+e^{rb}
        }
    }^2
    \ge0.
\]
For $r\ne0$,
\[
    \ln M_r(x,y)=\frac{h(r)}{r},
\]
while
\[
    \ln M_0(x,y)
    =
    \ln\sqrt{xy}
    =
    \frac{a+b}{2}
    =
    h'(0).
\]
Since $h$ is convex and $h(0)=0$, the slope $r\mapsto h(r)/r$ is nondecreasing in $r$ on
$(-\infty,0)$, and its limit as $r\uparrow0$ is $h'(0)$.
Therefore, for every $-\infty<\lambda\le\mu\le0$,
\[
    \ln M_\lambda(x,y)\le \ln M_\mu(x,y).
\]
Exponentiating gives
$
    M_\lambda(x,y)\le M_\mu(x,y).
$

It remains only to include $\lambda=-\infty$.
For every finite $\mu\in(-\infty,0]$,
\[
    M_{-\infty}(x,y)=\min(x,y)\le M_\mu(x,y).
\]
Indeed, the claim is immediate for $\mu=0$.
For $\mu<0$, since $t\mapsto t^\mu$ is decreasing on $(0,+\infty)$,
\[
    x^\mu\le \min(x,y)^\mu,
    \qquad
    y^\mu\le \min(x,y)^\mu.
\]
Thus
\[
    \frac{x^\mu+y^\mu}{2}
    \le
    \min(x,y)^\mu.
\]
Raising both sides to the power $1/\mu<0$ reverses the inequality and gives
\[
    M_\mu(x,y)
    =
    \lrb{\frac{x^\mu+y^\mu}{2}}^{1/\mu}
    \ge
    \min(x,y).
\]
Taking $\mu=0$ in the order-monotonicity statement gives $M_\lambda(x,y)\le M_0(x,y)=\sqrt{xy}$, while taking $\lambda=-\infty$ gives $\min(x,y)\le M_\lambda(x,y)$.
\end{proof}

\begin{lemma}[One-Dimensional Concavity of Normalized Means]
\label{lem:normalized-mean-concavity}
Fix $\lambda\in[-\infty,0]$.
The map
\[
    r\mapsto M_\lambda(1,r)
\]
is nondecreasing and concave on $[1,+\infty)$.
\end{lemma}

\begin{proof}
The claim is immediate for $\lambda=-\infty$, since $M_{-\infty}(1,r)=1$ for $r\ge1$.
For $\lambda=0$, it follows from $M_0(1,r)=\sqrt r$.

Let $\lambda\in(-\infty,0)$ and set
\[
    m_\lambda(r)\ceq M_\lambda(1,r)
    =
    \lrb{\frac{1+r^\lambda}{2}}^{1/\lambda}.
\]
Then
\[
    m_\lambda'(r)
    =
    \frac12 r^{\lambda-1}
    \lrb{\frac{1+r^\lambda}{2}}^{1/\lambda-1}
    \ge0,
\]
and
\[
    m_\lambda''(r)
    =
    \frac{\lambda-1}{4}
    r^{\lambda-2}
    \lrb{\frac{1+r^\lambda}{2}}^{1/\lambda-2}
    \le0.
\]
Hence $m_\lambda$ is nondecreasing and concave on $[1,+\infty)$.
\end{proof}

\begin{lemma}[Differential Identities and Kernel Representation]
\label{lem:holder_differential_identities}
For every $\lambda\in(-\infty,0]$, the map $M_\lambda$ is twice continuously differentiable on
$(0,+\infty)^2$ and, for every $x,y>0$,
\[
    \partial_x M_\lambda(x,y)
    =
    \frac12 M_\lambda(x,y)^{1-\lambda}x^{\lambda-1},
    \qquad
    \partial_y M_\lambda(x,y)
    =
    \frac12 M_\lambda(x,y)^{1-\lambda}y^{\lambda-1}.
\]
Moreover,
\[
    \partial_{xy}^2 M_\lambda(x,y)
    =
    H_\lambda(x,y).
\]
Consequently, for every $x,y\ge0$,
\[
    M_\lambda(x,y)
    =
    \int_0^x\int_0^y H_\lambda(u,v)\,\dif v\,\dif u.
\]
\end{lemma}

\begin{proof}
For $\lambda=0$, the first-order identities follow directly from
$
    M_0(x,y)=\sqrt{xy}.
$
For $\lambda\in(-\infty,0)$, differentiating
\[
    M_\lambda(x,y)
    =
    \lrb{\frac{x^\lambda+y^\lambda}{2}}^{1/\lambda}
\]
gives
\[
    \partial_x M_\lambda(x,y)
    =
    \frac12 M_\lambda(x,y)^{1-\lambda}x^{\lambda-1},
    \qquad
    \partial_y M_\lambda(x,y)
    =
    \frac12 M_\lambda(x,y)^{1-\lambda}y^{\lambda-1}.
\]
The same formula is valid at $\lambda=0$ because
\[
    \frac12 M_0(x,y)x^{-1}
    =
    \frac12\sqrt{\frac{y}{x}},
    \qquad
    \frac12 M_0(x,y)y^{-1}
    =
    \frac12\sqrt{\frac{x}{y}}.
\]

Differentiating once more, for every $\lambda\in(-\infty,0]$,
\[
    \partial_{xy}^2M_\lambda(x,y)
    =
    \frac{1-\lambda}{4}
    M_\lambda(x,y)^{1-2\lambda}
    x^{\lambda-1}
    y^{\lambda-1}.
\]
Since
\[
    x^{\lambda-1}y^{\lambda-1}
    =
    \lrb{xy}^{\lambda-1}
    =
    M_0(x,y)^{2\lambda-2},
\]
we obtain
\[
    \partial_{xy}^2M_\lambda(x,y)
    =
    \frac{1-\lambda}{4}
    \lrb{\frac{M_\lambda(x,y)}{M_0(x,y)}}^{1-2\lambda}
    \frac1{M_0(x,y)}
    =
    H_\lambda(x,y).
\]

We now prove the integral representation. If $xy=0$, then both sides are equal to zero. Assume
$x,y>0$. Let $a\in(0,x)$ and $b\in(0,y)$. By the fundamental theorem of calculus,
\begin{align*}
    M_\lambda(x,y)
&=
    M_\lambda(a,y)
    +
    \int_a^x \partial_u M_\lambda(u,y)\,\dif u
\\
&=
    M_\lambda(a,y)
    +
    \int_a^x
    \lrb{
        \partial_u M_\lambda(u,b)
        +
        \int_b^y \partial_{uv}^2M_\lambda(u,v)\,\dif v
    }
    \dif u
\\
&=
    M_\lambda(a,y)+M_\lambda(x,b)-M_\lambda(a,b)
    +
    \int_a^x\int_b^y H_\lambda(u,v)\,\dif v\,\dif u.
\end{align*}
Letting $a\downarrow0$ and $b\downarrow0$, the three boundary terms vanish by the boundary
convention and continuity of $M_\lambda$ on $\bbR_+^2$. Since $H_\lambda\ge0$, the integrals over the rectangles $[a,x]\times[b,y]$ increase to the integral over $(0,x]\times(0,y]$ as $a\downarrow0$ and $b\downarrow0$. The monotone convergence theorem gives 
\[
    \int_a^x\int_b^y H_\lambda(u,v)\,\dif v\,\dif u
    \longrightarrow
    \int_0^x\int_0^y H_\lambda(u,v)\,\dif v\,\dif u.
\]
This proves the representation.
\end{proof}


\subsection{Shape and Price Regularity of the Fair Gain from Trade}

\begin{lemma}[Shape of the Fair Gain from Trade]\label{lem:single_trade_shape}
\label{lem:gft-basic-nonpos}
\label{lem:single-interval-monotonicity}
Fix $\lambda\in[-\infty,0]$ and $s,b\in[0,1]$.

If $b<s$, then
\[
    \gft_\lambda(p,s,b)=0
    \qquad
    \forall p\in[0,1].
\]
If $s\le b$, then
\[
    \gft_\lambda(p,s,b)=0
    \qquad
    \forall p\notin[s,b],
\]
and
\[
    \gft_\lambda(p,s,b)=M_\lambda(p-s,b-p)
    \qquad
    \forall p\in[s,b].
\]
Moreover, if $s<b$, then the map
\[
    p\mapsto \gft_\lambda(p,s,b)
\]
is nondecreasing on $\lsb{s,(s+b)/2}$ and nonincreasing on $\lsb{(s+b)/2,b}$.
\end{lemma}

\begin{proof}
The support and representation claims follow directly from the definition
\[
    \gft_\lambda(p,s,b)
    =
    \I\lcb{s\le p\le b}
    M_\lambda\lrb{(p-s)_+,(b-p)_+}.
\]

It remains to prove the monotonicity statement when $s<b$.
The claim is immediate for $\lambda=-\infty$, since on $[s,b]$ the function is
$
    p\mapsto \min\brb{p-s,b-p}$.

Assume now that $\lambda\in(-\infty,0]$ and $s<p<b$.
Write $
    x\ceq p-s$ and
    $y\ceq b-p$.
By \Cref{lem:holder_differential_identities},
\[
    \partial_x M_\lambda(x,y)
    =
    \frac12 M_\lambda(x,y)^{1-\lambda}x^{\lambda-1},
    \qquad
    \partial_y M_\lambda(x,y)
    =
    \frac12 M_\lambda(x,y)^{1-\lambda}y^{\lambda-1}.
\]
Thus
\[
    \frac{\dif}{\dif p}M_\lambda(p-s,b-p)
    =
    \partial_x M_\lambda(x,y)-\partial_y M_\lambda(x,y).
\]
If $p\le(s+b)/2$, then $x\le y$. Since $\lambda-1<0$, we have
$
    x^{\lambda-1}\ge y^{\lambda-1}$,
and therefore the derivative is nonnegative.
If $p\ge(s+b)/2$, then $x\ge y$, and the derivative is nonpositive.
Continuity handles the endpoints.
\end{proof}

\begin{lemma}[Uniform Price Regularity of Fair-Gain Objectives]
\label{lem:holder_price_regularity}
\label{lem:G_is_Holder}
For every $\lambda\in[-\infty,0]$, every $s,b\in[0,1]$, and every $p,q\in[0,1]$,
\begin{equation} \label{eq:holder:gft}
    \babs{\gft_\lambda(p,s,b)-\gft_\lambda(q,s,b)}
    \le
    \sqrt{\labs{p-q}}.
\end{equation}
In the Rawls case, the sharper bound
\[
    \babs{\gft_{-\infty}(p,s,b)-\gft_{-\infty}(q,s,b)}
    \le
    \labs{p-q}
\]
holds. Consequently, for every $\lambda\in[-\infty,0]$ and every $p,q\in[0,1]$,
\[
    \babs{G_\lambda(p)-G_\lambda(q)}
    \le
    \sqrt{\labs{p-q}},
\]
and $G_{-\infty}$ is $1$-Lipschitz on $[0,1]$.
\end{lemma}

\begin{proof}
The Rawls case is immediate, since
\[
    \gft_{-\infty}(p,s,b)
    =
    \I\lcb{s\le p\le b}\min\brb{p-s,b-p}
    =
    \min\brb{(p-s)_+,(b-p)_+}.
\]
Both maps $p\mapsto(p-s)_+$ and $p\mapsto(b-p)_+$ are $1$-Lipschitz, and the pointwise minimum of two $1$-Lipschitz real-valued functions is again $1$-Lipschitz.
Since $\labs{p-q}\le1$, this also implies
\[
    \babs{\gft_{-\infty}(p,s,b)-\gft_{-\infty}(q,s,b)}
    \le
    \labs{p-q}
    \le
    \sqrt{\labs{p-q}}.
\]

Assume now that $\lambda\in(-\infty,0]$.
If $s\ge b$, then $\gft_\lambda(\cdot,s,b)\equiv0$, and the claim is immediate.
Hence assume $s<b$ and set
$    L\ceq b-s$.
For $t\in[0,L]$, define
\[
    f_\lambda(t)
    \ceq
    \gft_\lambda(s+t,s,b)
    =
    M_\lambda(t,L-t).
\]
Notice that $f_\lambda$ is continuous on $[0,L]$, twice continuously differentiable on $(0,L)$,
satisfies
$    f_\lambda(0)=0=f_\lambda(L)$,
and is symmetric:
$    f_\lambda(t)=f_\lambda(L-t)$,
    for all $t\in[0,L]$.

\medskip
\noindent\emph{Step A: concavity of $f_\lambda$.}
We prove that $f_\lambda$ is concave on $[0,L]$.

First consider $\lambda\in(-\infty,0)$.
For $t\in(0,L)$, by \Cref{lem:holder_differential_identities},
\begin{align*}
    f'_{\lambda}(t)
&=
    \partial_xM_{\lambda}(t,L-t)-\partial_yM_{\lambda}(t,L-t)
\\
&=
    \frac12
    M_{\lambda}(t,L-t)^{1-\lambda}
    \lrb{t^{\lambda-1}-(L-t)^{\lambda-1}}.
\end{align*}
For $r\ge1$, define
\[
    \phi_\lambda(r)
    \ceq
    \lrb{\frac{1+r^{\lambda}}{2}}^{1/\lambda}.
\]
If
\[
    r(t)\ceq\frac{L-t}{t},
\]
then, for $t\in(0,L/2]$,
\[
    r(t)\ge1,
    \qquad
    M_{\lambda}(t,L-t)=t\,\phi_\lambda\brb{r(t)}.
\]
Therefore
\[
    f'_{\lambda}(t)
    =
    \frac12
    \phi_\lambda\brb{r(t)}^{1-\lambda}
    \lrb{1-r(t)^{\lambda-1}}.
\]
Define
\[
    g_\lambda(r)
    \ceq
    \frac12
    \phi_\lambda(r)^{1-\lambda}
    \lrb{1-r^{\lambda-1}},
    \qquad r\ge1.
\]
Then
\[
    f'_{\lambda}(t)=g_\lambda\brb{r(t)}
    \qquad
    \forall t\in(0,L/2].
\]
For $r\ge1$,
\[
    \phi'_\lambda(r)
    =
    \phi_\lambda(r)
    \frac{r^{\lambda-1}}{1+r^{\lambda}}
    >
    0,
\]
and
\[
    \frac{\dif}{\dif r}\lrb{1-r^{\lambda-1}}
    =
    (1-\lambda)r^{\lambda-2}
    >
    0.
\]
Thus both factors in the definition of $g_\lambda$ are nonnegative and increasing, so $g_\lambda$ is
increasing on $[1,+\infty)$. On the other hand,
\[
    r'(t)=-\frac{L}{t^2}<0.
\]
Hence
\[
    f''_{\lambda}(t)
    =
    g'_\lambda\brb{r(t)}\,r'(t)
    \le0
    \qquad
    \forall t\in(0,L/2).
\]
By symmetry of $f_{\lambda}$, the same inequality holds on $(L/2,L)$. Hence
$f_{\lambda}$ is concave on $(0,L)$, and therefore on $[0,L]$ by continuity.

For $\lambda=0$,
$f_0(t)=\sqrt{t(L-t)}$,
and for $t\in(0,L)$,
\[
    f_0''(t)
    =
    -\frac{1}{\sqrt{t(L-t)}}
    -
    \frac{(L-2t)^2}{4\brb{t(L-t)}^{3/2}}
    <
    0.
\]
Thus $f_0$ is concave on $[0,L]$ as well.

\medskip
\noindent\emph{Step B: maximal $h$-increments occur at the boundary.}
Fix $h\in(0,L]$ and define
\[
    F(t)\ceq f_\lambda(t+h)-f_\lambda(t),
    \qquad
    t\in[0,L-h].
\]
If $h=L$, there is nothing to prove. Assume $h\in(0,L)$.
Since $f_\lambda$ is concave, its derivative is nonincreasing on $(0,L)$, and therefore
\[
    F'(t)=f'_\lambda(t+h)-f'_\lambda(t)\le0
    \qquad
    \forall t\in(0,L-h).
\]
By continuity, $F$ is nonincreasing on $[0,L-h]$. Hence
$
    \sup_{t\in[0,L-h]}F(t)
    =
    F(0)
    =
    f_\lambda(h)$.
Using the symmetry of $f_\lambda$ and $f_\lambda(L)=0$, we also get
$
    \inf_{t\in[0,L-h]}F(t)
    =
    F(L-h)
    =
    -f_\lambda(h)$.
Thus, for every $t\in[0,L-h]$,
\begin{equation}
\label{eq:holder_boundary_increment}
    \babs{f_\lambda(t+h)-f_\lambda(t)}
    \le
    f_\lambda(h)
    =
    M_\lambda(h,L-h).
\end{equation}

\medskip
\noindent\emph{Step C: boundary increment bound.}
By \Cref{lem:holder_basic_bounds}, for every $h\in[0,L]$,
\begin{equation}
\label{eq:holder_boundary_increment_bound}
    M_\lambda(h,L-h)
    \le
    M_0(h,L-h)
    =
    \sqrt{h(L-h)}
    \le
    \sqrt h,
\end{equation}
where the last inequality uses $L\le1$.

\medskip
\noindent\emph{Step D: back to $p,q$.}
Assume without loss of generality that $q\le p$.
Define
\[
    \tau(u)
    \ceq
    \min\brb{\max\lrb{u-s,0},L},
    \qquad
    u\in[0,1].
\]
Then
\[
    \gft_\lambda(u,s,b)
    =
    f_\lambda\brb{\tau(u)}
    \qquad
    \forall u\in[0,1].
\]
Let $
    \widetilde h
    \ceq
    \tau(p)-\tau(q)$.
Since $\tau$ is nondecreasing, $\widetilde h\ge0$. Moreover,
\[
    \widetilde h
    =
    \babs{[q,p]\cap[s,b]}
    \le
    p-q.
\]
If $\widetilde h=0$, the claim is immediate. Otherwise, using
\eqref{eq:holder_boundary_increment} with $t=\tau(q)$ and $h=\widetilde h$, and then
\eqref{eq:holder_boundary_increment_bound}, we get
\[
    \babs{\gft_\lambda(p,s,b)-\gft_\lambda(q,s,b)}
    =
    \Babs{f_\lambda\brb{\tau(p)}-f_\lambda\brb{\tau(q)}}
    \le
    M_\lambda\lrb{\widetilde h,L-\widetilde h}
    \le
    \sqrt{\widetilde h}
    \le
    \sqrt{p-q}.
\]
This proves \eqref{eq:holder:gft}.
Taking expectations yields
\[
    \babs{G_\lambda(p)-G_\lambda(q)}
    =
    \Babs{\E\bsb{\gft_\lambda(p,S,B)-\gft_\lambda(q,S,B)}}
    \le
    \E\Bsb{\babs{\gft_\lambda(p,S,B)-\gft_\lambda(q,S,B)}}
    \le
    \sqrt{\labs{p-q}}.
\]
The Rawls Lipschitz claim for $G_{-\infty}$ follows in the same way from the pointwise Rawls
Lipschitz bound.
\end{proof}


\subsection{Comparison Across Fairness Levels}

\begin{lemma}[Approximation by the Rawls Objective]
\label{lem:rawls_approx}
For every $\lambda<0$ and every $p\in[0,1]$,
\begin{equation}
\label{eq:Gslambda:minus:infty}
    0
    \le
    G_\lambda(p)-G_{-\infty}(p)
    \le
    \frac{2^{-1/\lambda}-1}{2}.
\end{equation}
Consequently, for every $\lambda<0$ and every $p\in[0,1]$,
\begin{equation} \label{eq:Gstart:minus:G:lambda:p}
    G_\lambda^\star-G_\lambda(p)
    \le
    G_{-\infty}^\star-G_{-\infty}(p)
    +
    \frac{2^{-1/\lambda}-1}{2}.
\end{equation}
\end{lemma}

\begin{proof}
Fix $\lambda<0$.
For $x,y\ge0$, we first show that
\[
    \min(x,y)
    \le
    M_\lambda(x,y)
    \le
    2^{-1/\lambda}\min(x,y).
\]
The first inequality follows from \Cref{lem:holder_basic_bounds}. For the second, if $xy=0$, the
claim follows from the boundary convention. Assume $x,y>0$. Since one of $x^\lambda$ and
$y^\lambda$ is equal to $\min(x,y)^\lambda$ and the other is nonnegative,
\[
    \frac{x^\lambda+y^\lambda}{2}
    \ge
    \frac{\min(x,y)^\lambda}{2}.
\]
Because $1/\lambda<0$, the map $z\mapsto z^{1/\lambda}$ is decreasing on $(0,+\infty)$, and so
\[
    M_\lambda(x,y)
    =
    \lrb{\frac{x^\lambda+y^\lambda}{2}}^{1/\lambda}
    \le
    \lrb{\frac{\min(x,y)^\lambda}{2}}^{1/\lambda}
    =
    2^{-1/\lambda}\min(x,y).
\]

Therefore, for every $p,s,b\in[0,1]$,
\[
    0
    \le
    \gft_\lambda(p,s,b)-\gft_{-\infty}(p,s,b)
    \le
    \lrb{2^{-1/\lambda}-1}\gft_{-\infty}(p,s,b).
\]
Since
\[
    \gft_{-\infty}(p,s,b)
    =
    \I\lcb{s\le p\le b}\min\brb{(p-s)_+,(b-p)_+}
    \le
    \frac12,
\]
we obtain
\[
    0
    \le
    \gft_\lambda(p,s,b)-\gft_{-\infty}(p,s,b)
    \le
    \frac{2^{-1/\lambda}-1}{2}.
\]
Substituting $s=S$ and $b=B$ and taking expectations proves the first claim \eqref{eq:Gslambda:minus:infty}.

For the second claim \eqref{eq:Gstart:minus:G:lambda:p}, let
\[
    \eta_\lambda\ceq \frac{2^{-1/\lambda}-1}{2}.
\]
For every $q\in[0,1]$, the one-sided approximation gives
\[
    G_\lambda(q)-G_\lambda(p)
    \le
    G_{-\infty}(q)+\eta_\lambda-G_{-\infty}(p).
\]
Taking the supremum over $q\in[0,1]$ yields
\[
    G_\lambda^\star-G_\lambda(p)
    \le
    G_{-\infty}^\star-G_{-\infty}(p)+\eta_\lambda.
\]
This proves the claim.
\end{proof}

\begin{lemma}[Lipschitz Dependence on the Fairness Parameter]
\label{lem:holder_lipschitz_fairness_parameter}
\label{lem:M_lipschitz_in_alpha}
For every $\alpha,\beta\in[0,+\infty)$,
\[
    \sup_{x,y\in[0,1]}
    \labs{M_{-\alpha}(x,y)-M_{-\beta}(x,y)}
    \le
    \labs{\alpha-\beta},
\]
recalling that $M_0$ denotes the geometric mean.
\end{lemma}

\begin{proof}
By symmetry, it is enough to consider $0\le x\le y\le1$.
If $x=0$, then the boundary convention gives
\[
    M_{-\alpha}(0,y)=M_{-\beta}(0,y)=0,
\]
and the claim is immediate.
Hence assume $0<x\le y$.
There exists $q\ge0$ such that $x=ye^{-q}$.
By homogeneity,
\[
    M_{-\alpha}(x,y)
    =
    yM_{-\alpha}(e^{-q},1),
    \qquad
    M_{-\beta}(x,y)
    =
    yM_{-\beta}(e^{-q},1).
\]
Since $y\le1$, it is enough to prove that, for every $q\ge0$,
\[
    \labs{M_{-\alpha}(e^{-q},1)-M_{-\beta}(e^{-q},1)}
    \le
    \labs{\alpha-\beta}.
\]

For $\alpha>0$ and $q\ge0$, define
\[
    m_\alpha(q)
    \ceq
    M_{-\alpha}(e^{-q},1)
    =
    \lrb{\frac{e^{\alpha q}+1}{2}}^{-1/\alpha}.
\]
Also define
\[
    m_0(q)
    \ceq
    M_0(e^{-q},1)
    =
    e^{-q/2}.
\]
We prove that $\alpha\mapsto m_\alpha(q)$ is $1$-Lipschitz uniformly in $q$.

For $\alpha>0$, write
$t\ceq \frac{\alpha q}{2}$.
Since
\[
    \frac{e^{\alpha q}+1}{2}
    =
    e^{\alpha q/2}\cosh\lrb{\frac{\alpha q}{2}},
\]
we have
\[
    m_\alpha(q)
    =
    \exp\lrb{
        -\frac q2
        -
        \frac1\alpha
        \ln\cosh\lrb{\frac{\alpha q}{2}}
    }.
\]
Differentiating with respect to $\alpha$ gives
\[
    \frac{\partial}{\partial \alpha}m_\alpha(q)
    =
    m_\alpha(q)
    \frac{
        \ln\cosh(t)-t\tanh(t)
    }{\alpha^2}.
\]
Therefore
\[
    \labs{
        \frac{\partial}{\partial \alpha}m_\alpha(q)
    }
    =
    m_\alpha(q)
    \frac{
        t\tanh(t)-\ln\cosh(t)
    }{\alpha^2}.
\]
Now the function
$
    h(t)
    \ceq
    t\tanh(t)-\ln\cosh(t)
$
satisfies $h(0)=0$ and
$
    h'(t)
    =
    t\,\operatorname{sech}^2(t)
    \le
    t$.
Hence
\[
    0\le h(t)\le\frac{t^2}{2}.
\]
Using also $m_\alpha(q)\le m_0(q)=e^{-q/2}$ from \Cref{lem:holder_basic_bounds}, we obtain
\[
    \labs{
        \frac{\partial}{\partial \alpha}m_\alpha(q)
    }
    \le
    e^{-q/2}
    \frac{t^2}{2\alpha^2}
    =
    e^{-q/2}
    \frac{q^2}{8}.
\]
Finally,
\[
    \sup_{q\ge0}e^{-q/2}\frac{q^2}{8}
    =
    \frac{2}{e^2}
    \le1.
\]
Thus
\[
    \labs{
        \frac{\partial}{\partial \alpha}m_\alpha(q)
    }
    \le1
    \qquad
    \forall \alpha>0,\ \forall q\ge0.
\]
Since $m_\alpha(q)\to m_0(q)$ as $\alpha\downarrow0$, it follows that
\[
    \labs{m_\alpha(q)-m_\beta(q)}
    \le
    \labs{\alpha-\beta}
    \qquad
    \forall \alpha,\beta\in[0,+\infty),\ \forall q\ge0.
\]
The reduction above proves the claim.
\end{proof}


\subsection{Kernel Identities and Pointwise Bounds}

\begin{lemma}[Basic Kernel Bound]
\label{lem:kernel_basic_bound}
For every $\lambda\in(-\infty,0]$ and every $x,y>0$,
\[
    H_\lambda(x,y)
    \le
    \frac{1-\lambda}{4\sqrt{xy}}.
\]
\end{lemma}

\begin{proof}
By definition,
\[
    H_\lambda(x,y)
    =
    \frac{1-\lambda}{4}
    \lrb{\frac{M_\lambda(x,y)}{M_0(x,y)}}^{1-2\lambda}
    \frac1{M_0(x,y)}.
\]
By \Cref{lem:holder_basic_bounds},
$
    M_\lambda(x,y)\le M_0(x,y)$.
Since $1-2\lambda\ge1$ and $M_0(x,y)=\sqrt{xy}$, the claim follows.
\end{proof}

\begin{lemma}[Kernel Imbalance Bound]
\label{lem:kernel_imbalance_bound}
\label{cor:kernel_imbalance_bound}
For every $\lambda\in(-\infty,0]$ and every $x,y>0$,
\[
    H_\lambda(x,y)
    \le
    \frac{1-\lambda}{\sqrt{xy}}
    \lrb{
        1-\labs{\frac{x-y}{x+y}}
    }^{-\lambda}.
\]
\end{lemma}

\begin{proof}
By definition,
\[
    H_\lambda(x,y)
    =
    \frac{1-\lambda}{4}
    \lrb{\frac{M_\lambda(x,y)}{M_0(x,y)}}^{1-2\lambda}
    \frac1{M_0(x,y)}.
\]
Since $M_0(x,y)=\sqrt{xy}$, it is enough to prove that
\begin{equation}
\label{eq:mean_ratio_imbalance_bound}
    \lrb{\frac{M_\lambda(x,y)}{M_0(x,y)}}^{1-2\lambda}
    \le
    4
    \lrb{
        1-\labs{\frac{x-y}{x+y}}
    }^{-\lambda}.
\end{equation}

If $\lambda=0$, the left-hand side of \eqref{eq:mean_ratio_imbalance_bound} is equal to $1$,
while the right-hand side is equal to $4$.
Hence the claim is immediate. We may assume $\lambda<0$.

By symmetry in $x$ and $y$, it is enough to consider $0<x\le y$.
Set
$
    r\ceq \sqrt{\frac{x}{y}}\in(0,1]$.
Then $x=yr^2$. Moreover,
\[
    M_\lambda(x,y)
    =
    \lrb{\frac{x^\lambda+y^\lambda}{2}}^{1/\lambda}
    =
    y\lrb{\frac{r^{2\lambda}+1}{2}}^{1/\lambda}
    =
    y2^{-1/\lambda}r^2\lrb{1+r^{-2\lambda}}^{1/\lambda}.
\]
Since
$
    M_0(x,y)=\sqrt{xy}=yr$,
we get
\[
    \frac{M_\lambda(x,y)}{M_0(x,y)}
    =
    2^{-1/\lambda}r\lrb{1+r^{-2\lambda}}^{1/\lambda}.
\]
Therefore
\[
    \lrb{\frac{M_\lambda(x,y)}{M_0(x,y)}}^{1-2\lambda}
    =
    \frac{
        2^{2-1/\lambda}r^{1-2\lambda}
    }{
        \lrb{1+r^{-2\lambda}}^{2-1/\lambda}
    }.
\]

On the other hand,
\[
    \labs{\frac{x-y}{x+y}}
    =
    \frac{y-x}{x+y}
    =
    \frac{1-r^2}{1+r^2},
\]
and hence
\[
    1-\labs{\frac{x-y}{x+y}}
    =
    \frac{2r^2}{1+r^2}.
\]
Thus \eqref{eq:mean_ratio_imbalance_bound} is equivalent to
\begin{equation}
\label{eq:equivalent_kernel_imbalance_bound}
    2^{\lambda-1/\lambda}r\lrb{1+r^2}^{-\lambda}
    \le
    \lrb{1+r^{-2\lambda}}^{2-1/\lambda}.
\end{equation}

First assume $-1\le\lambda<0$. Since $0<r\le1$ and $-\lambda\le1$, we have
$r^{-2\lambda}\ge r^2$, hence
\[
    1+r^{-2\lambda}\ge1+r^2.
\]
Also, by the arithmetic-geometric mean inequality,
$
    1+r^{-2\lambda}
    \ge
    2r^{-\lambda}$.
Therefore,
\begin{align*}
    \lrb{1+r^{-2\lambda}}^{2-1/\lambda}
&=
    \lrb{1+r^{-2\lambda}}^{-\lambda}
    \lrb{1+r^{-2\lambda}}^{2+\lambda-1/\lambda}
\ge
    \lrb{1+r^2}^{-\lambda}
    \lrb{1+r^{-2\lambda}}^{-1/\lambda}
\\
&\ge
    \lrb{1+r^2}^{-\lambda}
    \lrb{2r^{-\lambda}}^{-1/\lambda}
=
    2^{-1/\lambda}r\lrb{1+r^2}^{-\lambda}.
\end{align*}
Since $\lambda<0$, we have $2^\lambda\le1$, and therefore
$
    2^{\lambda-1/\lambda}
    \le
    2^{-1/\lambda}.
$
This proves \eqref{eq:equivalent_kernel_imbalance_bound} in the case $-1\le\lambda<0$.

Now assume $\lambda\le-1$. Since the map $t\mapsto t^{-\lambda}$ is convex on $[0,+\infty)$,
Jensen's inequality gives
\[
    \frac{1+r^{-2\lambda}}{2}
    \ge
    \lrb{\frac{1+r^2}{2}}^{-\lambda}.
\]
Equivalently,
\[
    1+r^{-2\lambda}
    \ge
    2^{1+\lambda}\lrb{1+r^2}^{-\lambda}.
\]
Again, by the arithmetic-geometric mean inequality,
$
    1+r^{-2\lambda}
    \ge
    2r^{-\lambda}$.
Since $1+r^{-2\lambda}\ge1$,
\[
    \lrb{1+r^{-2\lambda}}^{2-1/\lambda}
    \ge
    \lrb{1+r^{-2\lambda}}^{1-1/\lambda}.
\]
Therefore,
\begin{align*}
    \lrb{1+r^{-2\lambda}}^{2-1/\lambda}
&\ge
    \lrb{1+r^{-2\lambda}}
    \lrb{1+r^{-2\lambda}}^{-1/\lambda}
\\
&\ge
    2^{1+\lambda}\lrb{1+r^2}^{-\lambda}
    \lrb{2r^{-\lambda}}^{-1/\lambda}
\\
&=
    2^{1+\lambda-1/\lambda}
    r\lrb{1+r^2}^{-\lambda}.
\end{align*}
Since
    $2^{1+\lambda-1/\lambda}
    \ge
    2^{\lambda-1/\lambda}$,
we again obtain \eqref{eq:equivalent_kernel_imbalance_bound}.
This proves \eqref{eq:mean_ratio_imbalance_bound}.

Substituting \eqref{eq:mean_ratio_imbalance_bound} into the definition of $H_\lambda$ yields
\[
    H_\lambda(x,y)
    \le
    \frac{1-\lambda}{4}
    \cdot
    4
    \lrb{
        1-\labs{\frac{x-y}{x+y}}
    }^{-\lambda}
    \frac1{\sqrt{xy}},
\]
which is the desired bound.
\end{proof}

\section{Proof of the Kernel Reconstruction Lemma}
\label{app:kernel_reconstruction}
\convolutionlemma*

\begin{proof}
    \label{proof:convolution_lemma}
	First, notice that if $p \le s$ or $b \le p$ the statement is immediate.
    It is then enough to prove the statement for $s<p<b$.
    In this case, we have
	\begin{align*}
		\gft_{-\infty}(p,s,b)
	&=
		\min\brb{(p-s)_+,(b-p)_+}
	=
		\int_0^1 \I\{ u \le (p-s)_+ \}\I\{ u \le (b-p)_+ \} \dif u
	\\
    &=
		\int_0^1 \I\{s \le p-u  \} \I\{p+u\le b\} \dif u\;
    =
        \int_0^{\min(p,1-p)} \I\{s \le p-u  \} \I\{p+u\le b\} \dif u,
	\end{align*}
    where in the last equality we used that the integrand vanishes whenever $u>p$ (since then $p-u<0\le s$), and also whenever $u>1-p$ (since then $p+u>1\ge b$).
    For the second claim, replace $s$ and $b$ by the independent random variables $S$ and $B$, take expectations, and apply Tonelli's theorem.
    This gives, for any $p \in [0,1]$,
    \begin{align*}
        G_{-\infty}(p)
        &=
        \int_0^{\min(p,1-p)}
        \E\bsb{
            \I\{S \le p-u\}
            \I\{p+u \le B\}
        }
        \dif u
        \\
        &=
        \int_0^{\min(p,1-p)}
        \Pb\bsb{\lcb{S \le p-u} \cap \lcb{ p+u \le B}}
        \dif u
        \\
        &=
        \int_0^{\min(p,1-p)}
        \Pb\lsb{S \le p-u} \Pb\lsb{ p+u \le B}
        \dif u
        =
        \int_0^{\min(p,1-p)}
        F_S(p-u)F_B^\circ(p+u)
        \dif u ,
    \end{align*}
    where the penultimate equality uses the independence of $S$ and $B$, and the last equality uses the definitions of $F_S$ and $F_B^\circ$.
\end{proof}

\kernelreconstructionlemma*

\begin{proof}
Fix $\lambda\in(-\infty,0]$ and $p,s,b\in[0,1]$.
Throughout the proof, we interpret the integrand arbitrarily on the boundary sets where
$p-v=0$ or $w-p=0$; these sets have two-dimensional Lebesgue measure zero and hence do not affect the value of the integrals.
We first prove the pointwise identity.
If $p\le s$, then for every $v\in[0,p]$ we have $v\le p\le s$, so
$\I\{s\le v\}=0$ except possibly on a null set.
If $b\le p$, then for every $w\in[p,1]$ we have $w\ge p\ge b$, so
$\I\{w\le b\}=0$ except possibly on a null set.
In both cases, the right-hand side is equal to zero, and so is
$\gft_\lambda(p,s,b)$.

It remains to consider the case $s<p<b$.
Set $x\coloneqq p-s$ and
 $y\coloneqq b-p$.
By \Cref{lem:holder_differential_identities},
\[
    M_\lambda(x,y)
    =
    \int_0^x\int_0^y
    H_\lambda(\alpha,\beta)\,\dif\beta\,\dif\alpha .
\]
Using the change of variables
$
    \alpha=p-v,
     \beta=w-p$,
we obtain
\begin{align*}
    \gft_\lambda(p,s,b)
&=
    M_\lambda(p-s,b-p)
=
    \int_0^{p-s}\int_0^{b-p}
    H_\lambda(\alpha,\beta)\,\dif\beta\,\dif\alpha
\\
&=
    \int_s^p\int_p^b
    H_\lambda(p-v,w-p)\,\dif w\,\dif v
=
    \int_0^p\int_p^1
    H_\lambda(p-v,w-p)
    \I\{s\le v\}
    \I\{w\le b\}
    \,\dif w\,\dif v .
\end{align*}
This proves the first claim.

For the second claim, substitute $s=S$ and $b=B$, take expectations, and apply Tonelli's theorem,
using nonnegativity of the integrand. Then
\begin{align*}
    G_\lambda(p)
&=
    \int_0^p\int_p^1
    H_\lambda(p-v,w-p)
    \E\bsb{
        \I\{S\le v\}
        \I\{w\le B\}
    }
    \,\dif w\,\dif v
\\
&=
    \int_0^p\int_p^1
    H_\lambda(p-v,w-p)
    \Pb\lrb{\{S\le v\}\cap\{w\le B\}}
    \,\dif w\,\dif v .
\end{align*}
By independence of $S$ and $B$,
\[
    \Pb\lrb{\{S\le v\}\cap\{w\le B\}}
    =
    \Pb\{S\le v\}\Pb\{w\le B\}
    =
    F_S(v)F_B^\circ(w).
\]
Therefore
\[
    G_\lambda(p)
    =
    \int_0^p\int_p^1
    H_\lambda(p-v,w-p)
    F_S(v)F_B^\circ(w)
    \,\dif w\,\dif v .
\]
This concludes the proof.
\end{proof}

\begin{remark}[The Rawlsian Limit as an Anti-Diagonal Approximation of Identity]
\label{remark:kernel_approximation}
The kernel reconstruction admits a useful interpretation in the Rawlsian limit.
For a fixed $p\in[0,1]$ and $\lambda\in(-\infty,0]$, define the absolutely continuous measure $\mu_{\lambda,p}$ on $[0,p]\times[p,1]$ by
\[
    \mu_{\lambda,p}\lsb{A}
    \coloneqq
    \int_A
    H_{\lambda}(p-v,w-p)\dif w\dif v
\]
for every measurable set $A\subset [0,p]\times[p,1]$.
Let
$
    L_p
    \coloneqq
    \min\lrb{p,1-p}
$
and define $\gamma_p:[0,L_p]\to[0,p]\times[p,1]$ by
$
    \gamma_p(u)
    \coloneqq
    (p-u,p+u)
$.
We denote by
$
    \nu_p
$
the push-forward of Lebesgue measure on $[0,L_p]$ through $\gamma_p$.
The measure $\nu_p$ is supported on the anti-diagonal segment
\[
    \Gamma_p
    \coloneqq
    \bcb{(p-u,p+u):0\le u\le L_p}.
\]
Equivalently, $\nu_p=(1/\sqrt{2})\mathcal H^1|_{\Gamma_p}$, where $\mathcal H^1$ is the $1$-dimensional Hausdorff measure.
Then, as $\lambda\to-\infty$, $\mu_{\lambda,p}$ behaves as an approximation of $\nu_p$.
More precisely:
\begin{enumerate}
    \item $\mu_{\lambda,p}$ is nonnegative;
    \item its total mass converges to the mass of the limiting anti-diagonal measure:
    \[
        \mu_{\lambda,p}\bsb{[0,p]\times[p,1]}
        \to
        \nu_p\bsb{\Gamma_p}
        =
        L_p,
        \qquad
        \lambda\to-\infty\;;
    \]
    \item for every $\eta>0$, its mass outside an $\eta$-tube around the anti-diagonal vanishes:
    \[
        \mu_{\lambda,p}\Bsb{\bcb{(v,w):\labs{v+w-2p}>\eta}}
        \to
        0,
        \qquad
        \lambda\to-\infty\;;
    \]
    \item the limiting mass is uniformly distributed along the anti-diagonal segment: for every interval $I\subseteq[0,L_p]$,
    \[
        \mu_{\lambda,p}
        \Bsb{
            \bcb{(v,w)\in[0,p]\times[p,1]: p-v\in I}
        }
        \to
        \nu_p\Bsb{\bcb{(p-u,p+u):u\in I}}
        =
        |I|,
        \qquad
        \lambda\to-\infty.
    \]
\end{enumerate}
It follows that, for every bounded continuous $\varphi:[0,p]\times[p,1]\to\mathbb R$,
\[
    \int_0^p\int_p^1
    \varphi(v,w)H_{\lambda}(p-v,w-p)\dif w\dif v
    \to
    \int_{\Gamma_p}\varphi\dif\nu_p
    =
    \int_0^{L_p}\varphi(p-u,p+u)\dif u,
    \qquad
    \lambda\to-\infty.
\]

Thus, as $\lambda\to-\infty$, the two-dimensional reconstruction collapses to the one-dimensional Rawlsian convolution along the anti-diagonal.
This provides a useful sanity check on the kernel formula and explains, at a structural level, why the Rawlsian case is substantially simpler than the finite-order Rawls-to-Nash objectives.
In the Rawlsian limit, discretizing the anti-diagonal yields a one-dimensional sum of uniformly bounded terms that can be arranged to use disjoint time indices.
For finite $\lambda$, the estimator of \Cref{def:estimator_G} sums over a rectangle of seller-side and buyer-side observations, with singular weights concentrating near the anti-diagonal and with repeated reuse of observations along rows and columns.
\end{remark}

\begin{figure}[h]
    \centering

    \begin{subfigure}[t]{0.32\textwidth}
        \centering
        \includegraphics[width=\linewidth]{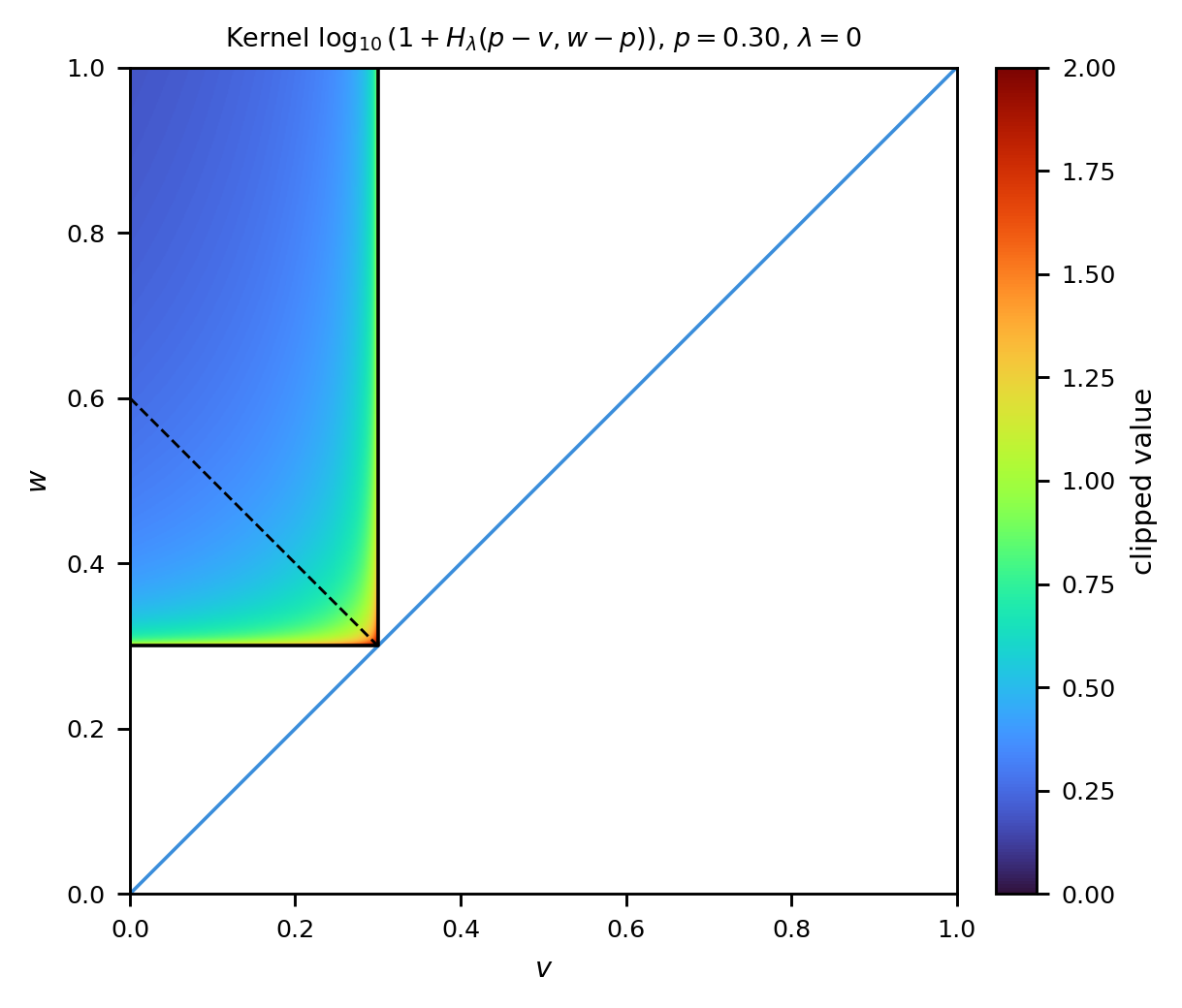}
        \caption{$\lambda=0$.}
        \label{fig:kernel-heatmap-lambda-0}
    \end{subfigure}
    \hfill
    \begin{subfigure}[t]{0.32\textwidth}
        \centering
        \includegraphics[width=\linewidth]{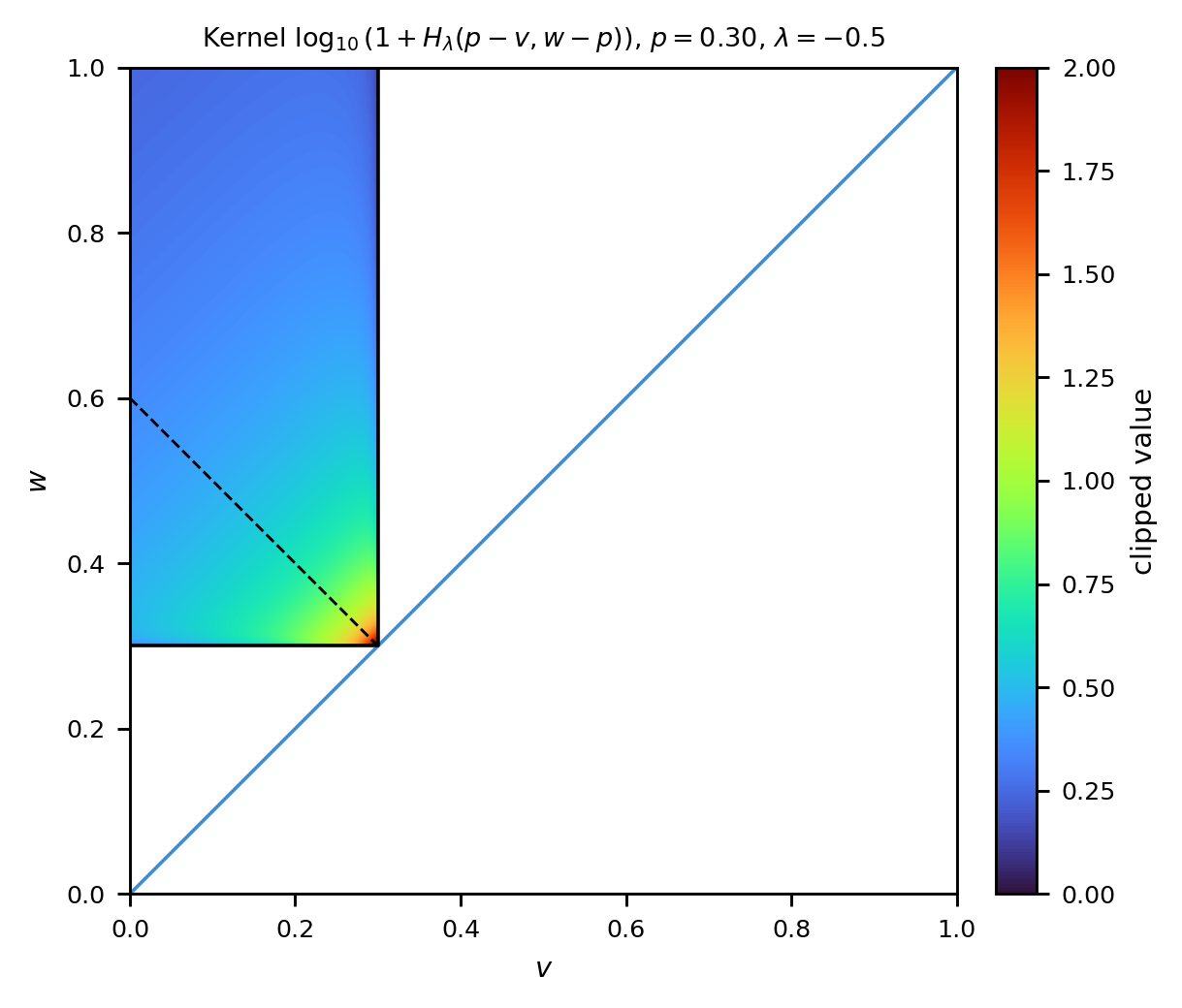}
        \caption{$\lambda=-0.5$.}
        \label{fig:kernel-heatmap-lambda-m0p5}
    \end{subfigure}
    \hfill
    \begin{subfigure}[t]{0.32\textwidth}
        \centering
        \includegraphics[width=\linewidth]{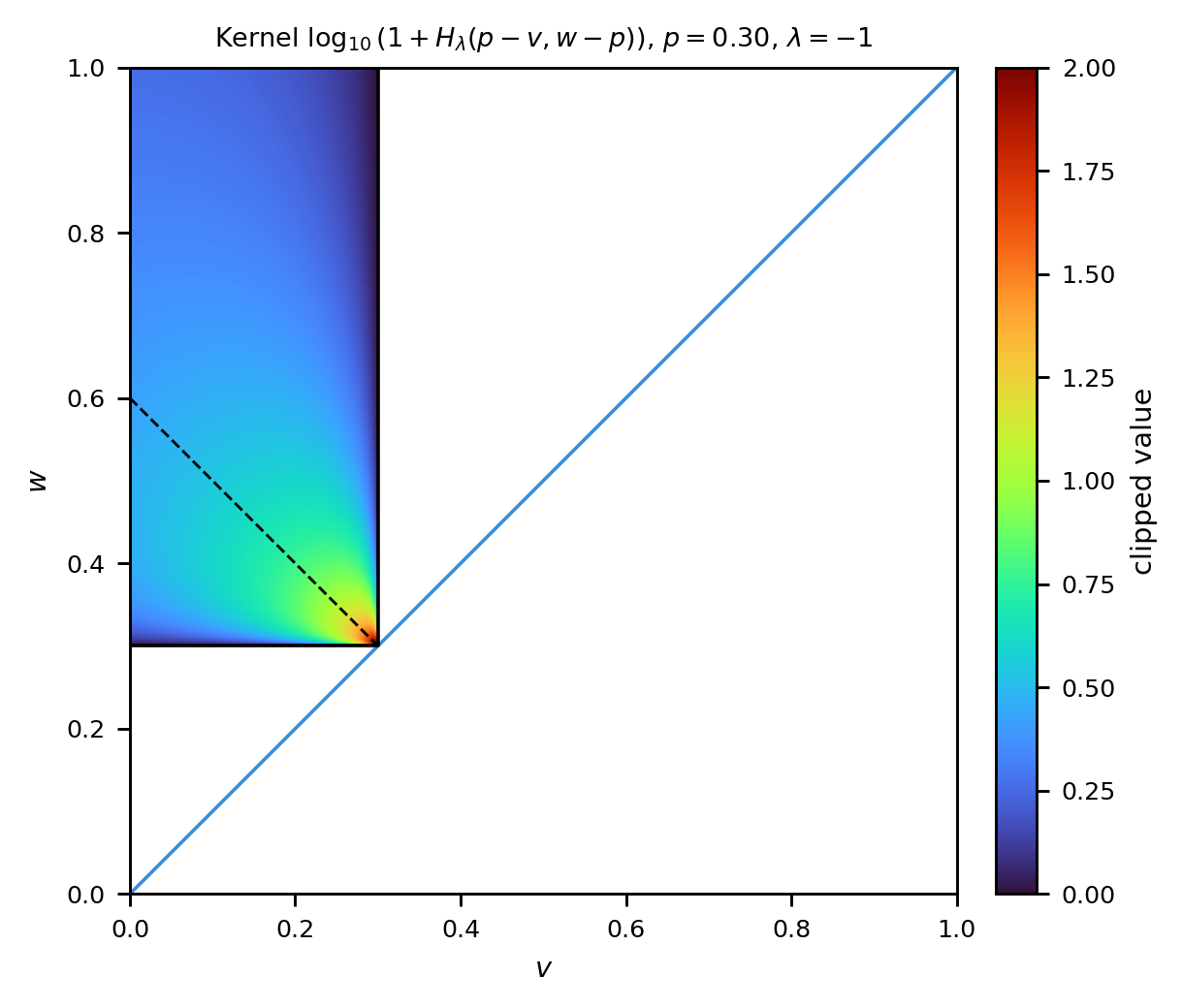}
        \caption{$\lambda=-1$.}
        \label{fig:kernel-heatmap-lambda-m1}
    \end{subfigure}

    \medskip

    \begin{subfigure}[t]{0.32\textwidth}
        \centering
        \includegraphics[width=\linewidth]{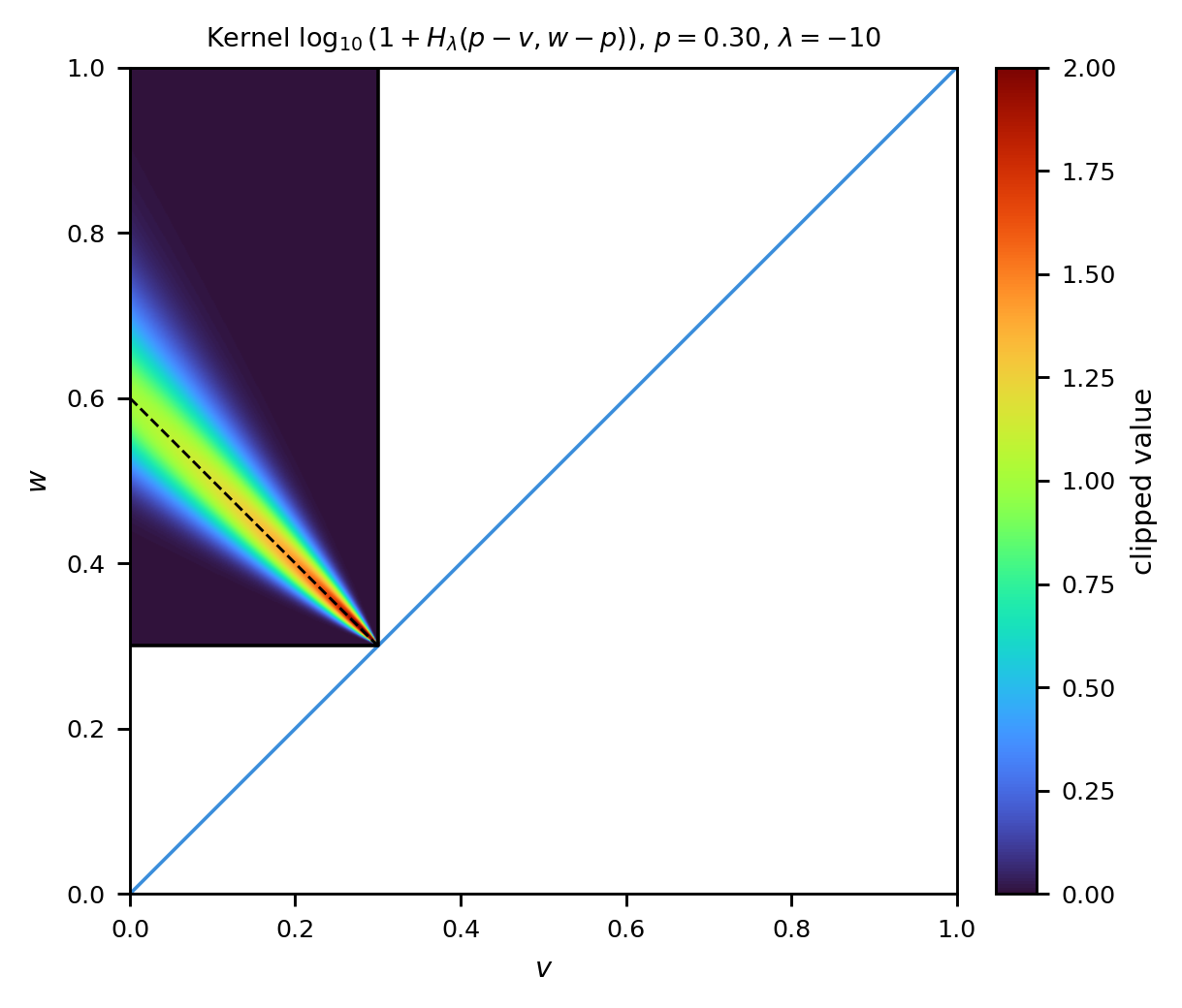}
        \caption{$\lambda=-10$.}
        \label{fig:kernel-heatmap-lambda-m10}
    \end{subfigure}
    \hfill
    \begin{subfigure}[t]{0.32\textwidth}
        \centering
        \includegraphics[width=\linewidth]{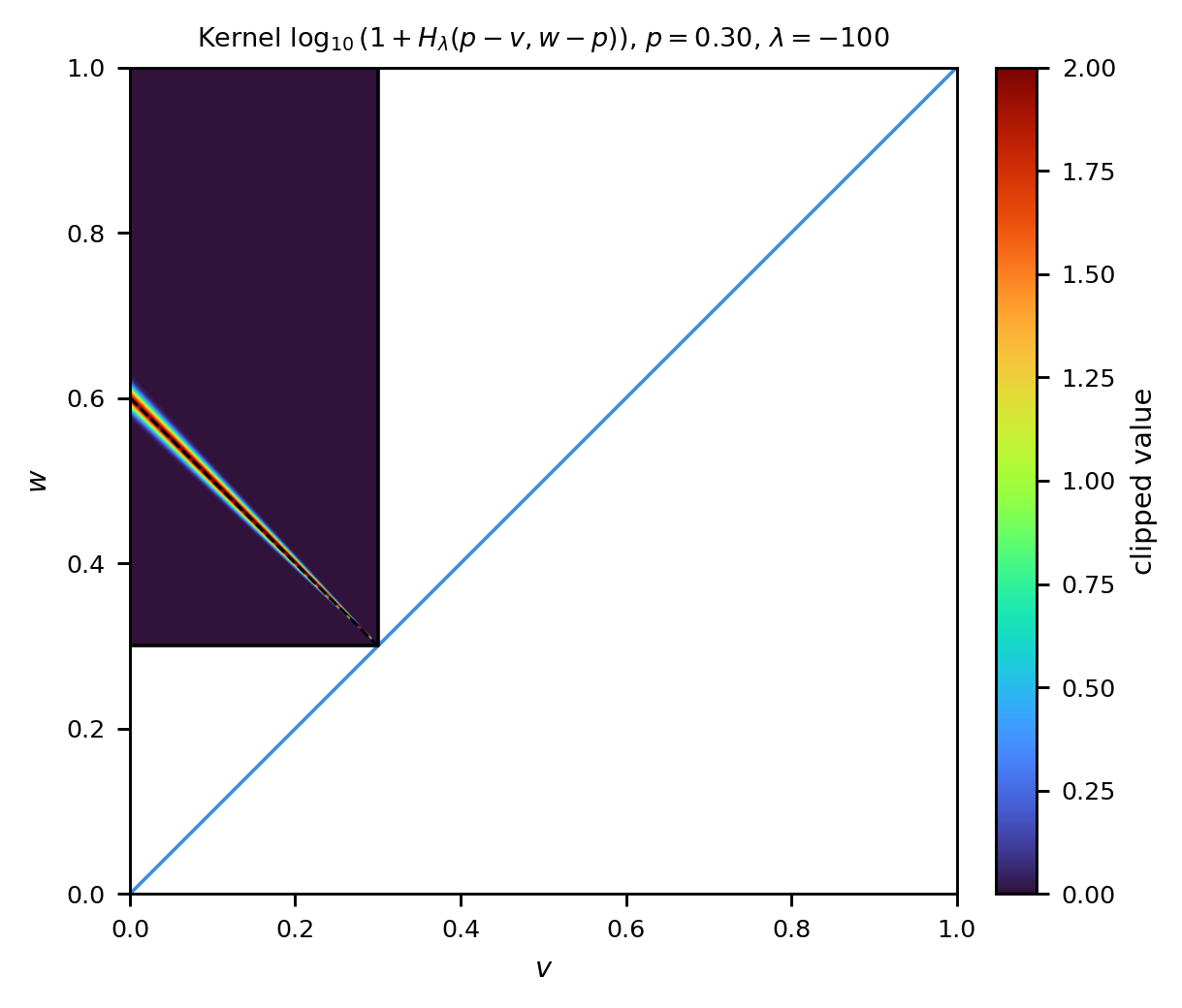}
        \caption{$\lambda=-100$.}
        \label{fig:kernel-heatmap-lambda-m100}
    \end{subfigure}
    \hfill
    \begin{subfigure}[t]{0.32\textwidth}
        \centering
        \includegraphics[width=\linewidth]{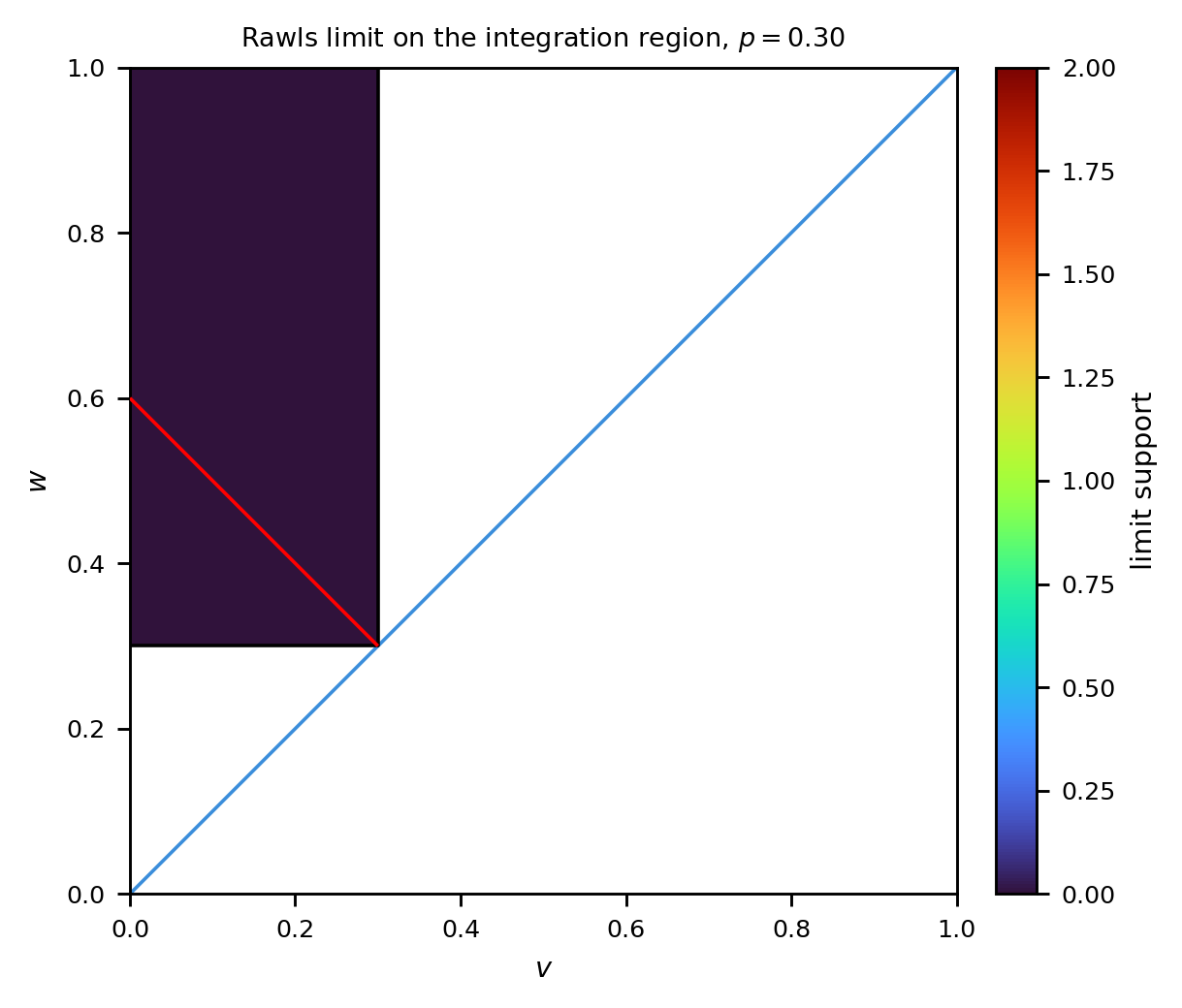}
        \caption{Rawls limit.}
        \label{fig:kernel-heatmap-rawls}
    \end{subfigure}

    \caption{
    Heatmaps of the reconstruction kernel on the relevant integration region.
    For finite $\lambda$, the panels show $\log_{10}\brb{1+H_\lambda(p-v,w-p)}$ on the rectangle $[0,p]\times[p,1]$, with $p=0.30$.
    The displayed color scale is common across finite-$\lambda$ panels and clipped to $[0,2]$.
    The blue line is the main diagonal $w=v$, while the black rectangle marks the integration region in \Cref{lem:kernel_reconstruction_lemma}.
    As $\lambda$ decreases, the kernel weights become increasingly concentrated around the Rawlsian anti-diagonal, that is shown in
    dashed line and in red in
    the last panel (see also \Cref{lem:convolution_lemma}).
    }
    \label{fig:kernel-heatmaps}
\end{figure}


\section{The Rectangular McDiarmid Inequality}
\label{app:concentration}

The following proposition is the concentration input used in the finite-$\lambda$ PAC analysis.
It controls sums of the form
\[
\sum_{i=1}^{\ell}\sum_{j=m}^{K}\phi_{i,j}(X_i,X_j),
\]
where the left and right coordinates of each summand are drawn from disjoint
blocks of independent random variables.
The proof applies McDiarmid's bounded-differences inequality to the
whole double sum and expresses the resulting bounded-differences constant in terms of row and column envelope sums.

\begin{proposition}[Rectangular McDiarmid Inequality]
\label{prop:rectangular:bounded:differences}
Let $K,\ell,m \in \N$ satisfy
\[
    K \ge 2,
    \qquad
    1 \le \ell < m \le K.
\]
For each $r \in \lcb{1,\dots,K}$, let $E_r$ be a measurable space and let $X_r$ be an $E_r$-valued random variable.
Assume that $X_1,\dots,X_K$ are independent.
For each
$
    i \in \lcb{1,\dots,\ell}
$
and
$
    j \in \lcb{m,\dots,K}
$,
let
\[
        \phi_{i,j}
    \colon
        E_i \times E_j \to \bbR
\]
be a measurable function.
Assume that there exist deterministic real numbers $a_{\infty,i,j}$ and deterministic nonnegative numbers $\phi_{\infty,i,j}$ such that, for every $i \in \lcb{1,\dots,\ell}$ and every $j \in \lcb{m,\dots,K}$,
\begin{equation}
\label{eq:rectangular:bd:ass}
    \phi_{i,j}(X_i,X_j)
    \in
    \lsb{
        a_{\infty,i,j},
        a_{\infty,i,j} + \phi_{\infty,i,j}
    }
    \qquad
    \textrm{almost surely}.
\end{equation}
Then, for every $\eps>0$, with the convention that $\exp(-\eps^2/0)=0$,
\begin{align*}
    &
    \Pb\lsb{
        \labs{
            \sum_{i=1}^{\ell}
            \sum_{j=m}^{K}
            \phi_{i,j}(X_i,X_j)
            -
            \E\lsb{
                \sum_{i=1}^{\ell}
                \sum_{j=m}^{K}
                \phi_{i,j}(X_i,X_j)
            }
        }
        \ge \eps
    }
    \\
    &\le
    2\exp\lrb{
        -\frac{2\eps^2}{
            \displaystyle
            \sum_{i=1}^{\ell}
            \Brb{
                \sum_{j=m}^{K}
                \phi_{\infty,i,j}
            }^2
            +
            \sum_{j=m}^{K}
            \Brb{
                \sum_{i=1}^{\ell}
                \phi_{\infty,i,j}
            }^2
        }
    }.
\end{align*}
\end{proposition}

\begin{proof}
Fix $\eps>0$.
For every $i\in\{1,\dots,\ell\}$ and every $j\in\{m,\dots,K\}$, define the clipped function
\[
    \widetilde\phi_{i,j}(x,y)
    \coloneqq
    \min\Brb{
        a_{\infty,i,j}+\phi_{\infty,i,j},
        \max\brb{a_{\infty,i,j},\phi_{i,j}(x,y)}
    }.
\]
By \eqref{eq:rectangular:bd:ass}, we have
\[
    \widetilde\phi_{i,j}(X_i,X_j)
    =
    \phi_{i,j}(X_i,X_j)
    \qquad
    \textrm{almost surely}.
\]
Since the number of pairs $(i,j)$ is finite, the double sum built from the clipped functions agrees almost surely with the original double sum.
In particular, the two double sums have the same law and expectation.
Thus, without loss of generality, we may and do assume throughout the proof that, for every $i \in \{1,\dots,\ell\}$, every $j\in\{m,\dots,K\}$, every $x\in E_i$, and every $y\in E_j$,
\begin{equation}
\label{eq:rectangular:bd:ass_strengthened}
    \phi_{i,j}(x,y)
    \in
    \lsb{
        a_{\infty,i,j},
        a_{\infty,i,j}+\phi_{\infty,i,j}
    }.
\end{equation}

Define
\[
    S
    \coloneqq
    \sum_{i=1}^{\ell}
    \sum_{j=m}^{K}
    \phi_{i,j}(X_i,X_j).
\]
We view $S$ as a measurable function of the independent random variables
$X_1,\dots,X_K$.

Let $x=(x_1,\dots,x_K)\in E_1\times\cdots\times E_K$, and define
\[
    F(x)
    \coloneqq
    \sum_{i=1}^{\ell}
    \sum_{j=m}^{K}
    \phi_{i,j}(x_i,x_j).
\]
For $r\in\lcb{1,\dots,K}$, let $x'$ be another point of
$E_1\times\cdots\times E_K$ such that $x_q'=x_q$ for every $q\neq r$.
We now bound the oscillation of $F$ when only the $r$-th coordinate is changed.

If $r\in\lcb{1,\dots,\ell}$, then only the summands with first index $r$ may change.
Hence, by \eqref{eq:rectangular:bd:ass_strengthened},
\[
    \labs{F(x)-F(x')}
    \le
    \sum_{j=m}^{K}
    \phi_{\infty,r,j}.
\]
If $r\in\lcb{m,\dots,K}$, then only the summands with second index $r$ may change.
Hence, again by \eqref{eq:rectangular:bd:ass_strengthened},
\[
    \labs{F(x)-F(x')}
    \le
    \sum_{i=1}^{\ell}
    \phi_{\infty,i,r}.
\]
Finally, if $r\in\lcb{\ell+1,\dots,m-1}$, then the function $F$ does not depend on the $r$-th coordinate, and therefore
$    \labs{F(x)-F(x')}=0$.
Thus $F$ satisfies the bounded-differences condition with constants
\[
    c_r
    \coloneqq
    \begin{cases}
        \displaystyle
        \sum_{j=m}^{K}
        \phi_{\infty,r,j},
        &\text{if } r\in\lcb{1,\dots,\ell},\\[2ex]
        0,
        &\text{if } r\in\lcb{\ell+1,\dots,m-1},\\[2ex]
        \displaystyle
        \sum_{i=1}^{\ell}
        \phi_{\infty,i,r},
        &\text{if } r\in\lcb{m,\dots,K}.
    \end{cases}
\]
By McDiarmid's inequality recalled in \Cref{lem:mcdiarmid_boundeddifferences},
\[
    \Pb\lsb{
        \labs{
            F(X_1,\dots,X_K)
            -
            \E\lsb{F(X_1,\dots,X_K)}
        }
        \ge \eps
    }
    \le
    2\exp\lrb{
        -\frac{2\eps^2}{
            \sum_{r=1}^{K}c_r^2
        }
    }.
\]
Since $S=F(X_1,\dots,X_K)$ and
\[
    \sum_{r=1}^{K}c_r^2
    =
    \sum_{i=1}^{\ell}
    \lrb{
        \sum_{j=m}^{K}
        \phi_{\infty,i,j}
    }^2
    +
    \sum_{j=m}^{K}
    \lrb{
        \sum_{i=1}^{\ell}
        \phi_{\infty,i,j}
    }^2,
\]
we obtain
\begin{align*}
    &
    \Pb\lsb{
        \labs{
            S-\E\lsb{S}
        }
        \ge \eps
    }
    \le
    2\exp\lrb{
        -\frac{2\eps^2}{
            \displaystyle
            \sum_{i=1}^{\ell}
            \lrb{
                \sum_{j=m}^{K}
                \phi_{\infty,i,j}
            }^2
            +
            \sum_{j=m}^{K}
            \lrb{
                \sum_{i=1}^{\ell}
                \phi_{\infty,i,j}
            }^2
        }
    }.
\end{align*}
This proves the claim.
\end{proof}

\section{Proofs for Exploration Guarantees}
\label{app:proof_PAC_upper}

\noindent\emph{Notation.}
We recall that, for every $\lambda\in[-\infty,0]$, we defined
\[
    G_\lambda(p)
    \coloneqq
    \E\bsb{\gft_\lambda(p,S,B)},
    \qquad
    G_\lambda^\star
    \coloneqq
    \max_{p\in[0,1]}G_\lambda(p).
\]
For every $\lambda\in[-\infty,0]$, fix
\[
    p_\lambda^\star\in\argmax_{p\in[0,1]}G_\lambda(p).
\]
We also recall that $F_S$ denotes the sellers' cdf and $F_B^\circ$ denotes the buyers' survival function.

\subsection{Finite-\texorpdfstring{$\lambda$}{lambda} Optimization Guarantees}
\label{app:finite:lambda:optim:guarantee}

\begin{lemma}[Bias of the Finite-$\lambda$ Estimator]
\label{lem:finite_bias}
For every $\lambda\in(-\infty,0]$, every $n\ge 4$, and every $k\in\{2,\dots,n-2\}$,
\[
    0
\le
    G_\lambda\lrb{\frac{k}{n}}
    -
    \E\lsb{\widehat G_{\lambda,k}(n)}
\le
    \frac{\sqrt2}{\sqrt n}.
\]
\end{lemma}

\begin{proof}
By conditioning on $P_1,\dots,P_n$, and using that valuations are independent across rounds and independent of the randomized posted prices, we have, for every $i\le k$ and $j\ge k+1$,
\[
    \E\lsb{
        Y_iZ_j
        \mid
        P_1,\dots,P_n
    }
    =
    F_S(P_i)F_B^\circ(P_j).
\]
Therefore, we obtain
\begin{align*}
    \E\lsb{\widehat G_{\lambda,k}(n)}
&=
    \frac{1}{n^2}
    \sum_{i=1}^{k-1}
    \sum_{j=k+2}^{n}
    \E\lsb{
        H_\lambda\lrb{\frac{k}{n}-P_i,P_j-\frac{k}{n}}
        F_S(P_i)F_B^\circ(P_j)
    }
\\
&=
    \frac{1}{n^2}
    \sum_{i=1}^{k-1}
    \sum_{j=k+2}^{n}
    n^2
    \int_{\frac{i-1}{n}}^{\frac{i}{n}}
    \int_{\frac{j-1}{n}}^{\frac{j}{n}}
    H_\lambda\lrb{\frac{k}{n}-v,w-\frac{k}{n}}
    F_S(v)F_B^\circ(w)
    \dif w\dif v
\\
&=
    \int_0^{(k-1)/n}
    \int_{(k+1)/n}^{1}
    H_\lambda\lrb{\frac{k}{n}-v,w-\frac{k}{n}}
    F_S(v)F_B^\circ(w)
    \dif w\dif v.
\end{align*}

Recall from \Cref{lem:kernel_reconstruction_lemma} that
\[
 G_\lambda\lrb{\frac{k}{n}}
 =
  \int_0^{\frac{k}{n}}
    \int_{\frac{k}{n}}^{1}
    H_\lambda\lrb{\frac{k}{n}-v,w-\frac{k}{n}}
    F_S(v)F_B^\circ(w)
    \dif w\dif v.
\]
Hence, since $H_{\lambda}$ is nonnegative and $0\le F_S\le 1$ and $0\le F_B^\circ \le 1$, we get
\begin{align*}
    0
&\le
    G_\lambda\lrb{\frac{k}{n}}
    -
    \E\lsb{\widehat G_{\lambda,k}(n)}
    \\
&\le     \int_0^{k/n}
    \int_{\frac{k}{n}}^{\frac{k+1}{n}}
    H_\lambda\lrb{\frac{k}{n}-v,w-\frac{k}{n}}
    \dif w\dif v
    +
    \int_{\frac{k-1}{n}}^{\frac{k}{n}}
    \int_{\frac{k}{n}}^{1}
    H_\lambda\lrb{\frac{k}{n}-v,w-\frac{k}{n}}
    \dif w\dif v
    \\
&=
    \int_0^{k/n}
    \int_0^{1/n}
    H_\lambda(x,y)
    \dif y\dif x
    +
    \int_0^{1/n}
    \int_0^{1-k/n}
    H_\lambda(x,y)
    \dif y\dif x.
\end{align*}

By \Cref{lem:holder_differential_identities}, for every $x,y\ge0$,
\[
    M_\lambda(x,y)
    =
    \int_0^x\int_0^y
    H_\lambda(\alpha,\beta)\,\dif\beta\,\dif\alpha .
\]
Therefore,
\[
    \int_0^{k/n}
    \int_0^{1/n}
    H_\lambda(x,y)
    \,\dif y\,\dif x
    =
    M_\lambda\lrb{\frac{k}{n},\frac1n}
\]
and
\[
    \int_0^{1/n}
    \int_0^{1-k/n}
    H_\lambda(x,y)
    \,\dif y\,\dif x
    =
    M_\lambda\lrb{\frac1n,1-\frac{k}{n}} .
\]
Hence
\[
    0
    \le
    G_\lambda\lrb{\frac{k}{n}}
    -
    \E\lsb{\widehat G_{\lambda,k}(n)}
    \le
    M_\lambda\lrb{\frac{k}{n},\frac1n}
    +
    M_\lambda\lrb{\frac1n,1-\frac{k}{n}} .
\]

Since from \Cref{lem:M-bounds} we have that $M_\lambda\le M_0$ for every $\lambda\le 0$, we get
\begin{align*}
    M_\lambda\lrb{\frac{k}{n},\frac{1}{n}}
    +
    M_\lambda\lrb{1-\frac{k}{n},\frac{1}{n}}
&\le
    M_{0}\lrb{\frac{k}{n},\frac{1}{n}}
    +
    M_{0}\lrb{1-\frac{k}{n},\frac{1}{n}}
\\
&=
    \sqrt{\frac{k}{n^2}}
    +
    \sqrt{\frac{n-k}{n^2}}
=
    \frac{2}{n}\lrb{\frac{1}{2}\sqrt{k}
    +
    \frac{1}{2}\sqrt{n-k}}
\\
&\le
    \frac{2}{n} \sqrt{\frac{1}{2}k + \frac{1}{2}(n-k)
    }
=
    \frac{\sqrt2}{\sqrt n},
\end{align*}
where in the last inequality we used Jensen's inequality.
This proves the claim.
\end{proof}

\begin{lemma}[Off-Diagonal Quadratic Sums]
\label{lem:offdiag_quadratic}
There exists a universal constant $C_{\rm off}>0$ such that the following holds.
Fix $\lambda \in (-\infty,0]$, $n\ge 4$, and $k\in\{2,\dots,n-2\}$.
For $i\in\{1,\dots,k-1\}$ and $j\in\{k+2,\dots,n\}$, define
\[
    \phi_{i,j}^{\rm off}
\coloneqq
    \frac{1}{n^2}
    H_\lambda\lrb{\frac{k}{n}-P_i,P_j-\frac{k}{n}}Y_iZ_j
    \I\bcb{\labs{(k-i)-(j-k-1)}\ge 4}.
\]
Let $\phi_{\infty,i,j}^{\rm off}$ be the deterministic envelope
\[
    \phi_{\infty,i,j}^{\rm off}
\coloneqq
    \sup_{\substack{
        p_i\in\lsb{\frac{i-1}{n},\frac{i}{n}}\\
        p_j\in\lsb{\frac{j-1}{n},\frac{j}{n}}
    }}
    \frac{1}{n^2}
    H_\lambda\lrb{\frac{k}{n}-p_i,p_j-\frac{k}{n}}
    \I\bcb{\labs{(k-i)-(j-k-1)}\ge 4}.
\]
Then
\[
    0
\le
    \phi_{i,j}^{\rm off}
\le
    \phi_{\infty,i,j}^{\rm off}
,
\]
and
\[
    \sum_{j=k+2}^{n}
    \lrb{
        \sum_{i=1}^{k-1}
        \phi_{\infty,i,j}^{\rm off}
    }^2
    \le
    C_{\rm off}\frac{(1+\ln n)^2}{n}
\]
and
\[
    \sum_{i=1}^{k-1}
    \lrb{
        \sum_{j=k+2}^{n}
        \phi_{\infty,i,j}^{\rm off}
    }^2
    \le
    C_{\rm off}\frac{(1+\ln n)^2}{n}.
\]
For instance, one may take $C_{\rm off}=828$.
\end{lemma}

We remark that the two quantities
\[
    \sum_{j=k+2}^{n}
    \lrb{
        \sum_{i=1}^{k-1}
        \phi_{\infty,i,j}^{\rm off}
    }^2
    \qquad\text{and}\qquad
    \sum_{i=1}^{k-1}
    \lrb{
        \sum_{j=k+2}^{n}
        \phi_{\infty,i,j}^{\rm off}
    }^2
\]
are precisely the row and column quadratic sums entering the bounded-differences constant in
\Cref{prop:rectangular:bounded:differences}, when that proposition is applied with envelope
$\phi_{\infty,i,j}^{\rm off}$ to the off-diagonal part of the finite-$\lambda$ estimator.

\begin{proof}
The bound
$
    0\le \phi_{i,j}^{\rm off}\le \phi_{\infty,i,j}^{\rm off}
$
is immediate from the definition of $\phi_{\infty,i,j}^{\rm off}$ and from $0\le Y_i,Z_j\le1$.

We now prove the two deterministic quadratic-sum bounds.
Instead of working directly with the original indices $i,j$, we introduce coordinates adapted to the kernel singularities and to the off-diagonal split.
Thus, for $i\in\{1,\dots,k-1\}$ and $j\in\{k+2,\dots,n\}$, set
\[
    a\coloneqq k-i\in\{1,\dots,k-1\},
\qquad
    b\coloneqq j-k-1\in\{1,\dots,n-k-1\}.
\]
For notational convenience, write
\[
    \psi_{a,b}^{\rm off}
    \coloneqq
    \phi_{\infty,k-a,k+b+1}^{\rm off},
    \qquad
    a\in\{1,\dots,k-1\},
    \qquad
    b\in\{1,\dots,n-k-1\}.
\]
Whenever needed below, we extend this notation by setting
$    \psi_{a,b}^{\rm off}=0$
if either $a\notin\{1,\dots,k-1\}$ or $b\notin\{1,\dots,n-k-1\}$.
With this convention, sums over integer ranges outside the original domain are still well-defined.
We also understand sums whose lower index exceeds their upper index to be equal to zero.

Notice that the definition of $\phi_{\infty,i,j}^{\rm off}$ gives, for
$
    a\in\{1,\dots,k-1\}
$
and
$
    b\in\{1,\dots,n-k-1\}
$,
\[
    \psi_{a,b}^{\rm off}
    =
    \sup_{\substack{
        p_i\in\lsb{\frac{k-a-1}{n},\frac{k-a}{n}}\\
        p_j\in\lsb{\frac{k+b}{n},\frac{k+b+1}{n}}
    }}
    \frac{1}{n^2}
    H_\lambda\lrb{\frac{k}{n}-p_i,p_j-\frac{k}{n}}
    \I\bcb{\labs{a-b}\ge4}.
\]
Equivalently, after the change of variables
$
    x\coloneqq \frac{k}{n}-p_i,
    y\coloneqq p_j-\frac{k}{n}$,
we have
\[
    x\in\lsb{\frac{a}{n},\frac{a+1}{n}},
    \qquad
    y\in\lsb{\frac{b}{n},\frac{b+1}{n}},
\]
and therefore
\[
    \psi_{a,b}^{\rm off}
    =
    \frac{1}{n^2}
    \I\bcb{\labs{a-b}\ge4}
    \sup_{\substack{
        x\in\lsb{\frac{a}{n},\frac{a+1}{n}}\\
        y\in\lsb{\frac{b}{n},\frac{b+1}{n}}
    }}
    H_\lambda(x,y).
\]

We first derive a pointwise upper bound on one side of the off-diagonal region, namely $a\ge b+4$.
The other side, $b\ge a+4$, will be treated afterwards by exchanging the roles of $a$ and $b$.
Fix therefore $a,b$ with $a\ge b+4$, and take arbitrary
\[
    x\in\lsb{\frac{a}{n},\frac{a+1}{n}},
    \qquad
    y\in\lsb{\frac{b}{n},\frac{b+1}{n}}.
\]
Then $x>y$, and
\[
    x\ge \frac{a}{n},
\qquad
    y\ge \frac{b}{n},
\qquad
    x+y\le \frac{a+b+2}{n},
\qquad
    x-y\ge \frac{a-b-1}{n}.
\]
Recall from \Cref{cor:kernel_imbalance_bound} that
\[
    H_\lambda(x,y)
\le
    \frac{1-\lambda}{\sqrt{xy}}
    \lrb{1-\frac{x-y}{x+y}}^{-\lambda}.
\]
Now, we use the scalar inequality
\begin{equation}
\label{eq:scalar_inequality}
    (1-\lambda)(1-u)^{-\lambda}\le\frac1u,
    \qquad
    \forall \lambda\in(-\infty,0],\ \forall u\in(0,1).
\end{equation}
Indeed, setting $r\coloneqq -\lambda\ge 0$, this is equivalent to
\[
    (1+r)u(1-u)^r\le 1.
\]
If $r=0$, this inequality reduces to $u\le 1$, which is immediate.
If $r>0$, the map
\[
    u\mapsto (1+r)u(1-u)^r
\]
attains its maximum at $u=1/(1+r)$, where its value is
\[
    \lrb{\frac{r}{1+r}}^r\le 1.
\]
Applying \eqref{eq:scalar_inequality} with
\[
    u=\frac{x-y}{x+y}
\]
gives
\[
    H_\lambda(x,y)
\le
    \frac{1}{\sqrt{xy}}\frac{x+y}{x-y}
\le
    \frac{n}{\sqrt{ab}}\frac{a+b+2}{a-b-1}.
\]
Since $a\ge b+4$, and $b \ge 1$, we have
\[
    a+b+2\le 2a,
\qquad
    \frac{1}{\sqrt b}\le \frac{\sqrt2}{\sqrt{b+1}}.
\]
Therefore,
\[
    H_\lambda(x,y)
\le
    \frac{2\sqrt2 n}{\sqrt{b+1}}
    \frac{\sqrt a}{a-b-1}
\le
    \frac{3n}{\sqrt{b+1}}
    \frac{\sqrt a}{a-b-1}.
\]
Thus, on the side $a\ge b+4$,
\[
    \psi_{a,b}^{\rm off}
\le
    \frac{3}{n\sqrt{b+1}}
    \frac{\sqrt a}{a-b-1}.
\]

For fixed $b$, summing over $a\ge b+4$ gives
\begin{equation}
\label{eq:phi_off_estimate_1}
    \sum_{a=b+4}^{k-1}\psi_{a,b}^{\rm off}
\le
    \frac{3}{n\sqrt{b+1}}
    \sum_{a=b+4}^{k-1}
    \frac{\sqrt a}{a-b-1}.
\end{equation}
Set $d\coloneqq a-b-1$, so that
\[
    d\in\{3,\dots,k-b-2\}.
\]
Then
\[
    \sum_{a=b+4}^{k-1}
    \frac{\sqrt a}{a-b-1}
=
    \sum_{d=3}^{k-b-2}
    \frac{\sqrt{b+1+d}}{d}.
\]
Using the inequality $\sqrt{x+y}\le \sqrt{x}+\sqrt{y}$, we get
\[
    \sqrt{b+1+d}\le \sqrt{b+1}+\sqrt d,
\]
and therefore
\[
    \sum_{d=3}^{k-b-2}
    \frac{\sqrt{b+1+d}}{d}
\le
    \sqrt{b+1}\sum_{d=3}^{k-b-2}\frac1d
    +
    \sum_{d=3}^{k-b-2}\frac1{\sqrt d}.
\]
Moreover,
\[
    \sum_{d=3}^{k-b-2}\frac1d
\le
    \sum_{d=1}^{n}\frac1d
\le
    1+\int_1^n \frac{1}{r}\dif r
=
    1+\ln n,
\]
and
\[
    \sum_{d=3}^{k-b-2}\frac1{\sqrt d}
\le
    \sum_{d=1}^{n}\frac1{\sqrt d}
\le
    \int_0^n \frac{1}{\sqrt r}\dif r
=
    2\sqrt n.
\]
Substituting these bounds into \eqref{eq:phi_off_estimate_1}, we obtain
\begin{align*}
    \sum_{a=b+4}^{k-1}\psi_{a,b}^{\rm off}
&\le
    \frac{3}{n\sqrt{b+1}}
    \sum_{a=b+4}^{k-1}
    \frac{\sqrt a}{a-b-1}
=
    \frac{3}{n\sqrt{b+1}}
    \sum_{d=3}^{k-b-2}
    \frac{\sqrt{b+1+d}}{d}
\\
&\le
    \frac{3}{n\sqrt{b+1}}
    \lrb{
        \sqrt{b+1}\sum_{d=3}^{k-b-2}\frac1d
        +
        \sum_{d=3}^{k-b-2}\frac1{\sqrt d}
    }
\le
    \frac{3(1+\ln n)}{n}
    +
    \frac{6}{\sqrt{n(b+1)}}.
\end{align*}
Using $(u+v)^2\le 2u^2+2v^2$, we get
\[
    \lrb{
        \sum_{a=b+4}^{k-1}
        \psi_{a,b}^{\rm off}
    }^2
\le
    \frac{18(1+\ln n)^2}{n^2}
    +
    \frac{72}{n(b+1)}.
\]
Summing over $b$, and using that $1+\ln n\ge1$, yields
\begin{align}
\label{eq:offdiag:right_side}
    \sum_{b=1}^{n-k-1}
    \lrb{
        \sum_{a=b+4}^{k-1}
        \psi_{a,b}^{\rm off}
    }^2
&\le
    \frac{18(1+\ln n)^2}{n^2}(n-k-1)
    +
    \frac{72}{n}
    \sum_{b=1}^{n-k-1}\frac1{b+1}
    \notag
\\
&\le
    \frac{18(1+\ln n)^2}{n}
    +
    \frac{72}{n}
    \sum_{b=1}^{n}\frac1b
    \notag
\le
    \frac{18(1+\ln n)^2}{n}
    +
    \frac{72}{n}
    \lrb{1+\int_1^n \frac{1}{r}\dif r}
    \notag
\\
&\le
    90\frac{(1+\ln n)^2}{n}.
\end{align}

We next control the contribution to the same column sum coming from the opposite off-diagonal side $b\ge a+4$.
By symmetry, for $a \in \lcb{1,\dots, k-1}$ and $b \in \lcb{1,\dots,n-k-1}$ with $b\ge a+4$, we have
\[
    \psi_{a,b}^{\rm off}
\le
    \frac{3}{n\sqrt{a+1}}
    \frac{\sqrt b}{b-a-1}.
\]
For fixed $b$, this gives
\[
    \sum_{a=1}^{b-4}\psi_{a,b}^{\rm off}
\le
    \frac{3\sqrt b}{n}
    \sum_{a=1}^{b-4}
    \frac{1}{\sqrt{a+1}(b-a-1)}.
\]
Set $d\coloneqq b-a-1$.
Then
\[
    \sum_{a=1}^{b-4}
    \frac{1}{\sqrt{a+1}(b-a-1)}
=
    \sum_{d=3}^{b-2}
    \frac{1}{d\sqrt{b-d}}.
\]
We split the latter sum according to $d\le b/2$ and $d>b/2$.
For $d\le b/2$, we have
\[
    \frac{\sqrt b}{\sqrt{b-d}}\le \sqrt2,
\]
and therefore
\[
    \frac{3\sqrt b}{n}
    \sum_{\substack{3\le d\le b-2\\ d\le b/2}}
    \frac{1}{d\sqrt{b-d}}
\le
    \frac{3\sqrt2}{n}
    \sum_{d=3}^{b-2}\frac1d
\le
    \frac{3\sqrt2}{n}
    \sum_{d=1}^{n}\frac1d
\le
    \frac{3\sqrt2}{n}
    \lrb{1+\int_1^n \frac{1}{r}\dif r}
=
    \frac{3\sqrt2(1+\ln n)}{n}.
\]
For $d>b/2$, we have $1/d\le 2/b$, and hence
\begin{align*}
    \frac{3\sqrt b}{n}
    \sum_{\substack{3\le d\le b-2\\ d>b/2}}
    \frac{1}{d\sqrt{b-d}}
\le
    \frac{6}{n\sqrt b}
    \sum_{\substack{3\le d\le b-2\\ d>b/2}}
    \frac{1}{\sqrt{b-d}}
\le
    \frac{6}{n\sqrt b}
    \sum_{d'=1}^{b}
    \frac1{\sqrt{d'}}
\le
    \frac{6}{n\sqrt b}
    \int_0^b \frac{1}{\sqrt r}\dif r
=
    \frac{12}{n}.
\end{align*}
Consequently, since $1+\ln n\ge1$, we have
\[
    \sum_{a=1}^{b-4}\psi_{a,b}^{\rm off}
\le
    \frac{18(1+\ln n)}{n}.
\]
Therefore,
\begin{align}
\label{eq:offdiag:left_side}
    \sum_{b=1}^{n-k-1}
    \lrb{
        \sum_{a=1}^{b-4}
        \psi_{a,b}^{\rm off}
    }^2
&\le
    \sum_{b=1}^{n-k-1}
    \lrb{
        \frac{18(1+\ln n)}{n}
    }^2
=
    \frac{324(1+\ln n)^2}{n^2}(n-k-1)
    \notag
\\
&\le
    \frac{324(1+\ln n)^2}{n}.
\end{align}

We now return from the auxiliary coordinates to the first quadratic sum appearing in the statement of the lemma.
In the coordinates $a=k-i$ and $b=j-k-1$, this sum is
\[
    \sum_{j=k+2}^{n}
    \lrb{
        \sum_{i=1}^{k-1}
        \phi_{\infty,i,j}^{\rm off}
    }^2
    =
    \sum_{b=1}^{n-k-1}
    \lrb{
        \sum_{a=1}^{k-1}
        \psi_{a,b}^{\rm off}
    }^2.
\]
For each fixed $b$, the definition of $\psi_{a,b}^{\rm off}$ gives
$
    \psi_{a,b}^{\rm off}=0
$
for every $a\in\{1,\dots,k-1\}$ such that $\labs{a-b}\le3$.
Therefore
\[
    \sum_{a=1}^{k-1}
        \psi_{a,b}^{\rm off}
=
    \sum_{\substack{1\le a\le k-1\\ \labs{a-b}\ge4}}
        \psi_{a,b}^{\rm off}
=
    \sum_{a=1}^{b-4}
        \psi_{a,b}^{\rm off}
    +
    \sum_{a=b+4}^{k-1}
        \psi_{a,b}^{\rm off},
\]
where the last equality uses the zero-extension convention above.
Hence, using $(u+v)^2\le2u^2+2v^2$ together with
\eqref{eq:offdiag:right_side} and \eqref{eq:offdiag:left_side}, we obtain
\begin{align*}
    \sum_{j=k+2}^{n}
    \lrb{
        \sum_{i=1}^{k-1}
        \phi_{\infty,i,j}^{\rm off}
    }^2
&=
    \sum_{b=1}^{n-k-1}
    \lrb{
        \sum_{a=1}^{k-1}
        \psi_{a,b}^{\rm off}
    }^2
\le
    2
    \sum_{b=1}^{n-k-1}
    \lrb{
        \sum_{a=b+4}^{k-1}
        \psi_{a,b}^{\rm off}
    }^2
    +
    2
    \sum_{b=1}^{n-k-1}
    \lrb{
        \sum_{a=1}^{b-4}
        \psi_{a,b}^{\rm off}
    }^2
\\
&\le
    2\cdot 90\frac{(1+\ln n)^2}{n}
    +
    2\cdot 324\frac{(1+\ln n)^2}{n}
\le
    828\frac{(1+\ln n)^2}{n}.
\end{align*}
Thus,
\[
    \sum_{j=k+2}^{n}
    \lrb{
        \sum_{i=1}^{k-1}
        \phi_{\infty,i,j}^{\rm off}
    }^2
    \le
    828\frac{(1+\ln n)^2}{n}.
\]

It remains to prove the transposed quadratic-sum bound.
This is the same estimate with the roles of rows and columns exchanged.
Indeed, in the same coordinates, the second deterministic quadratic sum in the statement is
\[
    \sum_{i=1}^{k-1}
    \lrb{
        \sum_{j=k+2}^{n}
        \phi_{\infty,i,j}^{\rm off}
    }^2
    =
    \sum_{a=1}^{k-1}
    \lrb{
        \sum_{b=1}^{n-k-1}
        \psi_{a,b}^{\rm off}
    }^2.
\]
Repeating the previous argument with the roles of $a$ and $b$ exchanged gives
\[
    \sum_{a=1}^{k-1}
    \lrb{
        \sum_{b=1}^{n-k-1}
        \psi_{a,b}^{\rm off}
    }^2
    \le
    828\frac{(1+\ln n)^2}{n}.
\]
Returning to the original indices, this is precisely
\[
    \sum_{i=1}^{k-1}
    \lrb{
        \sum_{j=k+2}^{n}
        \phi_{\infty,i,j}^{\rm off}
    }^2
    \le
    828\frac{(1+\ln n)^2}{n}.
\]

Together with the envelope bound proved at the beginning of the proof, this proves all the claims of the lemma with
$
    C_{\rm off}=828
$.
\end{proof}

\begin{lemma}[Diagonal-Strip Quadratic Sums]
\label{lem:diag_quadratic}
There exists a universal constant $C_{\rm diag}>0$ such that the following holds.
Fix $\lambda \in (-\infty,0]$, $n\ge4$, and $k\in\{2,\dots,n-2\}$.
For $i\in\{1,\dots,k-1\}$ and $j\in\{k+2,\dots,n\}$, define
\[
    \phi_{i,j}^{\rm diag}
    \coloneqq
    \frac{1}{n^2}
    H_\lambda\lrb{\frac{k}{n}-P_i,P_j-\frac{k}{n}}Y_iZ_j
    \I\lcb{\labs{(k-i)-(j-k-1)}\le3}.
\]
Let $\phi_{\infty,i,j}^{\rm diag}$ be the deterministic envelope
\[
    \phi_{\infty,i,j}^{\rm diag}
    \coloneqq
    \sup_{\substack{
        p_i\in\lsb{\frac{i-1}{n},\frac{i}{n}}\\
        p_j\in\lsb{\frac{j-1}{n},\frac{j}{n}}
    }}
    \frac{1}{n^2}
    H_\lambda\lrb{\frac{k}{n}-p_i,p_j-\frac{k}{n}}
    \I\lcb{\labs{(k-i)-(j-k-1)}\le3}.
\]
Then
\[
    0
    \le
    \phi_{i,j}^{\rm diag}
    \le
    \phi_{\infty,i,j}^{\rm diag},
\]
and
\[
    \sum_{j=k+2}^{n}
    \lrb{
        \sum_{i=1}^{k-1}
        \phi_{\infty,i,j}^{\rm diag}
    }^2
    \le
    56\frac{(1-\lambda)^2}{n^2}
\]
and
\[
    \sum_{i=1}^{k-1}
    \lrb{
        \sum_{j=k+2}^{n}
        \phi_{\infty,i,j}^{\rm diag}
    }^2
    \le
    56\frac{(1-\lambda)^2}{n^2}.
\]
\end{lemma}

\begin{proof}
The bound
$
    0\le \phi_{i,j}^{\rm diag}\le \phi_{\infty,i,j}^{\rm diag}
$
is immediate from the definition of $\phi_{\infty,i,j}^{\rm diag}$ and from $0\le Y_i,Z_j\le1$.

As in the off-diagonal case, set
\[
    a\coloneqq k-i\in\lcb{1,\dots,k-1},
    \qquad
    b\coloneqq j-k-1\in\lcb{1,\dots,n-k-1},
\]
and write
\[
    \psi_{a,b}^{\rm diag}
    \coloneqq
    \phi_{\infty,k-a,k+b+1}^{\rm diag}.
\]
We extend this notation by setting $\psi_{a,b}^{\rm diag}=0$ whenever either index is outside its admissible range.
Equivalently, after the change of variables
\[
x=\frac{k}{n}-p_i,\qquad y=p_j-\frac{k}{n},
\]
we have $x\in[a/n,(a+1)/n]$ and
$y\in[b/n,(b+1)/n]$, and hence
\[
    \psi_{a,b}^{\rm diag}
    =
    \frac{1}{n^2}
    \I\lcb{\labs{a-b}\le3}
    \sup_{\substack{
        x\in\lsb{\frac{a}{n},\frac{a+1}{n}}\\
        y\in\lsb{\frac{b}{n},\frac{b+1}{n}}
    }}
    H_\lambda(x,y).
\]

We first derive a pointwise envelope.
By \Cref{lem:kernel_basic_bound},
\[
    H_\lambda(x,y)
    \le
    \frac{1-\lambda}{4\sqrt{xy}}.
\]
Thus, for
$
    x\in\lsb{a/n,(a+1)/n}
$
and
$
    y\in\lsb{b/n,(b+1)/n}
$,
\[
    H_\lambda(x,y)
    \le
    \frac{1-\lambda}{4}\frac{n}{\sqrt{ab}}.
\]
Consequently,
\begin{equation}
\label{eq:diag_pointwise_envelope}
    \psi_{a,b}^{\rm diag}
    \le
    \frac{1-\lambda}{4n\sqrt{ab}}
    \I\lcb{\labs{a-b}\le3}.
\end{equation}

We now bound the quadratic sum with the outer sum over $j$, equivalently over $b$.
Fix $b\in\{1,\dots,n-k-1\}$.
There are at most seven values of $a$ such that $\labs{a-b}\le3$.
Moreover, whenever $a,b\ge1$ and $\labs{a-b}\le3$, we have
\[
    \frac{1}{\sqrt{ab}}
    \le
    \frac{3}{b}.
\]
Indeed, if $b\ge4$, then $a\ge b-3\ge b/4$, and hence $1/\sqrt{ab}\le2/b$.
If $b\le3$, then $1/\sqrt{ab}\le1\le3/b$.
Using \eqref{eq:diag_pointwise_envelope}, we get
\[
    \sum_{a=1}^{k-1}
    \psi_{a,b}^{\rm diag}
    \le
    7\cdot\frac{1-\lambda}{4n}\cdot\frac{3}{b}
    =
    \frac{21(1-\lambda)}{4nb}.
\]
Squaring and summing over $b$ yields
\begin{align*}
    \sum_{j=k+2}^{n}
    \lrb{
        \sum_{i=1}^{k-1}
        \phi_{\infty,i,j}^{\rm diag}
    }^2
    &=
    \sum_{b=1}^{n-k-1}
    \lrb{
        \sum_{a=1}^{k-1}
        \psi_{a,b}^{\rm diag}
    }^2
    \\
    &\le
    \frac{441(1-\lambda)^2}{16n^2}
    \sum_{b=1}^{n-k-1}\frac1{b^2}
    \le
    \frac{441(1-\lambda)^2}{8n^2}
    \le
    56\frac{(1-\lambda)^2}{n^2},
\end{align*}
where we used $\sum_{b\ge1}b^{-2}\le2$.

The transposed sum is identical, with the roles of $a$ and $b$ exchanged.
This proves all the claims of the lemma.
\end{proof}

\begin{proposition}[Finite-$\lambda$ Concentration with Explicit $\lambda$ Dependence]
\label{prop:moderate_concentration}
Fix $\lambda\in(-\infty,0]$, $n\ge4$, and $k\in\{2,\dots,n-2\}$.
Then, for every $t>0$,
\[
    \Pb\lsb{
        \labs{\widehat G_{\lambda,k}(n)-\E\lsb{\widehat G_{\lambda,k}(n)}}>t
    }
    \le
    2\exp\lrb{
        -
        \frac{nt^2}{
        \lrb{
            \sqrt{828}\lrb{1+\ln n}
            +
            \sqrt{56}\frac{1-\lambda}{\sqrt n}
        }^2}
    }.
\]
\end{proposition}

\begin{proof}
Write
\[
    \widehat G_{\lambda,k}(n)
    =
    \sum_{i=1}^{k-1}\sum_{j=k+2}^{n}\phi_{i,j},
\]
where
\[
    \phi_{i,j}
    \coloneqq
    \frac{1}{n^2}
    H_\lambda\lrb{\frac{k}{n}-P_i,P_j-\frac{k}{n}}Y_iZ_j.
\]
The off-diagonal and diagonal indicators form a partition of the index set, so
\[
    \phi_{i,j}=\phi_{i,j}^{\rm off}+\phi_{i,j}^{\rm diag}.
\]
Let
\[
    \phi_{\infty,i,j}
    \coloneqq
    \phi_{\infty,i,j}^{\rm off}+\phi_{\infty,i,j}^{\rm diag}.
\]
Then $0\le\phi_{i,j}\le\phi_{\infty,i,j}$.

By the triangle inequality, together with
\Cref{lem:offdiag_quadratic,lem:diag_quadratic},
\[
    \sum_{j=k+2}^{n}
    \lrb{
        \sum_{i=1}^{k-1}
        \phi_{\infty,i,j}
    }^2
    \le
    \frac{
        \lrb{
            \sqrt{828}\lrb{1+\ln n}
            +
            \sqrt{56}\frac{1-\lambda}{\sqrt n}
        }^2
    }{n}.
\]
The same bound holds for the transposed quadratic sum.
Applying \Cref{prop:rectangular:bounded:differences} yields the displayed concentration bound.
\end{proof}

\finitekernelhighproboptimization*

\begin{proof}
Fix $\lambda\in(-\infty,0]$, $\delta\in(0,1)$, and $n\ge4$.
Let
\[
    k^\star
    \in
    \argmax_{k\in\{2,\dots,n-2\}}G_\lambda\lrb{\frac{k}{n}}.
\]
By \Cref{lem:G_is_Holder}, and since $G_\lambda(0)=G_\lambda(1)=0$,
\[
    G_\lambda^\star-G_\lambda\lrb{\frac{k^\star}{n}}
    \le
    \frac1{\sqrt n}.
\]
Indeed, if a maximizer lies in $[1/n,1-1/n]$, it is within distance at most $1/n$ from the grid $\{2/n,\dots,(n-2)/n\}$.
If it lies within distance $1/n$ of the boundary, then the same bound follows from the fact that $G_\lambda$ vanishes at the boundary.

Now decompose
\begin{align*}
    G_\lambda^\star-G_\lambda\brb{\widehat p_\lambda(n)}
    &\le
    \lrb{G_\lambda^\star-G_\lambda\lrb{\frac{k^\star}{n}}}
    +
    \lrb{G_\lambda\lrb{\frac{k^\star}{n}}-\widehat G_{\lambda,k^\star}(n)}
    \\
    &\qquad
    +
    \lrb{\widehat G_{\lambda,\widehat k_\lambda(n)}(n)-G_\lambda\brb{\widehat p_\lambda(n)}},
\end{align*}
because
\[
    \widehat G_{\lambda,k^\star}(n)-\widehat G_{\lambda,\widehat k_\lambda(n)}(n)\le0.
\]
By \Cref{lem:finite_bias},
\[
    G_\lambda^\star-G_\lambda\brb{\widehat p_\lambda(n)}
    \le
    \frac{1+2\sqrt2}{\sqrt n}
    +
    2
    \max_{2\le k\le n-2}
    \bbabs{\widehat G_{\lambda,k}(n)-\E\lsb{\widehat G_{\lambda,k}(n)}}.
\]
By \Cref{prop:moderate_concentration} and a union bound over $k\in\{2,\dots,n-2\}$, with probability at least $1-\delta$,
\[
    \max_{2\le k\le n-2}
    \labs{\widehat G_{\lambda,k}(n)-\E\lsb{\widehat G_{\lambda,k}(n)}}
    \le
    \lrb{
        \sqrt{828}\lrb{1+\ln n}
        +
        \sqrt{56}\frac{1-\lambda}{\sqrt n}
    }
    \sqrt{\frac{\ln(2n/\delta)}{n}}.
\]
Since $n\ge4$ and $\delta\in(0,1)$, we have $\ln(2n/\delta)\ge1$.
Thus
\[
    \frac{1+2\sqrt2}{\sqrt n}
    \le
    \lrb{1+2\sqrt2}
    \sqrt{\frac{\ln(2n/\delta)}{n}}.
\]
Using
\[
    2\sqrt{828}+1+2\sqrt2\le62
    \qquad
    \text{and}
    \qquad
    2\sqrt{56}\le15,
\]
we obtain
\[
    G_\lambda^\star-G_\lambda\brb{\widehat p_\lambda(n)}
    \le
    \lrb{
        62\lrb{1+\ln n}
        +
        15\frac{1-\lambda}{\sqrt n}
    }
    \sqrt{\frac{\ln(2n/\delta)}{n}}.
\]
This proves the claim.
\end{proof}

\subsection{Rawls Optimization Guarantees}
\label{app:rawls:optim:guarantees}

\begin{lemma}[Rawls Estimator Bias]
\label{lem:rawls_bias} For every $n \in \N$ such that $n \ge 4$, and every $k\in\{1,\dots,n-1\}$,
\[
    0
\le
    G_{-\infty}\lrb{\frac{k}{n}}
    -
    \E\lsb{\widehat G_{-\infty,k}(n)}
\le
    \frac{1}{4n}.
\]
\end{lemma}

\begin{proof}

Fix $k\in\{1,\dots,n-1\}$.
For notational convenience, define $m_k \coloneqq \min\lrb{k,n-k}$.
For any $a\in\{0,\dots,m_k-1\}$, set
\[
    I_a
\coloneqq
    \lsb{\frac{a}{n},\frac{a+1}{n}}.
\]
By \Cref{lem:convolution_lemma},
\begin{align}
    \label{eq:decomposition_-infty}
    G_{-\infty}\lrb{\frac{k}{n}}
&=
    \int_0^{\min\lrb{\frac{k}{n},1-\frac{k}{n}}} F_S\lrb{\frac{k}{n}-v} F_B^\circ\lrb{\frac{k}{n}+v}\dif v\notag
\\
&=
    \sum_{a=0}^{m_k-1}\int_{I_a} F_S\lrb{\frac{k}{n}-v} F_B^\circ\lrb{\frac{k}{n}+v}\dif v.
\end{align}
Also, following the sampling procedure in \Cref{algo:collect_samples}, for $a\in\{0,\dots,m_k-1\}$,
\[
    P_{k-a}
\sim
    \textrm{Unif}\lrb{\lsb{\frac{k-a-1}{n},\frac{k-a}{n}}},
\qquad
    P_{k+a+1}
\sim
    \textrm{Unif}\lrb{\lsb{\frac{k+a}{n},\frac{k+a+1}{n}}},
\]
and therefore
\[
    \E\lsb{Y_{k-a}}
=
    n\int_{I_a} F_S\lrb{\frac{k}{n}-v}\dif v,
\qquad
    \E\lsb{Z_{k+a+1}}
=
    n\int_{I_a} F_B^\circ\lrb{\frac{k}{n}+v}\dif v.
\]
Since $Y_{k-a}$ and $Z_{k+a+1}$ depend on different rounds of the exploration procedure, and since valuations and randomized prices are independent across rounds, these two random variables are independent.
Hence, for any $a\in\{0,\dots,m_k-1\}$,
\[
    \E\lsb{Y_{k-a}Z_{k+a+1}}
=
    n^2
    \lrb{\int_{I_a} F_S\lrb{\frac{k}{n}-v}\dif v}
    \lrb{\int_{I_a} F_B^\circ\lrb{\frac{k}{n}+v}\dif v}.
\]
Multiplying by $1/n$ and summing over $a$, we get
\begin{equation}
\label{eq:expectation_estimator_-infty}
    \E\lsb{\widehat G_{-\infty,k}(n)}
=
    \sum_{a=0}^{m_k-1}
    \frac{1}{|I_a|}
    \lrb{\int_{I_a} F_S\lrb{\frac{k}{n}-v}\dif v}
    \lrb{\int_{I_a} F_B^\circ\lrb{\frac{k}{n}+v}\dif v}.
\end{equation}

For each $a$, define on $I_a$
\[
    f_a(v)
\coloneqq
    F_S\lrb{\frac{k}{n}-v},
\qquad
    f_a^\circ(v)
\coloneqq
    F_B^\circ\lrb{\frac{k}{n}+v}.
\]
Both $f_a$ and $f_a^\circ$ are non-increasing on $I_a$.
Thus, by Chebyshev's integral inequality (see, e.g., \cite[Theorem~10, Section~2.5]{mitrinovic1970analytic}),
\begin{equation}
\label{eq:chebyshev_integral}
    \frac{1}{|I_a|}\int_{I_a} f_a(v)f_a^\circ(v)\dif v
\ge
    \lrb{\frac{1}{|I_a|}\int_{I_a} f_a(v)\dif v}
    \lrb{\frac{1}{|I_a|}\int_{I_a} f_a^\circ(v)\dif v},
\end{equation}
and, multiplying by $|I_a|$, summing over $a\in\{0,\dots,m_k-1\}$, and recalling \eqref{eq:decomposition_-infty}, we obtain the lower bias bound
\[
    \E\lsb{\widehat G_{-\infty,k}(n)}
\le
    G_{-\infty}\lrb{\frac{k}{n}}.
\]

For the upper bias bound, write
\[
    \Delta^S_{k,a}
\coloneqq
    F_S\lrb{\frac{k-a}{n}}-F_S\lrb{\frac{k-a-1}{n}}
\ge
    0,
\]
and
\[
    \Delta^B_{k,a}
\coloneqq
    F_B^\circ\lrb{\frac{k+a}{n}}-F_B^\circ\lrb{\frac{k+a+1}{n}}
\ge
    0.
\]
On $I_a$, the length of the range of $f_a$ is at most $\Delta^S_{k,a}$ and the length of the range of $f_a^\circ$ is at most $\Delta^B_{k,a}$.
Using Gr\"uss' inequality (see, e.g., \cite[Section~2.13]{mitrinovic1970analytic}) on the interval $I_a$, we get
\[
    \frac1{|I_a|}\int_{I_a} f_af_a^\circ
    -
    \lrb{\frac1{|I_a|}\int_{I_a} f_a}
    \lrb{\frac1{|I_a|}\int_{I_a} f_a^\circ}
\le
    \frac14\Delta^S_{k,a}\Delta^B_{k,a}.
\]
Multiplying by $|I_a|=1/n$, summing over $a\in\{0,\dots,m_k-1\}$, and recalling \eqref{eq:decomposition_-infty} and \eqref{eq:expectation_estimator_-infty}, we obtain
\[
    G_{-\infty}\lrb{\frac{k}{n}}-\E\lsb{\widehat G_{-\infty,k}(n)}
\le
    \frac{1}{4n}\sum_{a=0}^{m_k-1}\Delta^S_{k,a}\Delta^B_{k,a}.
\]
Now
\[
    \sum_{a=0}^{m_k-1}\Delta^S_{k,a}
=
    F_S\lrb{\frac{k}{n}}-F_S\lrb{\frac{k-m_k}{n}}
\le
    1.
\]
Since each $\Delta^B_{k,a}\le 1$, we get
\[
    \sum_{a=0}^{m_k-1}\Delta^S_{k,a}\Delta^B_{k,a}
\le
    \sum_{a=0}^{m_k-1}\Delta^S_{k,a}
\le
    1.
\]
Hence
\[
    0
\le
    G_{-\infty}\lrb{\frac{k}{n}}-\E\lsb{\widehat G_{-\infty,k}(n)}
\le
    \frac{1}{4n}. \qedhere
\]
\end{proof}

\begin{lemma}[Rawls Concentration]
\label{lem:rawls_concentration}
For every $n \in \N$ such that $n\ge 4$, every $k\in\lcb{1,\dots,n-1}$, and every $\eps>0$,
\[
    \Pb
    \bbsb{
        \labs{\widehat G_{-\infty,k}(n)-\E\lsb{\widehat G_{-\infty,k}(n)}}>\eps
    }
\le
    2e^{-2n\eps^2}.
\]
Consequently,
\[
    \Pb\lsb{
        \max_{1\le k\le n-1}
        \labs{\widehat G_{-\infty,k}(n)-\E\lsb{\widehat G_{-\infty,k}(n)}}>\eps
    }
\le
    2ne^{-2n\eps^2}.
\]
\end{lemma}

\begin{proof}
For fixed $k\in\lcb{1,\dots,n-1}$, set
\[
    m_k\coloneqq \min\lrb{k,n-k},
\]
and write
\[
    \widehat G_{-\infty,k}(n)
=
    \frac1n\sum_{a=0}^{m_k-1} W_{k,a},
\qquad
    W_{k,a}
\coloneqq
    Y_{k-a}Z_{k+a+1}.
\]
The variables $\lrb{W_{k,a}}_{a=0}^{m_k-1}$ are independent, since the pairs of rounds
\[
    \lrb{k-a,k+a+1},
    \qquad
    a\in\{0,\dots,m_k-1\},
\]
are pairwise disjoint.
Moreover,
$
    0
\le
    \frac{W_{k,a}}{n}
\le
    \frac1n$.

Hence, Hoeffding's inequality applied to the sum
$
    \sum_{a=0}^{m_k-1} W_{k,a}/n
$
yields
\[
    \Pb\bbsb{
        \labs{\widehat G_{-\infty,k}(n)-\E\lsb{\widehat G_{-\infty,k}(n)}}>\eps
    }
\le
    2\exp \lrb{-\frac{2\eps^2}{m_k/n^2}}
\le
    2e^{-2n\eps^2},
\]
because $m_k\le n$.
The second claim follows by a union bound over $k\in\{1,\dots,n-1\}$, using $n-1\le n$.
\end{proof}

\rawlshighproboptimization*

\begin{proof}
Fix $\delta\in(0,1)$ and $n\ge4$.
Let
\[
    \bar{k}
    \in
    \argmax_{k\in\{1,\dots,n-1\}}
    G_{-\infty}\lrb{\frac{k}{n}}.
\]
By the Rawls part of \Cref{lem:holder_price_regularity}, if $k^\star\in\{1,\dots,n-1\}$ is the numerator of a grid point closest to $p^\star_{-\infty}$, and $\bar k\in\argmax_{1\le k\le n-1}G_{-\infty}(k/n)$, then
\[
    G_{-\infty}^\star
    -
    G_{-\infty}\lrb{\frac{\bar{k}}{n}}
    \le
    \frac1n.
\]
Then
\[
\begin{aligned}
G_{-\infty}^\star-G_{-\infty}(\widehat p_{-\infty}(n))
&\le
G_{-\infty}^\star-G_{-\infty}(\bar k/n)+
G_{-\infty}(\bar k/n)-\widehat G_{-\infty,\bar k}(n)\\
&\quad+
\widehat G_{-\infty,\widehat k_{-\infty}(n)}(n)
-G_{-\infty}(\widehat p_{-\infty}(n)).
\end{aligned}
\]
Using the optimality of $\widehat k_{-\infty}(n)$ for the empirical Rawls estimator and \Cref{lem:rawls_bias}, we get
\[
    G_{-\infty}^\star
    -
    G_{-\infty}\brb{\widehat p_{-\infty}(n)}
    \le
    \frac{5}{4n}
    +
    2
    \max_{1\le k\le n-1}
    \bbabs{
        \widehat G_{-\infty,k}(n)
        -
        \E\lsb{\widehat G_{-\infty,k}(n)}
    }.
\]
By \Cref{lem:rawls_concentration},
\[
    \Pb\lsb{
        \max_{1\le k\le n-1}
        \labs{
            \widehat G_{-\infty,k}(n)
            -
            \E\lsb{\widehat G_{-\infty,k}(n)}
        }
        >
        \sqrt{\frac{\ln(2n/\delta)}{2n}}
    }
    \le
    \delta.
\]
Thus, with probability at least $1-\delta$,
\[
    G_{-\infty}^\star
    -
    G_{-\infty}\brb{\widehat p_{-\infty}(n)}
    \le
    \frac{5}{4n}
    +
    \sqrt{\frac{2\ln(2n/\delta)}{n}}.
\]
Since $n\ge4$ and $\delta\in(0,1)$, we have $\ln(2n/\delta)\ge\ln8$.
Hence
\[
    \frac{5}{4n}
    +
    \sqrt{\frac{2\ln(2n/\delta)}{n}}
    \le
    2\sqrt{\frac{\ln(2n/\delta)}{n}}.
\]
This proves the claim.
\end{proof}

\subsection{The \texorpdfstring{$\lambda$}{lambda}-Free Optimization Error Bound}
\label{app:lambda:free:guarantee}

\lambdafreeoptimization*

\begin{proof}
First assume that $\lambda\in(-\infty,0]$ and $1-\lambda\le\sqrt n$.
Then $\widetilde p_\lambda(n)=\widehat p_\lambda(n)$.
By \Cref{prop:finite_lambda_high_probability_optimization}, with probability at least $1-\delta$,
\[
    G_\lambda^\star
    -
    G_\lambda\brb{\widetilde p_\lambda(n)}
    \le
    \lrb{
        62\lrb{1+\ln n}
        +
        15\frac{1-\lambda}{\sqrt n}
    }
    \sqrt{\frac{\ln(2n/\delta)}{n}}.
\]
Since $1-\lambda\le\sqrt n$ and $1+\ln n\ge1$,
\[
    62\lrb{1+\ln n}
    +
    15\frac{1-\lambda}{\sqrt n}
    \le
    77\lrb{1+\ln n}.
\]
This proves the claim in the moderate-$\lambda$ case.

Second, consider the Rawls case $\lambda=-\infty$.
Then $\widetilde p_{-\infty}(n)=\widehat p_{-\infty}(n)$.
By \Cref{prop:rawls_high_probability_optimization}, with probability at least $1-\delta$,
\[
    G_{-\infty}^\star
    -
    G_{-\infty}\brb{\widetilde p_{-\infty}(n)}
    \le
    2\sqrt{\frac{\ln(2n/\delta)}{n}}.
\]
Since $77\lrb{1+\ln n}\ge2$, this is bounded by the displayed radius.

Finally, consider $\lambda\in(-\infty,0]$ such that $1-\lambda>\sqrt n$.
Then $\widetilde p_\lambda(n)=\widehat p_{-\infty}(n)$.
Since $n\ge4$,
\[
    -\lambda
    =
    1-\lambda-1
    >
    \sqrt n-1
    \ge
    \frac{\sqrt n}{2}.
\]
Thus
\[
    2^{-1/\lambda}-1
    =
    \exp\lrb{\frac{\ln2}{-\lambda}}-1
    \le
    \frac{4\ln2}{\sqrt n}
    \le
    \frac{3}{\sqrt n}.
\]
By \Cref{lem:rawls_approx}, on the event of \Cref{prop:rawls_high_probability_optimization},
\[
    G_\lambda^\star
    -
    G_\lambda\brb{\widehat p_{-\infty}(n)}
    \le
    2\sqrt{\frac{\ln(2n/\delta)}{n}}
    +
    \frac{2^{-1/\lambda}-1}{2}.
\]
Using the previous bound, we get
\[
    G_\lambda^\star
    -
    G_\lambda\brb{\widehat p_{-\infty}(n)}
    \le
    2\sqrt{\frac{\ln(2n/\delta)}{n}}
    +
    \frac{3}{2\sqrt n}.
\]
Since $n\ge4$ and $\delta\in(0,1)$, we have $\ln(2n/\delta)\ge\ln8$.
Therefore
\[
    \frac{3}{2\sqrt n}
    \le
    2\sqrt{\frac{\ln(2n/\delta)}{n}}.
\]
Hence, with probability at least $1-\delta$,
\[
    G_\lambda^\star
    -
    G_\lambda\brb{\widetilde p_\lambda(n)}
    \le
    4\sqrt{\frac{\ln(2n/\delta)}{n}}.
\]
Since
$
    4
    \le
    77\lrb{1+\ln n}$,
the displayed radius in \Cref{prop:lambda_free_exploration_optimization} dominates this bound.
This concludes the proof.
\end{proof}

\subsection{Uniform PAC Guarantees over the Fairness Level}
\label{app:PAC:guarantee}

\begin{proposition}[Finite-Target PAC Guarantees in the Moderate Regime]
\label{prop:moderate_pac}
There exists a universal constant $C_{\rm mod}>0$ such that the following holds.
For every $\eps\in(0,1)$, $\delta\in(0,1)$, every $n\ge4$, and every $\lambda\in(-\infty,0]$ satisfying
\[
    1-\lambda\le\sqrt n,
\]
if
\[
    n
    \ge
    C_{\rm mod}\frac{(1+\ln n)^2}{\eps^2}\ln\frac{16n}{\delta},
\]
then
\[
    \Pb\lsb{
        G_\lambda^\star-G_\lambda\brb{\widehat p_\lambda(n)}\le \eps
    }
    \ge
    1-\delta.
\]
Moreover, one may take $C_{\rm mod}=6000$.
\end{proposition}

\begin{proof}
By \Cref{prop:finite_lambda_high_probability_optimization}, and using $1-\lambda\le\sqrt n$, with probability at least $1-\delta$,
\[
    G_\lambda^\star-G_\lambda\brb{\widehat p_\lambda(n)}
    \le
    77\lrb{1+\ln n}
    \sqrt{\frac{\ln(2n/\delta)}{n}}.
\]
Since $\ln(2n/\delta)\le\ln(16n/\delta)$ and $77^2\le6000$, the displayed sample-size condition implies that the right-hand side is at most $\eps$.
\end{proof}

\begin{proposition}[Rawls PAC Guarantees]
\label{prop:rawls_pac}
There exists a universal constant $C_{\rm Rawls}>0$ such that the following holds.
For every $\eps\in(0,1)$, $\delta\in(0,1)$, and every $n\in\N$ satisfying
\begin{equation}
\label{eq:minimum_n_in_Rawls}
    n
\ge
    \max \lrb{4,\frac{C_{\rm Rawls}}{\eps^2}\ln\frac{C_{\rm Rawls}}{\eps^2\delta}},
\end{equation}
we have
\[
    \Pb\bbsb{
        G_{-\infty}^\star
        -
        G_{-\infty}\brb{\widehat p_{-\infty}(n)}
        \le
        \frac{\eps}{2}
    }
\ge
    1-\delta.
\]
Moreover, one may take $C_{\rm Rawls}=64$.
\end{proposition}

\begin{proof}
Fix $\eps\in(0,1)$ and $\delta\in(0,1)$.
Set
\[
    A\coloneqq \frac{64}{\eps^2\delta},
    \qquad
    n_0\coloneqq \frac{64}{\eps^2}\ln A .
\]
Assume
$
    n\ge \max\{4,n_0\}$.
We first show that
\[
    2\sqrt{\frac{\ln(2n/\delta)}{n}}
    \le
    \frac{\eps}{2}.
\]
The function
\[
    u\mapsto \frac{\ln(2u/\delta)}{u}
\]
is nonincreasing on the range under consideration, since
$\ln(2u/\delta)\ge1$.
Therefore,
\[
    \frac{\ln(2n/\delta)}{n}
    \le
    \frac{\ln(2n_0/\delta)}{n_0}.
\]
Now
\[
    \ln\frac{2n_0}{\delta}
    =
    \ln\lrb{
        \frac{128}{\eps^2\delta}\ln A
    }
    =
    \ln(2A\ln A)
    =
    \ln A+\ln(2\ln A).
\]
Since $A\ge64$, we have
$
    \ln(2\ln A)\le 3\ln A,
$
and hence
\[
    \ln\frac{2n_0}{\delta}
    \le
    4\ln A.
\]
Thus
\[
    \frac{\ln(2n/\delta)}{n}
    \le
    \frac{4\ln A}{(64/\eps^2)\ln A}
    =
    \frac{\eps^2}{16}.
\]
Consequently,
\[
    2\sqrt{\frac{\ln(2n/\delta)}{n}}
    \le
    \frac{\eps}{2}.
\]

By \Cref{prop:rawls_high_probability_optimization}, with probability at least $1-\delta$,
\[
    G_{-\infty}^\star
    -
    G_{-\infty}\brb{\widehat p_{-\infty}(n)}
    \le
    2\sqrt{\frac{\ln(2n/\delta)}{n}}
    \le
    \frac{\eps}{2}.
\]
This proves the proposition with $C_{\rm Rawls}=64$.
\end{proof}

\noindent\emph{Uniform PAC guarantee.}

Fix $\eps\in(0,1)$ and write
\[
    A_\eps
    \coloneqq
    \frac{\ln2}{\ln(1+\eps)}
    \qquad\textrm{and}\qquad
    J_\eps
    \coloneqq
    \left\lceil
        \frac{4\ln2}{\eps\ln(1+\eps)}
    \right\rceil.
\]
Then in \Cref{thm:uniform_pac_grid}, we have
\[
    \Gamma_\eps
    \coloneqq
    \lcb{
        -j\frac{\eps}{4}
        :
        j=0,\dots,J_\eps
    }.
\]
For every $\lambda\in[-A_\eps,0]$, recall that $\lambda_\eps$ is the closest point to $\lambda$ in $\Gamma_\eps$, with ties broken toward $0$.
Recall that the recommender in \Cref{thm:uniform_pac_grid}
is defined by
\[
    \widehat p_{\lambda,\eps}(n)
    \coloneqq
    \begin{cases}
    \widehat p_{-\infty}(n),
    &
    \textrm{if }-\infty\le \lambda<-A_\eps,\\[2mm]
    \widehat p_{\lambda_\eps}(n),
    &
    \textrm{if }-A_\eps\le \lambda\le0.
    \end{cases}
\]

\begin{lemma}[A Sufficient Explicit Condition for the Uniform Grid Bound]
\label{lem:uniform_grid_condition_from_explicit}
There exists a universal constant $C_{\rm grid}>0$ such that the following holds.
For every $\eps\in(0,1)$, every $\delta\in(0,1)$, every $C\ge C_{\rm grid}$, and every $n\in\N$, if
\[
    n
    \ge
    \frac{C}{\eps^2}
    \lrb{\ln\frac{C}{\eps^4\delta}}^3,
\]
then
\[
    n
    \ge
    \max\lcb{
        4,
        \frac{9}{\eps^2},
        \frac{64}{\eps^2}
        \ln\frac{704}{\eps^4\delta},
        \frac{C}{36\eps^2}
        (1+\ln n)^2
        \ln\frac{176n}{\eps^2\delta}
    }.
\]
Moreover, one may take
\[
    C_{\rm grid}\coloneqq \max(e^6,704) = 704.
\]
\end{lemma}

\begin{proof}
Fix $C_{\rm grid}\coloneqq 704$ and $C \ge C_{\rm grid}$.
Set
\[
    L
    \coloneqq
    \ln\frac{C}{\eps^4\delta},
    \qquad
    n_0
    \coloneqq
    \frac{C}{\eps^2}L^3.
\]
Assume
$
    n\ge n_0
$.
Since $\eps,\delta\in(0,1)$ and $C\ge e^6$, we have
\[
    L\ge \ln C\ge 6,
    \qquad
    n_0\ge C\ge e^6.
\]

First, since $C\ge9$ and $L\ge1$, we have
\[
    n
    \ge
    n_0
    =
    \frac{C}{\eps^2}L^3
    \ge
    \frac{9}{\eps^2}
    \ge
    4.
\]
Moreover, since $C\ge704$, we have
\[
    L
    =
    \ln\frac{C}{\eps^4\delta}
    \ge
    \ln\frac{704}{\eps^4\delta}.
\]
Thus, using $C\ge64$ and $L\ge1$,
\[
    n
    \ge
    \frac{C}{\eps^2}L^3
    \ge
    \frac{64}{\eps^2}
    \ln\frac{704}{\eps^4\delta}.
\]

It remains to control the logarithmic term.
Consider the function
\[
    \psi(u)
    \coloneqq
    \frac{
        (1+\ln u)^2
        \ln\frac{176u}{\eps^2\delta}
    }{u},
    \qquad
    u\ge e^6.
\]
We claim that $\psi$ is non-increasing on $\lsb{e^6,+\infty}$.
Indeed, setting $z\coloneqq\ln u$, we can write
\[
    \psi(u)
    =
    (1+z)^2
    \lrb{\ln\frac{176}{\eps^2\delta}+z}
    e^{-z}.
\]
Hence
\[
    \frac{\dif}{\dif z}\ln\psi
    =
    \frac{2}{1+z}
    +
    \frac{1}{\ln\frac{176}{\eps^2\delta}+z}
    -
    1.
\]
For $z\ge6$, and since $\ln(176/(\eps^2\delta))\ge\ln176>5$, we get
\[
    \frac{2}{1+z}
    +
    \frac{1}{\ln\frac{176}{\eps^2\delta}+z}
    -
    1
    \le
    \frac{2}{7}
    +
    \frac{1}{11}
    -
    1
    <
    0.
\]
Thus $\psi$ is non-increasing on $\lsb{e^6,+\infty}$.
Since $n\ge n_0\ge e^6$, it follows that
\[
    \psi(n)\le\psi(n_0).
\]

We now bound $\psi(n_0)$.
First,
\[
    \ln n_0
    =
    \ln\frac{C}{\eps^2}
    +
    3\ln L.
\]
Moreover,
\[
    \ln\frac{C}{\eps^2}
    =
    \ln\frac{C}{\eps^4\delta}
    +
    \ln\lrb{\eps^2\delta}
    \le
    L.
\]
Since $L\ge6$, we also have $3\ln L\le L$.
Therefore
\[
    \ln n_0\le 2L
    \qquad\textrm{and}\qquad
    1+\ln n_0\le 3L.
\]
Furthermore,
\[
    \ln\frac{176n_0}{\eps^2\delta}
    =
    \ln176
    +
    \ln\frac{C}{\eps^4\delta}
    +
    3\ln L
    =
    \ln176+L+3\ln L
    \le
    3L,
\]
where we used $\ln176\le L$ and $3\ln L\le L$.
Consequently,
\[
    \psi(n)
    \le
    \psi(n_0)
    \le
    \frac{(3L)^2\cdot 3L}{n_0}
    =
    \frac{27L^3}{n_0}
    =
    \frac{27\eps^2}{C}
    \le
    \frac{36\eps^2}{C}.
\]
Equivalently,
\[
    \frac{
        (1+\ln n)^2
        \ln\frac{176n}{\eps^2\delta}
    }{n}
    \le
    \frac{36\eps^2}{C},
\]
and hence
\[
    n
    \ge
    \frac{C}{36\eps^2}
    (1+\ln n)^2
    \ln\frac{176n}{\eps^2\delta}.
\]
This proves the claim.
\end{proof}

\uniformpacgrid*

\begin{proof}
Let
\[
    C_{\rm unif}
    \ge
    \max\lcb{
        C_{\rm grid},
        144C_{\rm mod}
    }.
\]
Assume
\[
    n
    \ge
    \frac{C_{\rm unif}}{\eps^2}
    \lrb{\ln\frac{C_{\rm unif}}{\eps^4\delta}}^3.
\]
By \Cref{lem:uniform_grid_condition_from_explicit}, applied with
$
    C=C_{\rm unif}
$,
we have
\[
    n
    \ge
    \max\lcb{
        4,
        \frac{9}{\eps^2},
        \frac{64}{\eps^2}
        \ln\frac{704}{\eps^4\delta},
        \frac{C_{\rm unif}}{36\eps^2}
        (1+\ln n)^2
        \ln\frac{176n}{\eps^2\delta}
    }.
\]
Since $C_{\rm unif}\ge144C_{\rm mod}$, we have
$
    \frac{C_{\rm unif}}{36}
    \ge
    4C_{\rm mod}$.
Hence
\begin{equation}
\label{eq:uniform_implicit_condition}
    n
    \ge
    \max\lcb{
        4,
        \frac{9}{\eps^2},
        \frac{64}{\eps^2}
        \ln\frac{704}{\eps^4\delta},
        4C_{\rm mod}
        \frac{(1+\ln n)^2}{\eps^2}
        \ln\frac{176n}{\eps^2\delta}
    }.
\end{equation}

We next record the size and approximation properties of the grid.
By construction of $\Gamma_\eps$, consecutive grid points are at distance $\eps/4$.
Therefore, for every $\lambda\in[-\ln2/\ln(1+\eps),0]$, its closest grid point $\lambda_\eps$ satisfies
\[
    \labs{\lambda-\lambda_\eps}
    \le
    \frac{\eps}{4}.
\]
Applying \Cref{lem:M_lipschitz_in_alpha} with
$
    \alpha=-\lambda
$
and
$
    \beta=-\lambda_\eps
$,
we get
\begin{equation}
\label{eq:grid_approx_M}
    \sup_{x,y\in[0,1]}
    \labs{
        M_\lambda(x,y)
        -
        M_{\lambda_\eps}(x,y)
    }
    \le
    \frac{\eps}{4},
    \qquad
    \forall \lambda\in[-\ln2/\ln(1+\eps),0].
\end{equation}
Furthermore, since
$
    \ln(1+\eps)\ge\eps/2
$
for $\eps\in(0,1)$, we have
\[
    J_\eps
    =
    \left\lceil
        \frac{4\ln2}{\eps\ln(1+\eps)}
    \right\rceil
    \le
    \left\lceil
        \frac{8\ln2}{\eps^2}
    \right\rceil
    \le
    \left\lceil
        \frac{8}{\eps^2}
    \right\rceil.
\]
Thus
\[
    \labs{\Gamma_\eps}
    =
    J_\eps+1
    \le
    \left\lceil\frac{8}{\eps^2}\right\rceil+1
    \le
    \frac{10}{\eps^2}.
\]
Writing
$
    N\coloneqq\labs{\Gamma_\eps}$,
we get $
    N+1
    \le
    \frac{11}{\eps^2}$.
Set
\[
    \delta_0
    \coloneqq
    \frac{\delta}{N+1}.
\]

We now define the good events.
Let
\[
    \cE_{-\infty}
    \coloneqq
    \Bcb{
        G_{-\infty}^\star
        -
        G_{-\infty}\brb{\widehat p_{-\infty}(n)}
        \le
        \frac{\eps}{2}
    }.
\]
By \Cref{prop:rawls_pac}, the event $\cE_{-\infty}$ has probability at least
$1-\delta_0$, provided
\[
    n
    \ge
    \frac{64}{\eps^2}\ln\frac{64}{\eps^2\delta_0}
    =
    \frac{64}{\eps^2}
    \ln\frac{64(N+1)}{\eps^2\delta}.
\]
Since
$
    N+1
    \le
    \frac{11}{\eps^2}$,
we have
\[
    \frac{64(N+1)}{\eps^2\delta}
    \le
    \frac{704}{\eps^4\delta}.
\]
Thus the Rawls condition is implied by
\[
    n
    \ge
    \frac{64}{\eps^2}
    \ln\frac{704}{\eps^4\delta},
\]
which holds by \eqref{eq:uniform_implicit_condition}.
Therefore
\[
    \Pb\lsb{\cE_{-\infty}}
    \ge
    1-\delta_0.
\]

Write
\[
    \Gamma_\eps
    =
    \{\lambda_1,\dots,\lambda_N\}.
\]
For every $r\in\{1,\dots,N\}$, define
\[
    \cE_r
    \coloneqq
    \Bcb{
        G_{\lambda_r}^\star
        -
        G_{\lambda_r}\brb{\widehat p_{\lambda_r}(n)}
        \le
        \frac{\eps}{2}
    }.
\]
We now check that \Cref{prop:moderate_pac} applies to every grid point $\lambda_r$.
Since
$
    \lambda_r\in[-J_\eps\eps/4,0]
$
and
$
    J_\eps\eps/4
    \le
    \ln2/\ln(1+\eps)+\eps/4
$,
we have
\[
    1-\lambda_r
    \le
    1+\frac{\ln2}{\ln(1+\eps)}+\frac{\eps}{4}
    \le
    1+\frac{2\ln2}{\eps}+\frac{\eps}{4}
    \le
    \frac{3}{\eps}.
\]
Since $n\ge 9/\eps^2$ by \eqref{eq:uniform_implicit_condition}, it follows that
$
    1-\lambda_r
    \le
    \sqrt n$.
Moreover, using again
$
    N+1\le 11/\eps^2
$,
we have
\[
    \ln\frac{16n}{\delta_0}
    =
    \ln\frac{16n(N+1)}{\delta}
    \le
    \ln\frac{176n}{\eps^2\delta}.
\]
By \eqref{eq:uniform_implicit_condition},
\[
    n
    \ge
    4C_{\rm mod}
    \frac{(1+\ln n)^2}{\eps^2}
    \ln\frac{176n}{\eps^2\delta}
    \ge
    4C_{\rm mod}
    \frac{(1+\ln n)^2}{\eps^2}
    \ln\frac{16n}{\delta_0}.
\]
Equivalently,
\[
    n
    \ge
    C_{\rm mod}
    \frac{(1+\ln n)^2}{(\eps/2)^2}
    \ln\frac{16n}{\delta_0}.
\]
Therefore, applying \Cref{prop:moderate_pac} with accuracy parameter $\eps/2$ and confidence parameter $\delta_0$, we get
\[
    \Pb\lsb{\cE_r}
    \ge
    1-\delta_0
    \qquad
    \forall r\in\{1,\dots,N\}.
\]

Define the simultaneous good event
\[
    \cE
    \coloneqq
    \cE_{-\infty}
    \cap
    \bigcap_{r=1}^{N}\cE_r.
\]
By the union bound,
\[
    \Pb\lsb{\cE}
    \ge
    1-(N+1)\delta_0
    =
    1-\delta.
\]

It remains to prove that, on $\cE$, the desired uniform PAC guarantee holds.
Fix $\lambda\in[-\infty,0]$.

\medskip
\noindent
\emph{Case 1: $-\infty\le\lambda<-\ln2/\ln(1+\eps)$.}
By definition,
\[
    \widehat p_{\lambda,\eps}(n)
    =
    \widehat p_{-\infty}(n).
\]
If $\lambda=-\infty$, then on $\cE$,
\[
    G_{-\infty}^\star
    -
    G_{-\infty}\brb{\widehat p_{-\infty}(n)}
    \le
    \frac{\eps}{2}
    \le
    \eps.
\]
Assume now that $\lambda\in(-\infty,-\ln2/\ln(1+\eps))$.
Since
$
    \lambda
    <
    -\frac{\ln2}{\ln(1+\eps)}$,
we have
\[
    2^{-1/\lambda}-1<\eps.
\]
Hence, by \Cref{lem:rawls_approx}, on $\cE$,
\begin{align*}
    G_\lambda^\star
    -
    G_\lambda\brb{\widehat p_{\lambda,\eps}(n)}
=
    G_\lambda^\star
    -
    G_\lambda\brb{\widehat p_{-\infty}(n)}
\le
    G_{-\infty}^\star
    -
    G_{-\infty}\brb{\widehat p_{-\infty}(n)}
    +
    \frac{2^{-1/\lambda}-1}{2}
\le
    \frac{\eps}{2}
    +
    \frac{\eps}{2}
    =
    \eps.
\end{align*}

\medskip
\noindent
\emph{Case 2: $-\ln2/\ln(1+\eps)\le\lambda\le0$.}
Set
$
    \mu
    \coloneqq
    \lambda_\eps\in\Gamma_\eps$.
By \eqref{eq:grid_approx_M},
\[
    \sup_{x,y\in[0,1]}
    \labs{
        M_\lambda(x,y)-M_\mu(x,y)
    }
    \le
    \frac{\eps}{4}.
\]
For every $p\in[0,1]$, we have
\[
    \labs{G_\lambda(p)-G_\mu(p)}
    \le
    \frac{\eps}{4}.
\]
Indeed, for every $p,s,b\in[0,1]$, from \Cref{lem:holder_lipschitz_fairness_parameter},
\[
    \labs{
        \gft_\lambda(p,s,b)
        -
        \gft_\mu(p,s,b)
    }
    \le
    \sup_{x,y\in[0,1]}
    \labs{
        M_\lambda(x,y)-M_\mu(x,y)
    }
    \le
    \frac{\eps}{4},
\]
and taking expectations gives the claim.
Therefore,
$    G_\lambda^\star
    \le
    G_\mu^\star+\frac{\eps}{4}$,
and 
\[
    -G_\lambda\brb{\widehat p_{\mu}(n)}
    \le
    -G_\mu\brb{\widehat p_{\mu}(n)}
    +
    \frac{\eps}{4}.
\]
Since
$
    \widehat p_{\lambda,\eps}(n)
    =
    \widehat p_\mu(n)$,
we obtain, on $\cE$,
\begin{align*}
    G_\lambda^\star
    -
    G_\lambda\brb{\widehat p_{\lambda,\eps}(n)}
=
    G_\lambda^\star
    -
    G_\lambda\brb{\widehat p_\mu(n)}
\le
    G_\mu^\star
    -
    G_\mu\brb{\widehat p_\mu(n)}
    +
    \frac{\eps}{2}
\le
    \frac{\eps}{2}
    +
    \frac{\eps}{2}
    =
    \eps,
\end{align*}
where the last inequality follows from the event $\cE_r$ corresponding to the grid point $\mu=\lambda_r$.

Since the argument holds for every $\lambda\in[-\infty,0]$ on the event $\cE$, and since $\Pb\lsb{\cE}\ge1-\delta$, we conclude that
\[
    \Pb\Bsb{
        \forall \lambda\in[-\infty,0],\quad
        G_\lambda^\star
        -
        G_\lambda\brb{\widehat p_{\lambda,\eps}(n)}
        \le
        \eps
    }
    \ge
    1-\delta.
\]
This proves the theorem.

Finally, recalling that one may take
$
    C_{\rm grid}=e^6\vee704
$
and, with the constants above,
$
    C_{\rm mod}=6000
$,
we may take
\[
    C_{\rm unif}
    =
    \max\lcb{
        e^6\vee704,
        144\cdot6000
    }
    =
    864000.
\]
\end{proof}


\section{Proofs for Regret Guarantees}
\label{app:etc_regret}

\begin{lemma}[Generic ETC Bound from a PAC Recommender]
\label{lem:generic_etc_from_pac}
Fix $\lambda\in[-\infty,0]$, $T\in\N$, $n\in\{1,\dots,T\}$, $\eps\in[0,1]$, and $\delta\in(0,1)$.
Consider a policy that explores for $n$ rounds, produces a recommendation $\widetilde p$, and then posts $\widetilde p$ during rounds $n+1,\dots,T$.
Assume that
\[
    \Pb\Bsb{
        G_\lambda^\star
        -
        G_\lambda\brb{\widetilde p}
        \le
        \eps
    }
\ge
    1-\delta.
\]
Then this policy has regret at most
\[
    n + (T-n)\eps + (T-n)\delta
    \le
    n + T\eps + T\delta.
\]
\end{lemma}

\begin{proof}
Since the rewards are bounded in $\lsb{0,1}$, the instantaneous regret is at most $1$.
Thus the regret accumulated during the first $n$ exploration rounds is at most $n$.
During the remaining $T-n$ rounds, the policy repeatedly plays $\widetilde p$.
Therefore the regret is at most
\begin{equation}
\label{eq:auxiliary_regret_bound}
        n
    +
    (T-n)
    \E\lsb{
        G_\lambda^\star
        -
        G_\lambda\brb{\widetilde p}
    }.
\end{equation}
Let
$
    \cE
    \coloneqq
    \lcb{
        G_\lambda^\star
        -
        G_\lambda\brb{\widetilde p}
        \le
        \eps
    }$.
By assumption, $\Pb\lsb{\cE}\ge 1-\delta$.
Moreover,
\[
    0
\le
    G_\lambda^\star
    -
    G_\lambda\brb{\widetilde p}
\le
    1
    \qquad
    \textrm{almost surely}.
\]
Hence
$
    \E\Bsb{
        G_\lambda^\star
        -
        G_\lambda\brb{\widetilde p}
    }
    \le
    \eps+\delta$.
Plugging this bound into \eqref{eq:auxiliary_regret_bound} gives
\[
    n+(T-n)(\eps+\delta)
    =
    n+(T-n)\eps+(T-n)\delta
    \le
    n+T\eps+T\delta.
\]
This proves the claim.
\end{proof}

\regretupperbound*

\begin{proof}
Fix $\lambda\in[-\infty,0]$ and $T\ge 4$.
Set
\[
    a_T:=64T^{2/3}\ln T .
\]
Since $T\ge 4$, we have $a_T\ge 4$, and therefore
\[
    n_T=\min\Brb{T,\lce{a_T}}.
\]
In particular,
\[
    n_T\le T,
    \qquad
    n_T\le a_T+1.
\]

We first consider the case $n_T=T$.
Since the rewards are bounded in $[0,1]$, the instantaneous regret is at most $1$.
Hence
\[
    R_\lambda\Brb{\mathrm{ETC}_\lambda(T),T}
    \le T
    =
    n_T
    \le
    2n_T+1.
\]

We now consider the case $n_T<T$.
Then $n_T=\lce{a_T}$, and in particular
\[
    n_T\ge a_T=64T^{2/3}\ln T.
\]
Applying \Cref{prop:lambda_free_exploration_optimization} with $\delta=1/T$ to the branch selected by Algorithm~\ref{algo:etc}, we get
\[
    \Pb\lsb{
        G_\lambda^\star
        -
        G_\lambda\brb{\widehat p_T}
        \le
        77\lrb{1+\ln n_T}
        \sqrt{\frac{\ln(2n_TT)}{n_T}}
    }
    \ge
    1-\frac1T.
\]
Since $n_T\le T$ and $T\ge 4$, we have
\[
    1+\ln n_T\le 1+\ln T\le 2\ln T,
    \qquad
    \ln(2n_TT)\le \ln(2T^2)\le 3\ln T,
\]
and therefore
\[
    77\lrb{1+\ln n_T}
    \sqrt{\frac{\ln(2n_TT)}{n_T}}
    \le
    154\sqrt{\frac{3}{64}}\,
    T^{-1/3}\ln T
    \le
    64T^{-1/3}\ln T
    =
    \frac{a_T}{T}
    \le
    \frac{n_T}{T}.
\]
Thus the committed price is $n_T/T$-optimal for $G_\lambda$ with probability at least $1-1/T$.
Applying \Cref{lem:generic_etc_from_pac} with $n=n_T$, $\eps=n_T/T$, and $\delta=1/T$ yields
\[
    R_\lambda\Brb{\mathrm{ETC}_\lambda(T),T}
    \le
    n_T+T\frac{n_T}{T}+1
    =
    2n_T+1.
\]

Combining the two cases, we have proved
\[
    R_\lambda\Brb{\mathrm{ETC}_\lambda(T),T}
    \le
    2n_T+1.
\]
Finally, since $n_T\le a_T+1$, we obtain
\[
    2n_T+1
    \le
    2\lrb{64T^{2/3}\ln T+1}+1
    =
    128T^{2/3}\ln T+3.
\]
This concludes the proof.
\end{proof}

\section[A Generic T23 Regret Lower Bound Template]{A Generic \texorpdfstring{$T^{2/3}$}{T^(2/3)} Regret Lower Bound Template}
\label{app:lower_bound_template}

\noindent\emph{Probability space, environment, and learner.}
Fix $T \in \N$ and $\eps_0>0$.
Let $\cX$ be an action set and let $\mathfrak{X}$ be a $\sigma$-algebra of subsets of $\cX$.
Let $\cX_{ -},\cX_{ +},\cX_{\mathrm{info}} \in \mathfrak{X}$ be non-empty and disjoint such that
\[
\cX
=
\cX_{ -}
\cup
\cX_{ +}
\cup
\cX_{\mathrm{info}}.
\]
Define the environment space $(\cY,\mathfrak{Y})$ to be a measurable space.
Let $\mathfrak{B}$ be the Borel $\sigma$-algebra on $[0,1]$.
Define the reward map $\rho \colon \cX \times \cY \to [0,1]$ to be a $\lrb{\mathfrak{X} \otimes \mathfrak{Y}}/\mathfrak{B}$-measurable map.
Define the feedback space $(\cF,\mathfrak{F})$ to be a countably generated measurable space (i.e., the $\sigma$-algebra $\mathfrak{F}$ can be generated by a countable subfamily of $\mathfrak{F}$).
Define the feedback map $\varphi \colon \cX \times \cY \to \cF$ to be a $\lrb{\mathfrak{X} \otimes \mathfrak{Y}}/\mathfrak{F}$-measurable map.
Let $(\Omega,\mathfrak{E},\Pb)$ be a probability space.
For each $\eps \in [-\eps_0,\eps_0]$, assume that $(Y_t^\eps)_{t \in [T]}$ is a $\cY$-valued $\mathfrak{E}/\mathfrak{Y}$-measurable $\Pb$-i.i.d.\ sequence (the environment process).
Assume that $(U_t)_{t \in [T]}$ is a $[0,1]$-valued $\mathfrak{E}/\mathfrak{B}$-measurable $\Pb$-i.i.d.\ sequence of uniform random variables (the random seeds process sequentially used by the learner algorithm).
Assume that the two families $(Y_t^\eps)_{t \in [T], \eps \in [-\eps_0,\eps_0]}$ and $(U_t)_{t \in [T]}$ are $\Pb$-independent of each other.
For each $\eps \in [-\eps_0,\eps_0]$, and each $t \in [T]$, and each $x \in \cX$, define the reward and feedback associated to the action $x$ at time $t$ if the environment is determined by $\eps$ as, respectively
\[
G^\eps_t(x) \coloneqq \rho(x,Y_t^\eps), \qquad F^\eps_t(x) \coloneqq \varphi(x,Y_t^\eps).
\]
A (randomized) learner algorithm $\alpha\ceq(\alpha_t)_{t\in[T]}$ is a sequence of $\lrb{(\mathfrak{B}\otimes\mathfrak{F})^{\otimes(t-1)} \otimes \mathfrak{B}} / \mathfrak{X}$-measurable maps
\[
\alpha_t\coloneqq \brb{[0,1]\times\cF}^{t-1} \times [0,1]\to\cX.
\]
If the learner uses $\alpha$ and interacts with an environment determined by $\eps\in[-\eps_0,\eps_0]$, then for $t=1,\dots,T$ the learner chooses (recursively)
\[
X_t^\eps
\coloneqq
\alpha_t\brb{U_1,F_1^\eps(X_1^\eps),\dots,U_{t-1},F_{t-1}^\eps(X_{t-1}^\eps),U_t},
\]
receives (but does not observe) reward $G_t^\eps(X_t^\eps)$, and observes feedback $F_t^\eps(X_t^\eps)$.

For each $\eps \in [-\eps_0,\eps_0]$, define the expected reward function
\[
\rho^\eps
\colon
\cX \to [0,1],\qquad  x\mapsto \E\bsb{G_t^\eps(x)},
\]
which we observe is well defined (i.e., it does not depend on $t$) because the sequence $\brb{G_t^\eps(x)}_{t \in [T]}$ is $\Pb$-i.i.d., and we observe it is $\mathfrak{X}/\mathfrak{B}$-measurable.

For brevity, for each $\eps \in [-\eps_0,\eps_0]$ and each $t \in [T]$, define the observed feedback variable
\[
O_t^\eps
\coloneqq
F_t^\eps(X_t^\eps)\in\cF.
\]
For $t\ge 2$, define the history space
\[
\cH_t\coloneqq ([0,1]\times\cF)^{t-1}\times[0,1],
\]
and define the history observed by the learner just before receiving the feedback at time $t$ when the underlying environment is determined by $\eps \in [-\eps_0,\eps_0]$ as
\[
H_t^\eps
\coloneqq
(U_1,O_1^\eps,\dots,U_{t-1},O_{t-1}^\eps,U_t)\in\cH_t,
\qquad
H_1^\eps
\coloneqq
U_1.
\]
With this convention, notice that the action can be written as
\[
X_t^\eps
=
\alpha_t(H_t^\eps).
\]
\noindent\emph{Regret.}
For each $\eps \in [-\eps_0,\eps_0]$, define $\rho^{\eps}_\star\ceq \sup_{x\in\cX}\rho^\eps(x)$, and define the regret of the learner algorithm $\alpha$ when the underlying environment is determined by $\eps$ as
\[
R^\eps(\alpha,T)
\coloneqq
\E\lsb{\sum_{t=1}^T\brb{\rho^{\eps}_\star-\rho^\eps(X_t^\eps)}}.
\]
\begin{assumption}[Standing Assumptions for the Pair $\pm\eps$]\label{ass:t23}
	There exist universal constants
	$c_{\mathrm{opt}},c_{\mathrm{gap}},c_{\mathrm{info}}>0$
	such that, for every $\eps\in(0,\eps_0]$, the following hold.
	\begin{enumerate}[label=(A\arabic*),leftmargin=2.2em]
		\item \emph{Optimal side.}
		\[
		\sup_{x \in \cX_{ +}}\rho^{+\eps}(x)
		= 
		\sup_{x \in \cX}\rho^{+\eps}(x),
		\qquad
		\sup_{x \in \cX_{ -}}\rho^{-\eps}(x)
		= 
		\sup_{x \in \cX}\rho^{-\eps}(x).
		\]
		\item \emph{Linear reward separation between sides.}
		\[
		\sup_{x\in\cX_{ +}}\rho^{+\eps}(x)-\sup_{x\in\cX_{ -}}\rho^{+\eps}(x)
		\ge
		c_{\mathrm{opt}}\eps,
		\qquad
		\sup_{x\in\cX_{ -}}\rho^{-\eps}(x)-\sup_{x\in\cX_{ +}}\rho^{-\eps}(x)\ge c_{\mathrm{opt}}\eps.
		\]
		\item \emph{Info actions are uniformly worse.}
		For every $x\in\cX_{\mathrm{info}}$,
		\[
		\rho^{+\eps}_\star-\rho^{+\eps}(x)
		\ge
		c_{\mathrm{gap}},
		\qquad
		\rho^{-\eps}_\star-\rho^{-\eps}(x)
		\ge
		c_{\mathrm{gap}}.
		\]
		\item \emph{Outside info, feedback does not depend on the sign.}
		For every $t\in[T]$ and every $x\in\cX\setminus\cX_{\mathrm{info}}$,
		\[
		\Pb_{F_t^{+\eps}(x)}
		=
		\Pb_{F_t^{-\eps}(x)}.
		\]
		\item \emph{Info actions carry limited information.}
		For every $t\in[T]$ and every $x\in\cX_{\mathrm{info}}$,
		\[
		\KL\lrb{\Pb_{F_t^{+\eps}(x)}\Vert \Pb_{F_t^{-\eps}(x)}}
		\le
		c_{\mathrm{info}}\eps^2.
		\]
	\end{enumerate}
\end{assumption}


For $\eps\in[-\eps_0,\eps_0]$ and $t\in[T]$ define the indicators
\[
I_{t,\mathrm{info}}^\eps\ceq \I\lcb{X_t^\eps\in\cX_{\mathrm{info}}},
\qquad
I_{t, +}^\eps\ceq \I\lcb{X_t^\eps\in\cX_{ +}},
\qquad
I_{t, -}^\eps\ceq \I\lcb{X_t^\eps\in\cX_{ -}},
\]
and the corresponding counters
\[
N_{\mathrm{info}}^\eps
\coloneqq
\sum_{t=1}^T I_{t,\mathrm{info}}^\eps,
\qquad
N_{ +}^\eps
\coloneqq
\sum_{t=1}^T I_{t, +}^\eps,
\qquad
N_{-}^\eps
\coloneqq
\sum_{t=1}^T I_{t, -}^\eps.
\]
Given that $\cX_{ -},\cX_{ +},\cX_{\mathrm{info}}$ form a partition of $\cX$, we have that
$N_{\mathrm{info}}^\eps+N_{ +}^\eps+N_{ -}^\eps=T$.


\begin{lemma}[Regret Lower Bounds via Counts]\label{lem:t23-regret-counts}
	Assume Assumption~\ref{ass:t23}. For every $\eps\in(0,\eps_0]$,
	\[
	R^{+\eps}(\alpha,T)
	\ge
	c_{\mathrm{gap}}\E\lsb{N_{\mathrm{info}}^{+\eps}}
	+
	c_{\mathrm{opt}}\eps\E\lsb{N_{ -}^{+\eps}},
	\]
	\[
	R^{-\eps}(\alpha,T)
	\ge
	c_{\mathrm{gap}}\E\lsb{N_{\mathrm{info}}^{-\eps}}
	+
	c_{\mathrm{opt}}\eps\E\lsb{N_{ +}^{-\eps}}.
	\]
\end{lemma}

\begin{proof}
	Fix $\eps\in(0,\eps_0]$. For $t\in[T]$,
	if $X_t^{+\eps}\in\cX_{\mathrm{info}}$, then Assumption~\ref{ass:t23} (A3) gives
	$\rho^{+\eps}_\star-\rho^{+\eps}(X_t^{+\eps})\ge c_{\mathrm{gap}}$.
	If $X_t^{+\eps}\in\cX_{ -}$, then Assumption~\ref{ass:t23} (A1+A2) gives
	\[
	\rho^{+\eps}_\star-\rho^{+\eps}(X_t^{+\eps})
	=
	\sup_{x \in \cX_+} \rho^{+\eps}(x) - \rho^{+\eps}(X_t^{+\eps})
	\ge
	\sup_{x \in \cX_+} \rho^{+\eps}(x) - \sup_{x \in \cX_-} \rho^{+\eps}(x)
	\ge
	c_{\mathrm{opt}}\eps.
	\]
	Summing over $t$ and taking expectations yields the first inequality. The second is symmetric.
\end{proof}


\begin{lemma}[Chain Rule]\label{lem:chain-rule}
	For each $\eps \in (0,\eps_0]$ and each $t \in [T]$, it holds
	\begin{equation}
		\label{eq:chain-rule}    
		\KL\lrb{\Pb_{(H_t^{+\eps},O_t^{+\eps})}\Vert \Pb_{(H_t^{-\eps},O_t^{-\eps})}}
		=
		\KL\Brb{\Pb_{H_t^{+\eps}}\Vert \Pb_{H_t^{-\eps}}}
		+
		\int_{\cH_t}
		\KL\lrb{\Pb_{F_t^{+\eps}\lrb{\alpha_t(h)}} \Vert \Pb_{F_t^{-\eps}\lrb{\alpha_t(h)}}}
		\dif\Pb_{H_t^{+\eps}}(h)
	\end{equation}
\end{lemma}

\begin{proof}
	Let $\mathfrak{H}_t \coloneqq (\mathfrak{B}\otimes\mathfrak{F})^{\otimes(t-1)}\otimes\mathfrak{B}$ be the canonical $\sigma$-algebra on
	\[
	\cH_t = ([0,1]\times\cF)^{t-1}\times[0,1].
	\]
	For brevity define the two probability measures on $(\cH_t\times\cF,\mathfrak{H}_t\otimes\mathfrak{F})$
	\[
	P^+\coloneqq \Pb_{(H_t^{+\eps},O_t^{+\eps})},\qquad
	P^-\coloneqq \Pb_{(H_t^{-\eps},O_t^{-\eps})},
	\]
	and their history marginals
	\[
	\mu^+\coloneqq \Pb_{H_t^{+\eps}},\qquad
	\mu^-\coloneqq \Pb_{H_t^{-\eps}}.
	\]
	
	\smallskip
	\noindent\emph{Step 1: define the conditional kernels explicitly.}
	For $h\in\cH_t$ and $B\in\mathfrak{F}$ set
	\[
	K^+(h,B)\coloneqq \Pb \bigl(\varphi(\alpha_t(h),Y_t^{+\eps})\in B\bigr),
	\qquad
	K^-(h,B)\coloneqq \Pb \bigl(\varphi(\alpha_t(h),Y_t^{-\eps})\in B\bigr).
	\]
	We claim that $h\mapsto K^\pm(h,B)$ is $\mathfrak{H}_t$-measurable for each fixed $B$.
	Indeed, let $\nu^\pm\coloneqq \Pb_{Y_t^{\pm\eps}}$ be the law of $Y_t^{\pm\eps}$ on $(\cY,\mathfrak{Y})$ (which is well defined because it does not depend on $t$).
	Then
	\[
	K^\pm(h,B)=\int_{\cY}\I \lcb{\varphi(\alpha_t(h),y)\in B}\dif\nu^\pm(y),
	\]
	and the integrand $(h,y)\mapsto \I \lcb{\varphi(\alpha_t(h),y)\in B}$ is $(\mathfrak{H}_t\otimes\mathfrak{Y})$-measurable by measurability of $\alpha_t$ and $\varphi$.
	Thus the integral is $\mathfrak{H}_t$-measurable.
	
	\smallskip
	\noindent\emph{Step 2: $K^\pm$ disintegrate $P^\pm$ over the history.}
	Let $A\in\mathfrak{H}_t$ and $B\in\mathfrak{F}$.
	By definition, $O_t^{\pm\eps}=\varphi(X_t^{\pm\eps},Y_t^{\pm\eps})$ with $X_t^{\pm\eps}=\alpha_t(H_t^{\pm\eps})$.
	Moreover, $Y_t^{\pm\eps}$ is $\Pb$-independent of $\sigma(H_t^{\pm\eps})$ (since it is independent of the past and of the seeds).
	Therefore,
	\begin{align*}
		P^\pm(A\times B)
		&=
		\Pb \bigl(H_t^{\pm\eps}\in A,\ O_t^{\pm\eps}\in B\bigr)
		\\
		&=
		\E\lsb{\I \lcb{H_t^{\pm\eps}\in A}
			\E\lsb{\I \lcb{\varphi(\alpha_t(H_t^{\pm\eps}),Y_t^{\pm\eps})\in B}\mid H_t^{\pm\eps}}}
		\\
		&=
		\E\lsb{\I \lcb{H_t^{\pm\eps}\in A}K^\pm(H_t^{\pm\eps},B)}
		=
		\int_A K^\pm(h,B)\dif\mu^\pm(h).
	\end{align*}
	Thus the explicitly defined kernels $K^\pm$ disintegrate $P^\pm$ over the
    history coordinate:
    \[
        P^\pm(A\times B)=\int_A K^\pm(h,B)\,\dif\mu^\pm(h).
    \]
	
	\smallskip
	\noindent\emph{Step 3: construct a measurable kernel density on a full-measure set (this is where countable generation is used).}
	Assume first that $\mu^+\ll \mu^-$ and that $K^+(h,\cdot)\ll K^-(h,\cdot)$ for $\mu^+$-a.e.\ $h$
	(otherwise at least one term in \eqref{eq:chain-rule} is $+\infty$ and the identity holds in $[0,+\infty]$ by standard absolute-continuity implications; we omit this routine case split).
	
	Let $g\coloneqq \dif\mu^+/\dif\mu^-$ be a Radon-Nikodym derivative.
    For $C\in\mathfrak H_t\otimes\mathfrak F$, write
    $C_h\coloneqq\{o\in\cF:(h,o)\in C\}$, and define
    \[
        \bar P^+(C)\coloneqq \int K^+(h,C_h)\,\dif\mu^+(h),
    \qquad
        \bar P^-(C)\coloneqq \int K^-(h,C_h)\,\dif\mu^+(h).
    \]
    Then $\bar P^+=P^+$. Moreover, $\bar P^-\ll P^-$: indeed, if
    $P^-(C)=0$, then
    \[
        \int K^-(h,C_h)\,\dif\mu^-(h)=0,
    \]
    so $K^-(h,C_h)=0$ for $\mu^-$-a.e. $h$, hence also for
    $\mu^+$-a.e. $h$, since $\mu^+\ll\mu^-$. Therefore
    $\bar P^-(C)=0$.

    Also $\bar P^+\ll\bar P^-$. Indeed, if $\bar P^-(C)=0$, then
    $K^-(h,C_h)=0$ for $\mu^+$-a.e.\ $h$, and the assumption
    $K^+(h,\cdot)\ll K^-(h,\cdot)$ gives $K^+(h,C_h)=0$ for
    $\mu^+$-a.e.\ $h$. Hence $\bar P^+(C)=0$.

    Let
    \[
        \widetilde d\coloneqq \frac{\dif\bar P^+}{\dif\bar P^-}.
    \]
    Then, for all $A\in\mathfrak H_t$ and $B\in\mathfrak F$,
    \[
        \int_A K^+(h,B)\,\dif\mu^+(h)
        =
        \int_A\int_B
        \widetilde d(h,o)K^-(h,\dif o)\,\dif\mu^+(h).
    \]
	Fix a countable generating family $(B_n)_{n\in\N}\subset\mathfrak{F}$ (it exists since $\mathfrak{F}$ is countably generated).
	For each $n$, the previous identity (with $B=B_n$) implies
	\[
	K^+(h,B_n)=\int_{B_n}\wt d(h,o)K^-(h,\dif o)
	\qquad \text{for $\mu^+$-a.e.\ }h,
	\]
	and taking a countable union of the corresponding null sets yields a set $N\in\mathfrak{H}_t$ with $\mu^+(N)=0$ such that the identity holds for all $n$ and all $h\notin N$.
	By a monotone-class argument in $B$ (for each fixed $h\notin N$), it extends to all $B\in\mathfrak{F}$:
	\[
	\forall h\notin N,\ \forall B\in\mathfrak{F},\qquad
	K^+(h,B)=\int_B \wt d(h,o)K^-(h,\dif o).
	\]
	In particular, for $h\notin N$, the map $o\mapsto \wt d(h,o)$ is a Radon-Nikodym derivative of $K^+(h,\cdot)$ w.r.t.\ $K^-(h,\cdot)$.
	
	\smallskip
	\noindent\emph{Step 4: factorize the Radon-Nikodym derivative of $P^+$ w.r.t.\ $P^-$.}
	Define $f(h,o)\coloneqq g(h)\wt d(h,o)$.
	Using the disintegrations from Step 2 and the kernel-density identity from Step 3, for all $A\in\mathfrak{H}_t$, $B\in\mathfrak{F}$,
	\begin{align*}
		\int_{A\times B} f(h,o)\dif P^-(h,o)
		&=
		\int_A g(h)\int_B \wt d(h,o)K^-(h,\dif o)\dif\mu^-(h)
		\\
		&=
		\int_A \int_B \wt d(h,o)K^-(h,\dif o)\dif\mu^+(h)
		=
		\int_A K^+(h,B)\dif\mu^+(h)
		=
		P^+(A\times B).
	\end{align*}
	Hence $f=\dif P^+/\dif P^-$.
	
	\smallskip
	\noindent\emph{Step 5: compute the KL and conclude.}
	By definition of KL divergence,
	\begin{align*}
		\KL(P^+\|P^-)
		&=
		\int_{\cH_t\times\cF} \ln f(h,o)\dif P^+(h,o)
		\\
		&=
		\int_{\cH_t\times\cF} \ln g(h)\dif P^+(h,o)
		+
		\int_{\cH_t\times\cF} \ln \wt d(h,o)\dif P^+(h,o)
		\\
		&=
		\int_{\cH_t} \ln g(h)\dif\mu^+(h)
		+
		\int_{\cH_t}\int_{\cF} \ln \wt d(h,o)K^+(h,\dif o)\dif\mu^+(h)
		\\
		&=
		\KL(\mu^+\|\mu^-)
		+
		\int_{\cH_t} \KL \bigl(K^+(h,\cdot)\|K^-(h,\cdot)\bigr)\dif\mu^+(h).
	\end{align*}
	Finally, by construction $K^\pm(h,\cdot)=\Pb_{F_t^{\pm\eps}(\alpha_t(h))}$, so this is exactly \eqref{eq:chain-rule}.
	
	\smallskip
	\noindent
	The existence of the $\mathfrak{H}_t$-measurable version of the integrand follows from Step 3, since
	$h\mapsto \int_{\cF}\ln\wt d(h,o)K^+(h,\dif o)$ is $\mathfrak{H}_t$-measurable (kernel integration of a measurable function).
\end{proof}

\begin{lemma}[Per-Round KL Upper Bound]\label{lem:t23-per-round-kl}
	Assume Assumption~\ref{ass:t23}. For every $\eps\in(0,\eps_0]$, every $t\in[T]$, and every $x\in\cX$,
	\[
	\KL\lrb{\Pb_{F_t^{+\eps}\lrb{x}} \Vert \Pb_{F_t^{-\eps}\lrb{x}}}
	\le
	c_{\mathrm{info}}\eps^2\I\lcb{x\in\cX_{\mathrm{info}}}.
	\]
\end{lemma}
\begin{proof}
	If $x\notin\cX_{\mathrm{info}}$, then Assumption~\ref{ass:t23} (A4) gives
	$\Pb_{F_t^{+\eps}(x)}=\Pb_{F_t^{-\eps}(x)}$, hence the KL is $0$.
	If $x\in\cX_{\mathrm{info}}$, then Assumption~\ref{ass:t23} (A5) gives
	\[
	\KL\Brb{\Pb_{F_t^{+\eps}(x)}\Vert \Pb_{F_t^{-\eps}(x)}}
	\le
	c_{\mathrm{info}}\eps^2.
	\]
	Combining the two cases yields the claim.
\end{proof}
\begin{lemma}[KL Chain Rule Bound]\label{lem:t23-kl-chain}
	Assume Assumption~\ref{ass:t23}. For every $\eps\in(0,\eps_0]$,
	\begin{align}
		\KL\lrb{\Pb_{(H_T^{+\eps},O_T^{+\eps})}\Vert \Pb_{(H_T^{-\eps},O_T^{-\eps})}}
		\le
		c_{\mathrm{info}}\eps^2\E\lsb{N_{\mathrm{info}}^{+\eps}}.
		\label{eq:t23-kl-chain}
	\end{align}
\end{lemma}
\begin{proof}
	For brevity define, for each $t\in[T]$,
	\[
	Z_t^\eps
	\coloneqq
	(H_t^\eps,O_t^\eps)
	=
	(U_1,O_1^\eps,\dots,U_t,O_t^\eps).
	\]
	
	\smallskip
	\noindent\emph{Step 1: chain rule recursion.}
	For every $t\ge 2$, the chain rule (\Cref{lem:chain-rule}) gives
	\begin{align*}
		\KL\lrb{\Pb_{Z_t^{+\eps}}\Vert \Pb_{Z_t^{-\eps}}}
		&=
		\KL\Brb{\Pb_{H_t^{+\eps}}\Vert \Pb_{H_t^{-\eps}}}
		+
		\int_{\cH_t}
		\KL\lrb{\Pb_{F_t^{+\eps}\lrb{\alpha_t(h)}} \Vert \Pb_{F_t^{-\eps}\lrb{\alpha_t(h)}}}
		\dif\Pb_{H_t^{+\eps}}(h)
		\\
		&=
		\KL\lrb{\Pb_{(Z_{t-1}^{+\eps},U_t)}\Vert \Pb_{(Z_{t-1}^{-\eps},U_t)}}
		+
		\E\lsb{\Bsb{\KL\lrb{\Pb_{F_t^{+\eps}(x)} \Vert \Pb_{F_t^{-\eps}(x)}}}_{x=\alpha_t(H_t^{+\eps})}}.
	\end{align*}
	Since $U_t$ is independent of $\lrb{Z_{t-1}^{+\eps},Z_{t-1}^{-\eps}}$,
	\[
	\KL\lrb{\Pb_{(Z_{t-1}^{+\eps},U_t)}\Vert \Pb_{(Z_{t-1}^{-\eps},U_t)}}
	=
	\KL\lrb{\Pb_{Z_{t-1}^{+\eps}}\Vert \Pb_{Z_{t-1}^{-\eps}}}.
	\]
	Therefore, for every $t\ge 2$,
	\[
	\KL\lrb{\Pb_{Z_t^{+\eps}}\Vert \Pb_{Z_t^{-\eps}}}
	=
	\KL\lrb{\Pb_{Z_{t-1}^{+\eps}}\Vert \Pb_{Z_{t-1}^{-\eps}}}
	+
	\E\lsb{\Bsb{\KL\lrb{\Pb_{F_t^{+\eps}(x)} \Vert \Pb_{F_t^{-\eps}(x)}}}_{x=\alpha_t(H_t^{+\eps})}}.
	\]
	Also, for $t=1$ the chain rule (\Cref{lem:chain-rule}) gives
	\begin{align*}
		\KL\lrb{\Pb_{Z_1^{+\eps}}\Vert \Pb_{Z_1^{-\eps}}}
		&=
		\KL\Brb{\Pb_{H_1^{+\eps}}\Vert \Pb_{H_1^{-\eps}}}
		+
		\int_{\cH_1}
		\KL\lrb{\Pb_{F_1^{+\eps}\lrb{\alpha_1(h)}} \Vert \Pb_{F_1^{-\eps}\lrb{\alpha_1(h)}}}
		\dif\Pb_{H_1^{+\eps}}(h)
		\\
		&=
		\KL\Brb{\Pb_{U_1}\Vert \Pb_{U_1}}
		+
		\int_{\cH_1}
		\KL\lrb{\Pb_{F_1^{+\eps}\lrb{\alpha_1(h)}} \Vert \Pb_{F_1^{-\eps}\lrb{\alpha_1(h)}}}
		\dif\Pb_{H_1^{+\eps}}(h)
		\\
		&=
		\E\lsb{\Bsb{\KL\lrb{\Pb_{F_1^{+\eps}(x)} \Vert \Pb_{F_1^{-\eps}(x)}}}_{x=\alpha_1(H_1^{+\eps})}}.
	\end{align*}

	\smallskip
	\noindent\emph{Step 2: plug the per-round bound.}
	For each $t \in [T]$, by Lemma~\ref{lem:t23-per-round-kl}, with $x = \alpha_t(h)$ where $h$ is the random history $H_t^{+\eps}$, we get
	\[
	\lsb{\KL\lrb{\Pb_{F_t^{+\eps}(x)} \Vert \Pb_{F_t^{-\eps}(x)}}}_{x=\alpha_t(H_t^{+\eps})}
	\le
	c_{\mathrm{info}}\eps^2\I\lcb{\alpha_t(H_t^{+\eps})\in\cX_{\mathrm{info}}}
	=
	c_{\mathrm{info}}\eps^2\I\lcb{X_t^{+\eps}\in\cX_{\mathrm{info}}}.
	\]
	Taking expectations and plugging into the recursion yields, for every $t\ge 2$,
	\[
	\KL\lrb{\Pb_{Z_t^{+\eps}}\Vert \Pb_{Z_t^{-\eps}}}
	\le
	\KL\lrb{\Pb_{Z_{t-1}^{+\eps}}\Vert \Pb_{Z_{t-1}^{-\eps}}}
	+
	c_{\mathrm{info}}\eps^2\E\bsb{\I\lcb{X_t^{+\eps}\in\cX_{\mathrm{info}}}}.
	\]
	Iterating from $t=T$ down to $t=2$ and using the base case $t=1$ gives
	\begin{align*}
		\KL\lrb{\Pb_{Z_T^{+\eps}}\Vert \Pb_{Z_T^{-\eps}}}
		&\le
		\KL\lrb{\Pb_{Z_1^{+\eps}}\Vert \Pb_{Z_1^{-\eps}}}
		+
		c_{\mathrm{info}}\eps^2\sum_{t=2}^T \E\bsb{\I\lcb{X_t^{+\eps}\in\cX_{\mathrm{info}}}}
		\\
		&=
		c_{\mathrm{info}}\eps^2\sum_{t=1}^T \E\bsb{\I\lcb{X_t^{+\eps}\in\cX_{\mathrm{info}}}}
		=
		c_{\mathrm{info}}\eps^2\E\lsb{N_{\mathrm{info}}^{+\eps}}.
	\end{align*}
	Finally, $Z_T^\eps=(H_T^\eps,O_T^\eps)$, so this is exactly \eqref{eq:t23-kl-chain}.
\end{proof}


\begin{lemma}[Pinsker and Bounded Test Functions]\label{lem:t23-tv}
	Let $g\colon\cH_T\times\cF\to\bbR$ be measurable and satisfy $0\le g \le T$.
	Then for every $\eps\in(0,\eps_0]$,
	\begin{align*}
		\Babs{\E\bsb{g(H_T^{+\eps},O_T^{+\eps})}-\E\bsb{g(H_T^{-\eps},O_T^{-\eps})}}
		&\le
		T\sqrt{\frac12\KL\lrb{\Pb_{(H_T^{+\eps},O_T^{+\eps})}\Vert \Pb_{(H_T^{-\eps},O_T^{-\eps})}}}.
	\end{align*}
\end{lemma}

\begin{proof}
	Let
	\[
	P
	\coloneqq
	\Pb_{(H_T^{+\eps},O_T^{+\eps})},
	\qquad
	Q
	\coloneqq
	\Pb_{(H_T^{-\eps},O_T^{-\eps})}.
	\]
	
	We have
	\begin{align*}
		\labs{\int g\dif P-\int g\dif Q}
		&=
		\labs{\int_0^T P\lrb{g \ge r} \dif r- \int_0^T Q\lrb{g \ge r} \dif r}
		\le
		\int_0^T \labs{P\lrb{g \ge r} - Q\lrb{g \ge r}}   \dif r
		\\
		&\le
		\int_0^T \sup_A\labs{P\lrb{A} - Q\lrb{A}}   \dif r
		=
		\int_0^T \lno{P - Q}_{\mathrm{TV}}   \dif r
		=
		T \lno{P - Q}_{\mathrm{TV}},
	\end{align*}
	and the conclusion follows by applying the Pinsker inequality (see, e.g., \cite{Tsybakov2008}).
\end{proof}


\begin{theorem}[$T^{2/3}$ Lower Bound]\label{thm:t23-template}
	Assume Assumption~\ref{ass:t23}.
	Define the threshold
	\[
	T_0
	\coloneqq
	\max\lrb{
		\frac{c_{\mathrm{gap}}}{4c_{\mathrm{info}}c_{\mathrm{opt}}\eps_0^3},
		\frac{2c_{\mathrm{opt}}^{2}}{c_{\mathrm{info}}c_{\mathrm{gap}}^{2}}}.
	\]
	Then for every horizon $T\ge T_0$ and every (randomized) algorithm,
	if we set
	\[
	\eps
	\coloneqq
	\lrb{\frac{c_{\mathrm{gap}}}{4c_{\mathrm{info}}c_{\mathrm{opt}}T}}^{1/3},
	\]
	we have $\eps\le \eps_0$ and
	\[
	\max\lrb{R^{+\eps}(\alpha,T),R^{-\eps}(\alpha,T)}
	\ge
	\frac{c_{\mathrm{gap}}^{1/3}c_{\mathrm{opt}}^{2/3}}{4^{4/3}c_{\mathrm{info}}^{1/3}}T^{2/3}.
	\]
	In particular, there exists $\sigma\in\{+1,-1\}$ such that
	\[
	R^{\sigma\eps}(\alpha,T)
	\ge
	\frac{c_{\mathrm{gap}}^{1/3}c_{\mathrm{opt}}^{2/3}}{4^{4/3}c_{\mathrm{info}}^{1/3}}T^{2/3}.
	\]
\end{theorem}
\begin{proof}
	Fix $T\ge T_0$ and set $\eps \coloneqq \big(\frac{c_{\mathrm{gap}}}{4c_{\mathrm{info}}c_{\mathrm{opt}}T}\big)^{1/3}$.
	By $T\ge \frac{c_{\mathrm{gap}}}{4c_{\mathrm{info}}c_{\mathrm{opt}}\eps_0^3}$ we have $\eps\le \eps_0$.
	Let
	\[
	M
	\coloneqq
	\E\lsb{N_{\mathrm{info}}^{+\eps}}+\E\lsb{N_{\mathrm{info}}^{-\eps}}.
	\]
	
	\smallskip
	\noindent\emph{Case 1 (many info actions).}
	If $M\ge \frac{1}{8c_{\mathrm{info}}\eps^2}$, then by Lemma~\ref{lem:t23-regret-counts},
	\[
	\frac{R^{+\eps}(\alpha,T)+R^{-\eps}(\alpha,T)}{2}
	\ge
	\frac{c_{\mathrm{gap}}}{2}M
	\ge
	\frac{c_{\mathrm{gap}}}{16c_{\mathrm{info}}\eps^2}.
	\]
	
	\smallskip
	\noindent\emph{Case 2 (few info actions).}
	Assume $M<\frac{1}{8c_{\mathrm{info}}\eps^2}$. Then Lemma~\ref{lem:t23-kl-chain} yields
	\[
	\KL\lrb{\Pb_{(H_T^{+\eps},O_T^{+\eps})}\Vert \Pb_{(H_T^{-\eps},O_T^{-\eps})}}
	\le
	c_{\mathrm{info}}\eps^2\E\lsb{N_{\mathrm{info}}^{+\eps}}
	\le
	c_{\mathrm{info}}\eps^2 M
	\le
	\frac18.
	\]
    Define the deterministic count functional $g_{+}\colon \cH_T\times\cF\to\bbR$ by
    \[
        g_{+}\lrb{u_1,o_1,\dots,u_T,o_T}
    \coloneqq
        \sum_{t=1}^T
        \I\lcb{
        \alpha_t\lrb{u_1,o_1,\dots,u_{t-1},o_{t-1},u_t}
        \in \cX_{+}
        }.
    \]
    Then $g_{+}$ is measurable, $0\le g_{+}\le T$, and, for each $\sigma\in\lcb{+1,-1}$,
    \[
        g_{+}\lrb{H_T^{\sigma\eps},O_T^{\sigma\eps}}
    =
        N_{+}^{\sigma\eps}.
    \]
    Applying Lemma~\ref{lem:t23-tv} with $g=g_{+}$ gives
	\[
	\babs{\E\lsb{N_{ +}^{+\eps}}-\E\lsb{N_{ +}^{-\eps}}}
	\le
	\frac{T}{4}.
	\]
	Using $N_{ -}^{+\eps}=T-N_{\mathrm{info}}^{+\eps}-N_{ +}^{+\eps}$, we get
	\begin{align*}
		\E\lsb{N_{ -}^{+\eps}}+\E\lsb{N_{ +}^{-\eps}}
		&=
		T-\E\lsb{N_{\mathrm{info}}^{+\eps}}-\E\lsb{N_{ +}^{+\eps}}+\E\lsb{N_{ +}^{-\eps}}
		\\
		&\ge
		T-M-\babs{\E\lsb{N_{ +}^{+\eps}}-\E\lsb{N_{ +}^{-\eps}}}
		\ge
		\frac{3T}{4}-M.
	\end{align*}
	Moreover, since $T\ge \frac{2c_{\mathrm{opt}}^{2}}{c_{\mathrm{info}}c_{\mathrm{gap}}^{2}}$ and
	$\eps^2=\lrb{\frac{c_{\mathrm{gap}}}{4c_{\mathrm{info}}c_{\mathrm{opt}}T}}^{2/3}$, we have
	\[
	\frac{1}{8c_{\mathrm{info}}\eps^2}
	\le
	\frac{T}{4},
	\]
	hence in Case~2 we get $M<\frac{T}{4}$ and therefore
	\[
	\E\lsb{N_{ -}^{+\eps}}+\E\lsb{N_{ +}^{-\eps}}
	\ge
	\frac{3T}{4}-M
	\ge
	\frac{T}{2}.
	\]
	By Lemma~\ref{lem:t23-regret-counts},
	\[
	\frac{R^{+\eps}(\alpha,T)+R^{-\eps}(\alpha,T)}{2}
	\ge
	\frac{c_{\mathrm{opt}}\eps}{2}\lrb{\E\lsb{N_{ -}^{+\eps}}+\E\lsb{N_{ +}^{-\eps}}}
	\ge
	\frac{c_{\mathrm{opt}}}{4}T\eps.
	\]
	
	\smallskip
	\noindent\emph{Combine.}
	In all cases,
	\[
	\max\lrb{R^{+\eps}(\alpha,T),R^{-\eps}(\alpha,T)}
	\ge
	\frac{R^{+\eps}(\alpha,T)+R^{-\eps}(\alpha,T)}{2}
	\ge
	\min\lrb{\frac{c_{\mathrm{gap}}}{16c_{\mathrm{info}}\eps^2},\frac{c_{\mathrm{opt}}}{4}T\eps}.
	\]
	With the chosen $\eps$, the two terms match, and
	\[
	\min\lrb{\frac{c_{\mathrm{gap}}}{16c_{\mathrm{info}}\eps^2},\frac{c_{\mathrm{opt}}}{4}T\eps}
	=
	\frac{c_{\mathrm{gap}}^{1/3}c_{\mathrm{opt}}^{2/3}}{4^{4/3}c_{\mathrm{info}}^{1/3}}T^{2/3}.
	\]
	This proves the stated bound, and choosing $\sigma$ as the sign attaining the maximum yields the final claim.
\end{proof}


\section[The T to the 2/3 Regret Lower Bound for Fair Bilateral Trade]
{The \texorpdfstring{$T^{2/3}$}{T to the 2/3} Regret Lower Bound for Fair Bilateral Trade}
\label{app:regret_lower_bound}

Fix throughout this section a fairness level $\lambda\in[-\infty,0]$.
All constants in the construction below are independent of $\lambda$.

\noindent\emph{Per-round model, i.i.d.}
Fix the constants
\[
\alpha \ceq \frac{1}{20},
\qquad
\eps_0 \ceq \frac{1}{50}.
\]
For each $\eps\in[-\eps_0,\eps_0]$, let $S^\eps$ be a discrete random variable on
$\lcb{0,\tfrac{1}{10},\tfrac{1}{5},\tfrac{3}{5}}$ with
\[
\Pb\lsb{S^\eps=0}=\alpha,
\qquad
\Pb\lsb{S^\eps=\tfrac{1}{10}}=\frac{\alpha}{2}-\eps,
\qquad
\Pb\lsb{S^\eps=\tfrac{1}{5}}=\frac{\alpha}{2}+\eps,
\qquad
\Pb\lsb{S^\eps=\tfrac{3}{5}}=1-2\alpha.
\]
Let $B$ be independent of $S^\eps$ with support
$\lcb{\tfrac{2}{5},\tfrac{4}{5},\tfrac{9}{10},1}$ and
\[
\Pb\lsb{B=\tfrac{2}{5}}=1-2\alpha=\frac{9}{10},
\qquad
\Pb\lsb{B=\tfrac{4}{5}}=\Pb\lsb{B=\tfrac{9}{10}}=\frac{\alpha}{2}=\frac{1}{40},
\qquad
\Pb\lsb{B=1}=\alpha=\frac{1}{20}.
\]
For the horizon $T$, let $\lcb{(S_t^\eps,B_t)}_{t\in[T]}$ be i.i.d.\ with
$S_t^\eps\sim S^\eps$, $B_t\sim B$, and $S_t^\eps\perp B_t$ for each $t$.

For an action $p\in[0,1]$, define the per-round reward and feedback by
\[
G_t^\eps(p)\ceq \gft_\lambda(p,S_t^\eps,B_t),
\qquad
F_t^\eps(p)\ceq \lrb{\I\lcb{S_t^\eps\le p},\I\lcb{p\le B_t}},
\]
and
\[
\gft_\lambda(p,s,b)
\ceq
\I\lcb{s\le p\le b}
M_\lambda\lrb{(p-s)_+,(b-p)_+}.
\]
Let
\[
\rho^\eps(p)\ceq \E\lsb{G_t^\eps(p)}.
\]
Since the process is i.i.d., this quantity does not depend on $t$.

\noindent\emph{Action partition.}
Set
\[
\cX \ceq [0,1],
\qquad
\cX_{\mathrm{info}}\ceq [1/10,1/5),
\qquad
\cX_- \ceq [0,1/2)\setminus\cX_{\mathrm{info}},
\qquad
\cX_+ \ceq [1/2,1].
\]
Then
\[
\cX=\cX_-\cup\cX_+\cup\cX_{\mathrm{info}}
\]
is a disjoint union.
The interval is half-open because for $p=1/5$ both perturbed seller atoms,
at $1/10$ and $1/5$, are accepted. Hence their total mass is
$(\alpha/2-\varepsilon)+(\alpha/2+\varepsilon)=\alpha$, which is independent
of the sign of $\varepsilon$.


\begin{lemma}[(A4) Outside $\cX_{\mathrm{info}}$ the Feedback Is Sign-Blind]
\label{lem:A4}
For every $\eps\in(0,\eps_0]$ and every $p\in\cX\setminus\cX_{\mathrm{info}}$,
\[
\Pb_{F_t^{+\eps}(p)}=\Pb_{F_t^{-\eps}(p)}
\qquad
\forall t\in[T].
\]
\end{lemma}

\begin{proof}
The second bit $\I\{p\le B_t\}$ does not depend on $\eps$.
For the first bit $Y^\eps(p)\ceq \I\{S_t^\eps\le p\}$, the only
$\eps$-dependent atoms of $S^\eps$ are at $1/10$ and $1/5$.
Hence the map $p\mapsto\Pb[S^\eps\le p]$ depends on $\eps$ only when
$p\in[1/10,1/5)$, namely only on $\cX_{\mathrm{info}}$.
Since $S_t^\eps$ and $B_t$ are independent, the two feedback bits are
independent for each fixed $p$. Therefore the full joint law of
$F_t^\eps(p)$ is sign-independent outside $\cX_{\mathrm{info}}$.
\end{proof}


\begin{lemma}[(A5) KL Is $O(\eps^2)$ on $\cX_{\mathrm{info}}$ with Explicit Constant]
\label{lem:A5}
For every $\eps\in(0,\eps_0]$, every $t\in[T]$, and every $p\in\cX_{\mathrm{info}}=[1/10,1/5)$,
\[
\KL\Brb{\Pb_{F_t^{+\eps}(p)}\Vert \Pb_{F_t^{-\eps}(p)}}
\le c_{\mathrm{info}}\eps^2,
\qquad
c_{\mathrm{info}}\ceq \frac{160000}{1991}<81.
\]
\end{lemma}

\begin{proof}
Fix $p\in[1/10,1/5)$.
Since $p\le1/5<2/5\le B_t$ almost surely, the second feedback bit equals $1$ deterministically and contributes nothing to the KL divergence.
Hence
\[
\KL\Brb{\Pb_{F_t^{+\eps}(p)}\Vert \Pb_{F_t^{-\eps}(p)}}
=
\KL\Brb{\Pb_{\I\lcb{S_t^{+\eps}\le p}}\Vert \Pb_{\I\lcb{S_t^{-\eps}\le p}}}.
\]
On $[1/10,1/5)$,
\[
q_+ \ceq \Pb\lsb{S_t^{+\eps}\le p}
=
\alpha+\lrb{\frac{\alpha}{2}-\eps}
=
\frac{3\alpha}{2}-\eps,
\]
and
\[
q_- \ceq \Pb\lsb{S_t^{-\eps}\le p}
=
\alpha+\lrb{\frac{\alpha}{2}+\eps}
=
\frac{3\alpha}{2}+\eps.
\]
Thus the first bit is $\mathrm{Ber}(q_\pm)$ and, using $\KL(P\Vert Q)\le \chi^2(P\Vert Q)$ (see, e.g., \citealp[Lemma~2.7]{Tsybakov2008}) and the formula for Bernoulli $\chi^2$ divergences,
\begin{align*}
    \KL\Brb{\Pb_{F_t^{+\eps}(p)}\Vert \Pb_{F_t^{-\eps}(p)}}
&=
    \KL\Brb{\mathrm{Ber}(q_+)\Vert\mathrm{Ber}(q_-)}
\\
&\le
    \chi^2\Brb{\mathrm{Ber}(q_+)\Vert\mathrm{Ber}(q_-)}
=
    \frac{(q_+-q_-)^2}{q_-(1-q_-)}
=
    \frac{4\eps^2}{q_-(1-q_-)}.
\end{align*}
For $|\eps|\le\eps_0$ and $\alpha=1/20$,
\[
q_-\in
\lsb{\frac{3}{40}-\frac{1}{50},\frac{3}{40}+\frac{1}{50}}
=
\lsb{\frac{11}{200},\frac{19}{200}}.
\]
Hence
\[
q_-(1-q_-)\ge \frac{11}{200}\cdot\frac{181}{200}=\frac{1991}{40000},
\]
and therefore
\[
\KL\Brb{\mathrm{Ber}(q_+)\Vert\mathrm{Ber}(q_-)}
\le
\frac{160000}{1991}\eps^2.
\]
\end{proof}


For a buyer atom $b$ and two prices $p,q\in[0,1]$, define the normalized
contribution of $B=b$ to $\rho^\eps(p)-\rho^\eps(q)$ by
\[
    C_b^{\eps,\lambda}(p,q)
    \coloneqq
    \E\bsb{
        \gft_\lambda(p,S^\eps,b)
        -
        \gft_\lambda(q,S^\eps,b)
    }.
\]
Thus, by independence of $S^\eps$ and $B$, the actual contribution of the
atom $B=b$ to $\rho^\eps(p)-\rho^\eps(q)$ is
\[
    \Pb\lsb{B=b} C_b^{\eps,\lambda}(p,q).
\]
In particular, for the buyer support $\mathcal B=\{2/5,4/5,9/10,1\}$,
\[
    \rho^\eps(p)-\rho^\eps(q)
    =
    \sum_{b\in\mathcal B}
    \Pb[B=b]C_b^{\eps,\lambda}(p,q).
\]

\begin{lemma}[Nonnegative Contribution of the Dominant Buyer Atom]
\label{lem:dominant-buyer-nonnegative}
For every $\eps\in[-\eps_0,\eps_0]$,
\[
    C_{2/5}^{\eps,\lambda}\lrb{\frac14,\frac15}\ge0.
\]
\end{lemma}

\begin{proof}
Set
\[
    A_\lambda\ceq M_\lambda\lrb{\frac14,\frac3{20}},
    \qquad
    B_\lambda\ceq M_\lambda\lrb{\frac1{10},\frac15},
    \qquad
    C_\lambda\ceq M_\lambda\lrb{\frac1{20},\frac3{20}}.
\]
Expanding $C_{2/5}^{\eps,\lambda}(1/4,1/5)$ over the seller atoms gives
\[
    C_{2/5}^{\eps,\lambda}\lrb{\frac14,\frac15}
    =
    \frac1{20}\lrb{A_\lambda-\frac15}
    +
    \lrb{\frac1{40}-\eps}\lrb{\frac3{20}-B_\lambda}
    +
    \lrb{\frac1{40}+\eps}C_\lambda.
\]
This expression is affine in $\eps$.
Since $\eps_0=1/50$, it is enough to check the two endpoints
$\eps=\pm1/50$.

By homogeneity, with
\[
    m_\lambda(r)\ceq M_\lambda(1,r),
\]
we have
\[
    A_\lambda=\frac3{20}m_\lambda\lrb{\frac53},
    \qquad
    B_\lambda=\frac1{10}m_\lambda(2),
    \qquad
    C_\lambda=\frac1{20}m_\lambda(3).
\]

At $\eps=1/50$, the inequality
$C_{2/5}^{\eps,\lambda}(1/4,1/5)\ge0$ is equivalent to
\[
    10A_\lambda-B_\lambda+9C_\lambda\ge \frac{37}{20}.
\]
After substitution, the left-hand side minus $37/20$ is
\[
    \frac32\lrb{m_\lambda\lrb{\frac53}-1}
    -
    \frac1{10}\lrb{m_\lambda(2)-1}
    +
    \frac9{20}\lrb{m_\lambda(3)-1}.
\]
Since $m_\lambda$ is nondecreasing on $[1,+\infty)$ by
\Cref{lem:normalized-mean-concavity},
\[
    m_\lambda(3)-1\ge m_\lambda(2)-1.
\]
Therefore the last display is at least
\[
    \frac32\lrb{m_\lambda\lrb{\frac53}-1}
    +
    \frac7{20}\lrb{m_\lambda(2)-1}
    \ge0.
\]

At $\eps=-1/50$, the required nonnegativity is equivalent to
\[
    10A_\lambda-9B_\lambda+C_\lambda\ge \frac{13}{20}.
\]
After substitution, the left-hand side minus $13/20$ is
\[
    \frac32\lrb{m_\lambda\lrb{\frac53}-1}
    -
    \frac9{10}\lrb{m_\lambda(2)-1}
    +
    \frac1{20}\lrb{m_\lambda(3)-1}.
\]
By \Cref{lem:normalized-mean-concavity}, the function $m_\lambda$ is concave
on $[1,+\infty)$. Since $m_\lambda(1)=1$, concavity gives
\[
    \frac{m_\lambda(5/3)-1}{5/3-1}
    \ge
    \frac{m_\lambda(2)-1}{2-1}.
\]
Hence
$
    m_\lambda(2)-1
    \le
    \frac32\lrb{m_\lambda\lrb{\frac53}-1}
$ and consequently,
\[
\begin{aligned}
&
    \frac32\lrb{m_\lambda\lrb{\frac53}-1}
    -
    \frac9{10}\lrb{m_\lambda(2)-1}
    +
    \frac1{20}\lrb{m_\lambda(3)-1}
\\
&\qquad
    \ge
    \frac3{20}\lrb{m_\lambda\lrb{\frac53}-1}
    +
    \frac1{20}\lrb{m_\lambda(3)-1}
    \ge0,
\end{aligned}
\]
where the last inequality uses again that $m_\lambda$ is nondecreasing and
$m_\lambda(1)=1$.

Thus $C_{2/5}^{\eps,\lambda}(1/4,1/5)\ge0$ for both endpoint values of
$\eps$, and hence for every $\eps\in[-\eps_0,\eps_0]$.
\end{proof}

\begin{lemma}[(A3) Info Actions Are Uniformly Suboptimal with Explicit Gap]
\label{lem:A3}
For every $\eps\in[-\eps_0,\eps_0]$ and every $p\in\cX_{\mathrm{info}}=[1/10,1/5)$,
\[
\rho^\eps_\star-\rho^\eps(p)\ge c_{\mathrm{gap}},
\qquad
c_{\mathrm{gap}}\ceq \frac{1}{40000}.
\]
\end{lemma}

\begin{proof}
Fix $\eps\in[-\eps_0,\eps_0]$ and $p\in[1/10,1/5)$.
We first show that
$
    \rho^\eps(p)\le \rho^\eps\lrb{\frac15}
$.
Indeed, if a seller-buyer atom can yield nonzero reward at price
$p\in[1/10,1/5)$, then necessarily
$
    S^\eps\in\lcb{0,\frac1{10}}
$, 
$
    B\in\lcb{\frac25,\frac45,\frac9{10},1}
$.
For every such pair $(s,b)$, the interval $[1/10,1/5]$ is contained in
$
    \lsb{s,\frac{s+b}{2}}$.
By \Cref{lem:single-interval-monotonicity},
$
    \gft_\lambda(p,s,b)
    \le
    \gft_\lambda\lrb{\frac15,s,b}
$.
Taking expectations gives $\rho^\eps(p)\le\rho^\eps(1/5)$.
It remains to lower bound
$
    \rho^\eps\lrb{\frac14}-\rho^\eps\lrb{\frac15}.
$
By \Cref{lem:dominant-buyer-nonnegative},
$
    C_{2/5}^{\eps,\lambda}\lrb{\frac14,\frac15}\ge0
$.
Since the actual contribution of the buyer atom $B=2/5$ is
$
    \Pb\lsb{B=2/5}
    C_{2/5}^{\eps,\lambda}\lrb{\frac14,\frac15}
$,
the $B=2/5$ part of
$
    \rho^\eps\lrb{\frac14}-\rho^\eps\lrb{\frac15}
$
is nonnegative.
Now consider the buyer-tail event
$
    B\in\lcb{\frac45,\frac9{10},1}
$.
For $S^\eps\in\{0,1/10\}$ and $B$ in this tail, the whole interval
$[1/5,1/4]$ lies on the nondecreasing side of the corresponding
single-trade reward curve.
Thus the total contribution of these atoms to
$
    \rho^\eps\lrb{\frac14}-\rho^\eps\lrb{\frac15}
$
is nonnegative by \Cref{lem:single-interval-monotonicity}.
For $S^\eps=1/5$ and $B$ in this tail, the reward at price $1/5$ is zero because the seller gain is zero.
At price $1/4$, by \Cref{lem:M-bounds},
$
    \gft_\lambda\lrb{\frac14,\frac15,B}
    =
    M_\lambda\lrb{\frac1{20},B-\frac14}
    \ge
    \min\lrb{\frac1{20},B-\frac14}
    =
    \frac1{20}
$.
Therefore
\[
\begin{aligned}
    \rho^\eps\lrb{\frac14}-\rho^\eps\lrb{\frac15}
    &\ge
    \Pb\lsb{S^\eps=\frac15}
    \Pb\lsb{B\in\lcb{\frac45,\frac9{10},1}}
    \cdot \frac1{20}
    \\
    &\ge
    \lrb{\frac1{40}-\frac1{50}}
    \cdot
    \frac1{10}
    \cdot
    \frac1{20}
    =
    \frac1{40000}.
\end{aligned}
\]
Since $\rho^\eps_\star\ge\rho^\eps(1/4)$ and $\rho^\eps(p)\le\rho^\eps(1/5)$, the claim follows.
\end{proof}


\begin{lemma}[Exact Symmetry at $\eps=0$]
\label{lem:sym0}
Let $\eps=0$.
Then the pair $(S^0,B)$ has the same law as $(1-B,1-S^0)$.
Moreover, for every $p\in[0,1]$,
\[
\E\lsb{\gft_\lambda(p,S^0,B)}
=
\E\lsb{\gft_\lambda(1-p,S^0,B)}.
\]
Equivalently,
\[
\rho^0(p)=\rho^0(1-p)
\qquad
\forall p\in[0,1].
\]
\end{lemma}

\begin{proof}
By inspection of the atoms and probabilities,
\begin{alignat*}{2}
\Pb\lsb{B=2/5}  &= \frac{9}{10}  = \Pb\lsb{S^0=3/5},
&\qquad
\Pb\lsb{B=4/5}  &= \frac{1}{40} = \Pb\lsb{S^0=1/5},
\\
\Pb\lsb{B=9/10} &= \frac{1}{40} = \Pb\lsb{S^0=1/10},
&\qquad
\Pb\lsb{B=1}    &= \frac{1}{20} = \Pb\lsb{S^0=0}.
\end{alignat*}
Thus $1-B$ has the same law as $S^0$, and $1-S^0$ has the same law as $B$.

Fix $(s,b)$ and set
$
    p'\ceq 1-p
$, 
$
    s'\ceq 1-b
$, 
$
    b'\ceq 1-s
$.
Then
$
\I\lcb{s\le p\le b}=\I\lcb{s'\le p'\le b'}
$
and
$
(b'-p',p'-s')=(p-s,b-p)
$.
Since $M_\lambda$ is symmetric in its two coordinates, we have
$
\gft_\lambda(p,s,b)=\gft_\lambda(p',s',b')
$, and, taking expectations and using
$
(S^0,B)\stackrel{d}{=}(1-B,1-S^0)
$,
yields the result.
\end{proof}


\begin{lemma}[Left-Side Localization]
\label{lem:left-localization}
For every $\eps\in[-\eps_0,\eps_0]$,
\[
\max_{p\in\cX_-}\rho^\eps(p)
=
\max_{p\in[1/5,2/5]}\rho^\eps(p).
\]
\end{lemma}

\begin{proof}
We first rule out $p\in[0,1/10)$.
Fix such a $p$.
Since $p<1/10$, only the seller atom $S^\eps=0$ can accept.
Hence
\[
\rho^\eps(p)
=
\alpha\E_B\lsb{M_\lambda(p,B-p)}.
\]
Using \Cref{lem:M-bounds}, for every buyer atom $b$,
\[
M_\lambda(p,b-p)\le \sqrt{p(b-p)}.
\]
Moreover, for each $b\in\lcb{2/5,4/5,9/10,1}$, the map
$
p\mapsto \sqrt{p(b-p)}
$
is increasing on $[0,1/10]$, because $1/10<b/2$.
Therefore
\[
\begin{aligned}
\rho^\eps(p)
&\le
\frac1{20}
\Biggl(
\frac9{10}\sqrt{\frac1{10}\cdot\frac3{10}}
+\frac1{40}\sqrt{\frac1{10}\cdot\frac7{10}}
+\frac1{40}\sqrt{\frac1{10}\cdot\frac8{10}}
+\frac1{20}\sqrt{\frac1{10}\cdot\frac9{10}}
\Biggr)
\\
&<
\frac1{20}
\Biggl(
\frac9{10}\cdot\frac15
+\frac1{40}\cdot\frac3{10}
+\frac1{40}\cdot\frac3{10}
+\frac1{20}\cdot\frac3{10}
\Biggr)
=
\frac{21}{2000}.
\end{aligned}
\]
On the other hand, using $M_\lambda(x,y)\ge\min(x,y)$,
\[
\begin{aligned}
\rho^\eps\lrb{\frac3{10}}
&\ge
\frac9{10}\cdot\frac1{10}\cdot\frac1{10}
+
\frac1{10}
\Biggl(
\frac1{20}\cdot\frac3{10}
+
\lrb{\frac1{40}-\eps}\frac2{10}
+
\lrb{\frac1{40}+\eps}\frac1{10}
\Biggr)
\\
&=
\frac9{800}-\frac{\eps}{100}
\ge
\frac9{800}-\frac1{5000}
=
\frac{221}{20000}.
\end{aligned}
\]
Since
$
\frac{221}{20000}>\frac{21}{2000}$,
no maximizer over $\cX_-$ can lie in $[0,1/10)$.

Next, consider $p\in(2/5,1/2)$.
Then the dominant buyer atom $B=2/5$ never accepts.
Only the buyer tail $\lcb{4/5,9/10,1}$ of total mass $1/10$ can contribute.
Also, only the low seller atoms $\lcb{0,1/10,1/5}$ of total mass $1/10$ can contribute.
Using \Cref{lem:M-bounds},
\[
\gft_\lambda(p,s,b)\le \sqrt{(p-s)(b-p)}.
\]
Since $p-s\le1/2$ and $b-p\le3/5$ in this region,
\[
\gft_\lambda(p,s,b)\le \sqrt{\frac12\cdot\frac35}<\frac35.
\]
Therefore
\[
\rho^\eps(p)\le \frac1{10}\cdot\frac1{10}\cdot\frac35=\frac3{500}.
\]
Since
$
\frac{221}{20000}>\frac3{500}$,
no maximizer over $\cX_-$ can lie in $(2/5,1/2)$.

Finally,
\[
\cX_-=[0,1/10)\cup[1/5,2/5]\cup(2/5,1/2).
\]
Thus the maximum over $\cX_-$ is attained on $[1/5,2/5]$.
\end{proof}

\begin{lemma}[Right-Side Localization]
\label{lem:right-localization}
For every $\eps\in[-\eps_0,\eps_0]$,
\[
\max_{p\in\cX_+}\rho^\eps(p)
=
\max_{p\in[3/5,4/5]}\rho^\eps(p).
\]
\end{lemma}

\begin{proof}
We first rule out $p\in[1/2,3/5)$.
For such prices, the seller atom $S^\eps=3/5$ does not accept and the buyer atom $B=2/5$ does not accept.
Thus trade can occur only when
\[
    S^\eps\in\lcb{0,1/10,1/5},
    \qquad
    B\in\lcb{4/5,9/10,1}.
\]
Both events have probability $1/10$.
Moreover, by \Cref{lem:M-bounds},
\[
    M_\lambda(p-s,b-p)
    \le
    \sqrt{(p-s)(b-p)}
    \le
    \frac{b-s}{2}
    \le
    \frac12.
\]
Therefore
\[
    \rho^\eps(p)\le \frac1{10}\cdot\frac1{10}\cdot\frac12=\frac1{200}.
\]
At $p=3/4$, using only the events $S^\eps=3/5$ and $B\in\lcb{9/10,1}$,
\[
    \rho^\eps\lrb{\frac34}
    \ge
    \Pb\lsb{S^\eps=3/5}\Pb\bsb{B\in\lcb{9/10,1}}
    \cdot\frac3{20}
    =
    \frac9{10}\cdot\frac3{40}\cdot\frac3{20}
    =
    \frac{81}{8000}.
\]
Since $81/8000>1/200$, no maximizer over $\cX_+$ lies in $[1/2,3/5)$.

We now rule out $p\in(4/5,1]$.
Fix such a $p$.
We prove pointwise that, for every seller-buyer atom $(s,b)$,
\[
    \gft_\lambda\lrb{\frac45,s,b}
    \ge
    \gft_\lambda(p,s,b).
\]
If trade does not occur at price $p$, then
$\gft_\lambda(p,s,b)=0$, while
$\gft_\lambda(4/5,s,b)\ge0$, and the claim is immediate.

Assume therefore that trade occurs at price $p$.
Since $p>4/5$, this forces $b\in\{9/10,1\}$.
Moreover, by construction of the seller support, $s\le3/5$.
Hence
\[
    \frac{s+b}{2}\le\frac45.
\]
Also $s\le4/5\le p\le b$, so both $4/5$ and $p$ lie in the trade
interval $[s,b]$, and both are on the nonincreasing side of the
single-trade reward curve. By \Cref{lem:single-interval-monotonicity},
\[
    \gft_\lambda\lrb{\frac45,s,b}
    \ge
    \gft_\lambda(p,s,b).
\]
Taking expectations gives
\[
    \rho^\eps\lrb{\frac45}\ge \rho^\eps(p).
\]
Since $4/5\in[3/5,4/5]$, points in $(4/5,1]$ are never needed to attain the maximum over $\cX_+$.
\end{proof}

\begin{lemma}[Reflected Left-Right Separation]
\label{lem:core-reflected-separation}
There exists a constant
\[
c_{\mathrm{core}}
\ceq
\frac{4-\sqrt2-\sqrt3}{1600}>0
\]
such that for every $\eps\in(0,\eps_0]$ and every $p\in[1/5,2/5]$,
\[
\rho^{+\eps}(1-p)-\rho^{+\eps}(p)\ge c_{\mathrm{core}}\eps,
\]
and
\[
\rho^{-\eps}(p)-\rho^{-\eps}(1-p)\ge c_{\mathrm{core}}\eps.
\]
\end{lemma}

\begin{proof}
Fix $p\in[1/5,2/5]$ and set $q\ceq1-p\in[3/5,4/5]$.
By construction, the only $\eps$-dependence of $S^\eps$ is the split of mass between $1/10$ and $1/5$.
Using \Cref{lem:sym0},
\[
    \rho^0(q)=\rho^0(p).
\]
Hence
\[
    \rho^{+\eps}(q)-\rho^{+\eps}(p)=\eps D_\lambda(p),
\]
where
\[
D_\lambda(p)
\ceq
\E_B\lsb{
    \gft_\lambda(q,1/5,B)
    -
    \gft_\lambda(q,1/10,B)
    -
    \gft_\lambda(p,1/5,B)
    +
    \gft_\lambda(p,1/10,B)
}.
\]
Similarly,
\[
    \rho^{-\eps}(p)-\rho^{-\eps}(q)=\eps D_\lambda(p).
\]
It remains to prove that
\[
    D_\lambda(p)\ge c_{\mathrm{core}},
    \qquad
    \forall p\in[1/5,2/5].
\]

For each buyer atom
\[
    b\in\mathcal B\ceq\lcb{\frac25,\frac45,\frac9{10},1},
\]
define its buyer-atom component in $D_\lambda(p)$ by
\[
D_{\lambda,b}(p)
\ceq
    \gft_\lambda(q,1/5,b)
    -
    \gft_\lambda(q,1/10,b)
    -
    \gft_\lambda(p,1/5,b)
    +
    \gft_\lambda(p,1/10,b).
\]
Thus
\[
    D_\lambda(p)
    =
    \sum_{b\in\mathcal B}
    \Pb\lsb{B=b}D_{\lambda,b}(p).
\]

The component corresponding to the buyer atom $B=2/5$ is nonnegative.
Indeed, since $q>2/5$, the reflected price $q$ never trades against $B=2/5$, and
\[
    D_{\lambda,2/5}(p)
    =
    M_\lambda\lrb{p-\frac1{10},\frac25-p}
    -
    M_\lambda\lrb{p-\frac15,\frac25-p}
    \ge0
\]
by monotonicity of $M_\lambda$ in its first coordinate.

We shall use two elementary one-sided increment bounds.
If $\lambda\in(-\infty,0]$ and $x,y>0$, \Cref{lem:holder_differential_identities} gives
\[
    \partial_x M_\lambda(x,y)
    =
    \frac12\lrb{\frac{M_\lambda(x,y)}{x}}^{1-\lambda}.
\]
If $0<x\le y$, then $M_\lambda(x,y)\ge x$ by \Cref{lem:holder_basic_bounds}, and therefore
\[
    \partial_x M_\lambda(x,y)\ge\frac12.
\]
If $0<y\le x$, then $M_\lambda(x,y)\le\sqrt{xy}$ by \Cref{lem:holder_basic_bounds}, and since $1-\lambda\ge1$ and $M_\lambda(x,y)/x\le1$,
\[
    \partial_x M_\lambda(x,y)
    \le
    \frac12\frac{M_\lambda(x,y)}{x}
    \le
    \frac12\sqrt{\frac{y}{x}}.
\]
The same increment bounds when $\lambda=-\infty$ follow directly from
$M_{-\infty}(x,y)=\min\lcb{x,y}$.
Consequently, for every $h>0$ and every $x,y\ge0$ such that $0\le x\le x+h\le y$,
\begin{equation}
\label{eq:core_lower_increment}
    M_\lambda(x+h,y)-M_\lambda(x,y)
    \ge
    \frac h2.
\end{equation}
Indeed, for $x>0$ this follows by integrating the lower derivative bound, while for $x=0$ it follows from
$M_\lambda(h,y)\ge M_{-\infty}(h,y)=h$.

Moreover, for every $h>0$ and every $x,y\ge0$ such that $0\le y\le x$ and $x>0$,
\begin{equation}
\label{eq:core_upper_increment}
    M_\lambda(x+h,y)-M_\lambda(x,y)
    \le
    \frac h2\sqrt{\frac{y}{x}}.
\end{equation}
Indeed, for $y>0$ this follows by integrating the upper derivative bound, while for $y=0$ both sides are zero.

The component corresponding to the buyer atom $B=4/5$ satisfies
\[
    D_{\lambda,4/5}(p)
    \ge
    \frac1{20}-\frac1{20\sqrt2}.
\]
Indeed,
\[
\begin{aligned}
D_{\lambda,4/5}(p)
&=
\lsb{
    M_\lambda\lrb{p-\frac1{10},\frac45-p}
    -
    M_\lambda\lrb{p-\frac15,\frac45-p}
}
\\
&\quad+
\lsb{
    M_\lambda\lrb{\frac45-p,p-\frac15}
    -
    M_\lambda\lrb{\frac9{10}-p,p-\frac15}
}.
\end{aligned}
\]
For the first bracket, set
\[
    x\ceq p-\frac15,
    \qquad
    h\ceq\frac1{10},
    \qquad
    y\ceq\frac45-p.
\]
Since $0\le x\le x+h\le y$ for every $p\in[1/5,2/5]$, \eqref{eq:core_lower_increment} gives
\[
    M_\lambda\lrb{p-\frac1{10},\frac45-p}
    -
    M_\lambda\lrb{p-\frac15,\frac45-p}
    \ge
    \frac1{20}.
\]
For the second bracket, set
\[
    x\ceq\frac45-p,
    \qquad
    h\ceq\frac1{10},
    \qquad
    y\ceq p-\frac15.
\]
Since $0\le y\le x$ and
\[
    \frac{y}{x}
    =
    \frac{p-1/5}{4/5-p}
    \le
    \frac12
    \qquad
    \text{for every }p\in[1/5,2/5],
\]
\eqref{eq:core_upper_increment} gives
\[
    M_\lambda\lrb{\frac9{10}-p,p-\frac15}
    -
    M_\lambda\lrb{\frac45-p,p-\frac15}
    \le
    \frac1{20\sqrt2}.
\]
Therefore
\[
    D_{\lambda,4/5}(p)
    \ge
    \frac1{20}-\frac1{20\sqrt2}.
\]

Similarly, the component corresponding to the buyer atom $B=9/10$ satisfies
\[
    D_{\lambda,9/10}(p)
    \ge
    \frac1{20}-\frac{\sqrt3}{40}.
\]
Indeed,
\[
\begin{aligned}
D_{\lambda,9/10}(p)
&=
\lsb{
    M_\lambda\lrb{p-\frac1{10},\frac9{10}-p}
    -
    M_\lambda\lrb{p-\frac15,\frac9{10}-p}
}
\\
&\quad+
\lsb{
    M_\lambda\lrb{\frac45-p,p-\frac1{10}}
    -
    M_\lambda\lrb{\frac9{10}-p,p-\frac1{10}}
}.
\end{aligned}
\]
The first bracket is at least $1/20$ by \eqref{eq:core_lower_increment}.
For the second bracket, set
\[
    x\ceq\frac45-p,
    \qquad
    h\ceq\frac1{10},
    \qquad
    y\ceq p-\frac1{10}.
\]
Since $0\le y\le x$ and
\[
    \frac{y}{x}
    =
    \frac{p-1/10}{4/5-p}
    \le
    \frac34
    \qquad
    \text{for every }p\in[1/5,2/5],
\]
\eqref{eq:core_upper_increment} gives
\[
    M_\lambda\lrb{\frac9{10}-p,p-\frac1{10}}
    -
    M_\lambda\lrb{\frac45-p,p-\frac1{10}}
    \le
    \frac1{20}\sqrt{\frac34}
    =
    \frac{\sqrt3}{40}.
\]
Therefore
\[
    D_{\lambda,9/10}(p)
    \ge
    \frac1{20}-\frac{\sqrt3}{40}.
\]

Finally, the component corresponding to the buyer atom $B=1$ is nonnegative.
Indeed,
\[
\begin{aligned}
D_{\lambda,1}(p)
&=
\lsb{
    M_\lambda\lrb{p-\frac1{10},1-p}
    -
    M_\lambda\lrb{p-\frac15,1-p}
}
\\
&\quad+
\lsb{
    M_\lambda\lrb{\frac45-p,p}
    -
    M_\lambda\lrb{\frac9{10}-p,p}
}.
\end{aligned}
\]
The first bracket is at least $1/20$ by \eqref{eq:core_lower_increment}.
For the second bracket, set
\[
    x\ceq\frac45-p,
    \qquad
    h\ceq\frac1{10},
    \qquad
    y\ceq p.
\]
Since $0\le y\le x$ and
\[
    \frac{y}{x}
    =
    \frac{p}{4/5-p}
    \le
    1
    \qquad
    \text{for every }p\in[1/5,2/5],
\]
\eqref{eq:core_upper_increment} gives
\[
    M_\lambda\lrb{\frac9{10}-p,p}
    -
    M_\lambda\lrb{\frac45-p,p}
    \le
    \frac1{20}.
\]
Therefore
\[
    D_{\lambda,1}(p)\ge0.
\]

Combining the normalized contributions with the buyer-atom probabilities,
\[
\begin{aligned}
D_\lambda(p)
&=
\sum_{b\in\mathcal B}
\Pb\lsb{B=b}D_{\lambda,b}(p)
\ge
\frac1{40}\lrb{\frac1{20}-\frac1{20\sqrt2}}
+
\frac1{40}\lrb{\frac1{20}-\frac{\sqrt3}{40}}
\\
&=
\frac{4-\sqrt2-\sqrt3}{1600}
=
c_{\mathrm{core}}.
\end{aligned}
\]
This proves the claim.
\end{proof}


\begin{lemma}[(A2) Linear Separation between Sides with Explicit $c_{\mathrm{opt}}$]
\label{lem:A2}
For every $\eps\in(0,\eps_0]$,
\[
\max_{p\in\cX_+}\rho^{+\eps}(p)-\max_{p\in\cX_-}\rho^{+\eps}(p)
\ge c_{\mathrm{opt}}\eps,
\]
and
\[
\max_{p\in\cX_-}\rho^{-\eps}(p)-\max_{p\in\cX_+}\rho^{-\eps}(p)
\ge c_{\mathrm{opt}}\eps,
\]
where
\[
c_{\mathrm{opt}}\ceq \frac{4-\sqrt2-\sqrt3}{1600}.
\]
\end{lemma}

\begin{proof}
Let
\[
p^\star\in\argmax_{p\in\cX_-}\rho^{+\eps}(p).
\]
By \Cref{lem:left-localization},
$
p^\star\in[1/5,2/5]$.
Hence
\[
1-p^\star\in[3/5,4/5]\subset\cX_+.
\]
Therefore
\[
\max_{p\in\cX_+}\rho^{+\eps}(p)
\ge
\rho^{+\eps}(1-p^\star).
\]
By \Cref{lem:core-reflected-separation},
\[
\rho^{+\eps}(1-p^\star)-\rho^{+\eps}(p^\star)
\ge
c_{\mathrm{opt}}\eps.
\]
Thus
\[
\max_{p\in\cX_+}\rho^{+\eps}(p)
\ge
\max_{p\in\cX_-}\rho^{+\eps}(p)+c_{\mathrm{opt}}\eps.
\]

For the second inequality, let
\[
q^\star\in\argmax_{q\in\cX_+}\rho^{-\eps}(q).
\]
By \Cref{lem:right-localization},
\[
q^\star\in[3/5,4/5].
\]
Set
$
p^\star\ceq1-q^\star$.
Then
\[
p^\star\in[1/5,2/5]\subset\cX_-.
\]
By \Cref{lem:core-reflected-separation},
\[
\rho^{-\eps}(p^\star)-\rho^{-\eps}(q^\star)
=
\rho^{-\eps}(p^\star)-\rho^{-\eps}(1-p^\star)
\ge
c_{\mathrm{opt}}\eps.
\]
Therefore
\[
\max_{p\in\cX_-}\rho^{-\eps}(p)
\ge
\rho^{-\eps}(p^\star)
\ge
\rho^{-\eps}(q^\star)+c_{\mathrm{opt}}\eps
=
\max_{q\in\cX_+}\rho^{-\eps}(q)+c_{\mathrm{opt}}\eps.
\]
\end{proof}

\begin{lemma}[(A1) Optimal Side]
\label{lem:A1}
For every $\eps\in(0,\eps_0]$, every maximizer of $\rho^{+\eps}$ lies in $\cX_+$ and every maximizer of $\rho^{-\eps}$ lies in $\cX_-$.
\end{lemma}

\begin{proof}
Fix $\eps\in(0,\eps_0]$.
By \Cref{lem:A2},
$
\max_{\cX_+}\rho^{+\eps}>\max_{\cX_-}\rho^{+\eps}$.
Moreover, by \Cref{lem:A3}, every $p\in\cX_{\mathrm{info}}$ is suboptimal.
Hence no maximizer can lie in $\cX_-\cup\cX_{\mathrm{info}}$, so every maximizer lies in $\cX_+$.
The case $-\eps$ is analogous.
\end{proof}

\noindent\emph{Conclusion: the lower bound.}

\regretlowerbound*

\begin{proof}
Fix $\lambda\in[-\infty,0]$.
Consider the two i.i.d.\ bilateral-trade instances constructed at the beginning of \Cref{app:regret_lower_bound}, with perturbation parameters $+\eps$ and $-\eps$.
By \Cref{lem:A1,lem:A2,lem:A3,lem:A4,lem:A5}, these two instances satisfy Assumption~\ref{ass:t23} with
\[
    \eps_0=\frac{1}{50},
    \qquad
    c_{\mathrm{gap}}=\frac{1}{40000},
    \qquad
    c_{\mathrm{opt}}=\frac{4-\sqrt2-\sqrt3}{1600},
    \qquad
    c_{\mathrm{info}}=\frac{160000}{1991}.
\]
Therefore, by \Cref{thm:t23-template}, for every horizon
\[
    T
    \ge
    \max\lrb{
        \frac{c_{\mathrm{gap}}}{4c_{\mathrm{info}}c_{\mathrm{opt}}\eps_0^3},
        \frac{2c_{\mathrm{opt}}^{2}}{c_{\mathrm{info}}c_{\mathrm{gap}}^{2}}
    },
\]
and every randomized pricing policy $\alpha$, there exists a sign $\sigma\in\lcb{+1,-1}$ such that, under the instance with perturbation $\sigma\eps$, where
\[
    \eps
    \ceq
    \lrb{
        \frac{c_{\mathrm{gap}}}
        {4c_{\mathrm{info}}c_{\mathrm{opt}}T}
    }^{1/3},
\]
the regret satisfies
\[
    R_\lambda(\alpha,T)
    =
    R^{\sigma\eps}(\alpha,T)
    \ge
    \frac{
        c_{\mathrm{gap}}^{1/3}
        c_{\mathrm{opt}}^{2/3}
    }{
        4^{4/3}
        c_{\mathrm{info}}^{1/3}
    }
    T^{2/3}.
\]
Both constructed instances are independent i.i.d.\ bilateral-trade instances with valuations in $[0,1]$.
Setting
\[
    \Delta\ceq 4-\sqrt2-\sqrt3,
\]
we have
\[
    c_{\mathrm{opt}}=\frac{\Delta}{1600}.
\]
Hence the lower-bound constant given by \Cref{thm:t23-template} is
\[
    c_\star
    \ceq
    \frac{
        c_{\mathrm{gap}}^{1/3}
        c_{\mathrm{opt}}^{2/3}
    }{
        4^{4/3}
        c_{\mathrm{info}}^{1/3}
    }
    =
    \lrb{
        \frac{1991\Delta^2}{2^{22}\cdot 10^{12}}
    }^{1/3}
    \ge
    10^{-6}.
\]
Moreover, the horizon threshold required by \Cref{thm:t23-template} is
\[
    \left\lceil
    \max\lrb{
        \frac{c_{\mathrm{gap}}}{4c_{\mathrm{info}}c_{\mathrm{opt}}\eps_0^3},
        \frac{2c_{\mathrm{opt}}^{2}}{c_{\mathrm{info}}c_{\mathrm{gap}}^{2}}
    }
    \right\rceil
    =
    \left\lceil
    \max\lrb{
        \frac{1991}{128\Delta},
        \frac{1991\Delta^2}{128}
    }
    \right\rceil
    =
    19.
\]
Therefore, for every $T\ge 19$,
\[
    R_\lambda(\alpha,T)
    \ge
    c_\star T^{2/3}
    \ge
    10^{-6}T^{2/3}.
\]
This concludes the proof.
\end{proof}


\section{The PAC Lower Bound}
\label{app:proof_PAC_lower}

\noindent\emph{Pure exploration, stopping, and recommendation.}
A pure-exploration algorithm is given by a sequential querying policy
$\alpha=(\alpha_t)_{t\ge 1}$, a stopping rule, and a recommendation rule.
For every $t\ge 1$, the map
\[
\alpha_t\colon ([0,1]\times\cF)^{t-1}\times [0,1]\to \cX
\]
is measurable. Under environment $\theta\in[-\eps_0,\eps_0]$, the learner chooses
recursively
\[
X_t^\theta
=
\alpha_t(U_1,O_1^\theta,\dots,U_{t-1},O_{t-1}^\theta,U_t),
\qquad
O_t^\theta=F_t^\theta(X_t^\theta).
\]
For $t\ge1$, write
\[
H_t^\theta
\ceq
(U_1,O_1^\theta,\dots,U_{t-1},O_{t-1}^\theta,U_t),
\qquad
H_1^\theta\ceq U_1.
\]
Let
\[
\mathcal G_t^\theta
\ceq
\sigma(U_1,O_1^\theta,\dots,U_t,O_t^\theta),
\qquad t\ge 1,
\]
and let $\tau^\theta$ be the stopping time induced by the stopping rule when the
environment is $\theta$. We assume throughout this section that $\tau^\theta$ is
a.s. finite under the environments under consideration; if its expectation is
infinite, the lower bound below is immediate.

For every $t\ge 1$, let
\[
\widehat{\alpha}_t
\colon
\brb{[0,1]\times\cF}^t\times[0,1]\to\cX
\]
be the recommendation map used when the algorithm stops after $t$ observations.
The final recommendation under environment $\theta$ is
\[
\widehat X^\theta
\ceq
\widehat{\alpha}_{\tau^\theta}
(U_1,O_1^\theta,\dots,U_{\tau^\theta},O_{\tau^\theta}^\theta,U_{\tau^\theta+1}).
\]
Equivalently, define the stopped observation history
\[
\mathcal T_\tau^\theta
\ceq
(\tau^\theta,
 U_1,O_1^\theta,\dots,
 U_{\tau^\theta},O_{\tau^\theta}^\theta,
 U_{\tau^\theta+1}),
\]
which takes values in the disjoint union
\[
\mathsf T
\ceq
\bigsqcup_{t\ge 1}
\lcb{t}\times([0,1]\times\cF)^t\times[0,1].
\]
Then $\widehat X^\theta$ is a measurable function of $\mathcal T_\tau^\theta$.

\noindent\emph{PAC guarantee.}
Let $\varepsilon>0$ and $\delta\in(0,1)$.
We say that the pure-exploration algorithm is $(\varepsilon,\delta)$-PAC over a
class of environments if, for every environment $\theta$ in the class,
\[
 \Pb_\theta\lsb{
    \rho_\star^\theta-\rho^\theta(\widehat X^\theta)\le \varepsilon
}
\ge
1-\delta.
\]

\noindent\emph{Finite initial segments.}
When the stopping time is random, we use the following convention.
We say that Assumption~\ref{ass:t23} holds on every finite initial segment if,
for every deterministic $m\ge1$, the finite-horizon problem obtained by
keeping only rounds $1,\dots,m$ satisfies Assumption~\ref{ass:t23} with the
same constants.

\begin{lemma}[KL Bound at the Stopping Time]\label{lem:pac-kl}
    Assume Assumption~\ref{ass:t23} holds on every finite initial segment.
    For every $\eta\in(0,\eps_0]$,
    \[
    \KL\lrb{
        \Pb_{\mathcal T_\tau^{+\eta}}
        \Big\Vert
        \Pb_{\mathcal T_\tau^{-\eta}}
    }
    \le
    c_{\mathrm{info}}\eta^2\,\E_{+\eta}\lsb{\tau^{+\eta}}.
    \]
\end{lemma}

\begin{proof}
    If $\E_{+\eta}\lsb{\tau^{+\eta}}=+\infty$, there is nothing to prove.
    We therefore assume that this expectation is finite.

    Fix $m\ge 1$ and set
    \[
        \tau_m^\theta \ceq \tau^\theta\wedge m.
    \]
    Let $\bot$ be a symbol not belonging to $\cF$ and set
    \[
        \cF_\bot \ceq \cF\cup\{\bot\},
        \qquad
        \mathfrak F_\bot
        \ceq
        \sigma\bigl(\mathfrak F\cup\{\{\bot\}\}\bigr),
    \]
    where $\mathfrak F_\bot$ is viewed as a sigma-field on $\cF_\bot$.
    Define the censored observations
    \[
        \widetilde O_{t,m}^\theta
        \ceq
        \begin{cases}
            O_t^\theta, & t\le \tau_m^\theta,\\
            \bot, & t> \tau_m^\theta.
        \end{cases}
    \]
    Consider the finite censored observation history
    \[
        \overline{\mathcal T}_m^\theta
        \ceq
        (U_1,\widetilde O_{1,m}^\theta,\dots,
         U_m,\widetilde O_{m,m}^\theta,U_{m+1})
        \in
        ([0,1]\times\cF_\bot)^m\times[0,1].
    \]
    We first bound the KL divergence between the laws of
    $\overline{\mathcal T}_m^{+\eta}$ and
    $\overline{\mathcal T}_m^{-\eta}$.

    The usual chain-rule argument gives
    \begin{align*}
    \KL\lrb{
        \Pb_{\overline{\mathcal T}_m^{+\eta}}
        \Big\Vert
        \Pb_{\overline{\mathcal T}_m^{-\eta}}
    }
    \le
    \sum_{t=1}^{m}
    \E_{+\eta}\bbsb{
        \I\lcb{t\le \tau^{+\eta}}\,
        \KL\lrb{
            \Pb_{F_t^{+\eta}(X_t^{+\eta})\mid H_t^{+\eta}}
            \Big\Vert
            \Pb_{F_t^{-\eta}(X_t^{+\eta})\mid H_t^{+\eta}}
        }
    }.
    \end{align*}
    Indeed, conditionally on the censored history before round $t$, if the
    algorithm has already stopped, then the $t$-th censored observation is the
    deterministic symbol $\bot$ under both environments and contributes zero KL.
    Otherwise, the previous censored history contains no symbol $\bot$, the
    algorithm queries the action selected by the same measurable rule
    $\alpha_t$, and the conditional KL contribution is that of the feedback
    distribution at the queried action.

    By Assumption~\ref{ass:t23} (A4) and (A5), for every action $x\in\cX$,
    \[
        \KL\lrb{
            \Pb_{F_t^{+\eta}(x)}
            \Big\Vert
            \Pb_{F_t^{-\eta}(x)}
        }
        \le
        c_{\mathrm{info}}\eta^2\,\I\lcb{x\in\cX_{\mathrm{info}}}
        \le
        c_{\mathrm{info}}\eta^2.
    \]
    Hence
    \[
    \KL\lrb{
        \Pb_{\overline{\mathcal T}_m^{+\eta}}
        \Big\Vert
        \Pb_{\overline{\mathcal T}_m^{-\eta}}
    }
    \le
    c_{\mathrm{info}}\eta^2
    \sum_{t=1}^m
    \Pb_{+\eta}\lrb{t\le \tau^{+\eta}}
    =
    c_{\mathrm{info}}\eta^2
    \E_{+\eta}\lsb{\tau^{+\eta}\wedge m}.
    \]

    Let $\star$ be a symbol not belonging to $\mathsf T$. Now let
    $\mathcal T_{\tau}^{\theta,(m)}$ denote the level-$m$ truncation of
    $\mathcal T_\tau^\theta$: it coincides with $\mathcal T_\tau^\theta$ on
    $\{\tau^\theta\le m\}$, and on $\{\tau^\theta>m\}$ it is equal to
    \[
        (\star,U_1,O_1^\theta,\dots,U_m,O_m^\theta,U_{m+1}).
    \]
    This is a measurable function of $\overline{\mathcal T}_m^\theta$ because
    the stopping decision up to time $m$ is determined by the observed history
    up to time $m$. Hence, by data processing,
    \[
    \KL\lrb{
        \Pb_{\mathcal T_{\tau}^{+\eta,(m)}}
        \Big\Vert
        \Pb_{\mathcal T_{\tau}^{-\eta,(m)}}
    }
    \le
    c_{\mathrm{info}}\eta^2
    \E_{+\eta}\lsb{\tau^{+\eta}\wedge m}.
    \]
    The sigma-fields generated by
    $\mathcal T_{\tau}^{\theta,(m)}$, $m\ge1$, increase to the sigma-field
    generated by the full stopped history $\mathcal T_\tau^\theta$.
    By the martingale convergence theorem for relative entropy on increasing
    sigma-fields (see, e.g., \citet[Eq.~(3)]{barron1991information}),
\[
D\!\left(P|_{\mathcal G_\infty}\middle\Vert Q|_{\mathcal G_\infty}\right)
=
\lim_{m\to\infty}
D\!\left(P|_{\mathcal G_m}\middle\Vert Q|_{\mathcal G_m}\right),
\]
where $\mathcal G_\infty=\sigma(\cup_m\mathcal G_m)$.
    Hence,
    \[
    \KL\lrb{
        \Pb_{\mathcal T_\tau^{+\eta}}
        \Big\Vert
        \Pb_{\mathcal T_\tau^{-\eta}}
    }
    =
    \lim_{m\to\infty}
    \KL\lrb{
        \Pb_{\mathcal T_{\tau}^{+\eta,(m)}}
        \Big\Vert
        \Pb_{\mathcal T_{\tau}^{-\eta,(m)}}
    }.
    \]
    Therefore,
    \[
    \KL\lrb{
        \Pb_{\mathcal T_\tau^{+\eta}}
        \Big\Vert
        \Pb_{\mathcal T_\tau^{-\eta}}
    }
    \le
    c_{\mathrm{info}}\eta^2
    \lim_{m\to\infty}
    \E_{+\eta}\lsb{\tau^{+\eta}\wedge m}.
    \]
    By monotone convergence,
    \[
        \lim_{m\to\infty}\E_{+\eta}\lsb{\tau^{+\eta}\wedge m}
        =
        \E_{+\eta}\lsb{\tau^{+\eta}}.
    \]
    This proves the claim. The extra seed used for the final recommendation has
    the same conditional law under the two environments, hence it does not
    create any additional KL contribution.
\end{proof}

We will need the following classical lemma, whose statement and proof can be
found, e.g., in \citet[Lemma 2.6]{Tsybakov2008}.

\begin{lemma}[Bretagnolle-Huber Inequality]\label{lem:BH}
    Let $P,Q$ be probability measures on the same measurable space, and let
    $\psi$ be any measurable test with values in $\{0,1\}$.
    Then
    \[
    P\lsb{\psi=0}+Q\lsb{\psi=1}
    \ge
    \frac{1}{2}\exp\brb{-\KL(P\Vert Q)}.
    \]
\end{lemma}

\begin{theorem}[Stopping-Time PAC Lower Bound]\label{thm:pac-template}
    Assume Assumption~\ref{ass:t23} holds on every finite initial segment.
    Let $\varepsilon,\delta$ satisfy
    \[
    0<\varepsilon
    \le
    \min\lrb{\frac{c_{\mathrm{gap}}}{2},
              \frac{c_{\mathrm{opt}}\eps_0}{2}}
    \qquad\text{and}\qquad
    \delta\in\lrb{0,\frac14}.
    \]
    Set
    \[
    \eta
    \coloneqq
    \frac{2\varepsilon}{c_{\mathrm{opt}}}.
    \]
    Then $\eta\in(0,\eps_0]$, and any pure-exploration algorithm that stops at
    an a.s. finite stopping time and is $(\varepsilon,\delta)$-PAC simultaneously
    under the two environments $+\eta$ and $-\eta$ must satisfy
    \[
    \E_{+\eta}\lsb{\tau^{+\eta}}
    \ge
    \frac{c_{\mathrm{opt}}^2}{4c_{\mathrm{info}}\varepsilon^2}
    \ln\frac{1}{4\delta}.
    \]
    Consequently,
    \[
    \max_{\sigma\in\{+,-\}}
    \E_{\sigma\eta}\lsb{\tau^{\sigma\eta}}
    \ge
    \frac{c_{\mathrm{opt}}^2}{4c_{\mathrm{info}}\varepsilon^2}
    \ln\frac{1}{4\delta}.
    \]
\end{theorem}

\begin{proof}
    Fix $\varepsilon,\delta$ as in the statement, and set
    $\eta=2\varepsilon/c_{\mathrm{opt}}$.

    We first show that any $\varepsilon$-optimal recommendation must reveal the
    sign of the environment.

    Under the environment $+\eta$, Assumption~\ref{ass:t23} (A1) and (A2) imply
    that for every $x\in\cX_-$,
    \[
    \rho_\star^{+\eta}-\rho^{+\eta}(x)
    \ge
    c_{\mathrm{opt}}\eta
    =
    2\varepsilon.
    \]
    Also, Assumption~\ref{ass:t23} (A3) implies that for every
    $x\in\cX_{\mathrm{info}}$,
    \[
    \rho_\star^{+\eta}-\rho^{+\eta}(x)
    \ge
    c_{\mathrm{gap}}
    \ge
    2\varepsilon.
    \]
    Hence every $\varepsilon$-optimal action under $+\eta$ must belong to
    $\cX_+$. Therefore, since the algorithm is $(\varepsilon,\delta)$-PAC,
    \[
     \Pb_{+\eta}\lsb{\widehat X^{+\eta}\in\cX_+}\ge 1-\delta.
    \]
    By the symmetric argument,
    \[
     \Pb_{-\eta}\lsb{\widehat X^{-\eta}\in\cX_-}\ge 1-\delta.
    \]

    Define the test
    \[
    \psi(z)
    \coloneqq
    \I\lcb{\widehat X(z)\in\cX_+}.
    \]
    Then
    \[
    \Pb_{+\eta}\lsb{\psi(\mathcal T_\tau^{+\eta})=0}\le \delta,
    \qquad
    \Pb_{-\eta}\lsb{\psi(\mathcal T_\tau^{-\eta})=1}\le \delta.
    \]
    Applying Lemma~\ref{lem:BH} together with Lemma~\ref{lem:pac-kl}, we obtain
    \begin{align*}
        2\delta
        &\ge
        \Pb_{+\eta}\lsb{\psi(\mathcal T_\tau^{+\eta})=0}
        +
        \Pb_{-\eta}\lsb{\psi(\mathcal T_\tau^{-\eta})=1}
        \\
        &\ge
        \frac12
        \exp \Big(
        -\KL\lrb{
            \Pb_{\mathcal T_\tau^{+\eta}}
            \Big\Vert
            \Pb_{\mathcal T_\tau^{-\eta}}
        }
        \Big)
        \\
        &\ge
        \frac12
        \exp \bigl(
            -c_{\mathrm{info}}\eta^2
            \E_{+\eta}\lsb{\tau^{+\eta}}
        \bigr).
    \end{align*}
    Taking logarithms gives
    \[
        c_{\mathrm{info}}\eta^2
        \E_{+\eta}\lsb{\tau^{+\eta}}
        \ge
        \ln \frac{1}{4\delta}.
    \]
    Since $\eta=2\varepsilon/c_{\mathrm{opt}}$, this yields
    \[
    \E_{+\eta}\lsb{\tau^{+\eta}}
    \ge
    \frac{c_{\mathrm{opt}}^2}{4c_{\mathrm{info}}\varepsilon^2}
    \ln \frac{1}{4\delta},
    \]
    as claimed.
\end{proof}


\paclowerbound*

\begin{proof}
Fix $\lambda\in[-\infty,0]$.
Let a pure-exploration algorithm stop at an a.s. finite stopping time $\tau$ and
be $(\varepsilon,\delta)$-PAC over the full class of independent i.i.d.\
bilateral-trade instances under $2$-bit threshold feedback.

Set
\[
    \eta
    \coloneqq
    \frac{2\varepsilon}{c_{\mathrm{opt}}},
\]
where, for the hard instances below,
\[
    c_{\mathrm{opt}}
    =
    \frac{4-\sqrt2-\sqrt3}{16\cdot 10^2}.
\]
Consider the two bilateral-trade instances constructed in
\Cref{app:regret_lower_bound}, with perturbation parameters $+\eta$ and $-\eta$.
These instances have independent seller and buyer valuations within each round,
are i.i.d.\ over time, and generate exactly the $2$-bit threshold feedback of
Interaction~Protocol~\ref{algo:interaction_protocol}.

As shown in \Cref{app:regret_lower_bound}, the same verification applies after
restricting the interaction to any deterministic finite horizon. Hence this pair
of instances satisfies Assumption~\ref{ass:t23} on every finite initial segment,
with
\[
    c_{\mathrm{gap}}
    =
    \frac{1}{4\cdot 10^4},
    \qquad
    c_{\mathrm{opt}}
    =
    \frac{4-\sqrt2-\sqrt3}{16\cdot 10^2},
    \qquad
    c_{\mathrm{info}}
    =
    \frac{16\cdot 10^4}{1991},
    \qquad
    \eps_0
    =
    \frac{1}{50}.
\]
Since
\[
    \frac{c_{\mathrm{opt}}\eps_0}{2}
    =
    \frac{4-\sqrt2-\sqrt3}{16\cdot 10^4}
    \le
    \frac{1}{8\cdot 10^4}
    =
    \frac{c_{\mathrm{gap}}}{2},
\]
the condition
\[
    \varepsilon
    \le
    \frac{4-\sqrt2-\sqrt3}{16\cdot 10^4}
\]
implies
\[
    0<\varepsilon
    \le
    \min\lrb{
        \frac{c_{\mathrm{gap}}}{2},
        \frac{c_{\mathrm{opt}}\eps_0}{2}
    }.
\]
Therefore, \Cref{thm:pac-template} applies. Since the algorithm is
$(\varepsilon,\delta)$-PAC over the full class, it is in particular
$(\varepsilon,\delta)$-PAC simultaneously on the two constructed environments.
Thus,
\[
    \E_{+\eta}\lsb{\tau^{+\eta}}
    \ge
    \frac{c_{\mathrm{opt}}^2}{4c_{\mathrm{info}}\varepsilon^2}
    \ln\frac{1}{4\delta}.
\]
In particular, there exists an independent i.i.d.\ bilateral-trade instance
$\nu$ among the two constructed instances such that
\[
    \E_{\nu}\lsb{\tau}
    \ge
    \frac{c_{\mathrm{opt}}^2}{4c_{\mathrm{info}}\varepsilon^2}
    \ln\frac{1}{4\delta}.
\]
Equivalently, the worst-case expected length satisfies
\[
    \sup_{\nu}\E_\nu\lsb{\tau}
    \ge
    \frac{c_{\mathrm{opt}}^2}{4c_{\mathrm{info}}\varepsilon^2}
    \ln\frac{1}{4\delta}.
\]

Moreover,
\[
    \frac{c_{\mathrm{opt}}^2}{4c_{\mathrm{info}}}
    =
    \frac{
        \lrb{\frac{4-\sqrt2-\sqrt3}{16\cdot 10^2}}^2
    }{
        4\cdot\frac{16\cdot 10^4}{1991}
    }
    =
    \frac{1991\lrb{4-\sqrt2-\sqrt3}^2}
    {16384 \cdot 10^8}.
\]
This proves the claimed lower bound with
\[
    c_{\mathrm{PAC}}
    \ceq
    \frac{1991\lrb{4-\sqrt2-\sqrt3}^2}
    {16384 \cdot 10^8}.
\]
\end{proof}


\section[Linear Lower Bound when Removing the Assumption of Independence between St and Bt]{Linear Lower Bound when Removing the Assumption of Independence between $S_t$ and $B_t$}
\label{app:linear}
In this section, similarly as in Appendices~\ref{app:lower_bound_template} to \ref{app:proof_PAC_lower}, a randomized algorithm is given by a sequence of functions $\alpha_1 , \dots , \alpha_T$ with $\alpha_1 : [0,1] \mapsto [0,1]$ and for $t \ge 2$, $\alpha_t : [0,1]^t \times \{0,1\}^{2t-2} \mapsto [0,1]$. This randomized algorithm relies on random seeds $\xi_1,\dots,\xi_T$, uniformly distributed on $[0,1]$ and independent of $(S_t,B_t)_{t \in [T]}$.
The price $P_1$ posted by the algorithm is $\alpha_1(\xi_1)$. For $t \ge 2$, the price $P_t$ posted by the algorithm is 
\[
\alpha_t \lrb{\xi_1,\dots,\xi_t,\I\lcb{S_1 \le P_1 }, \I\lcb{P_1 \le B_1 },\dots,\I\lcb{S_{t-1} \le P_{t-1} }, \I\lcb{P_{t-1} \le B_{t-1} }}.
\]

\linearlowerbound*

The proof of this theorem uses a similar unidentifiable setting as in \cite[Theorem 5]{cesa2024bilateralMOR}, although there are many more technical details that need to be carefully addressed here.

\begin{proof}
 Fix a horizon $T\in\N$, a fairness level $\lambda\in[-\infty,0]$, and a randomized pricing policy $\alpha$.
	Let $ 0 \le \varepsilon \le \frac{1}{20} $ to be tuned later.
	For $j=1,2$,
	we let $(S_t,B_t)_{t \in [T]} \sim \cL^{j}(T)$ if and only if $(S_t,B_t)_{t \in [T]}$ are i.i.d. and, for all $t \in [T]$,
	if $j=1$,
	$(S_t,B_t)$ has density at $x \in [0,1]^2$ given by
	\begin{align*}
	&\frac{1}{3 \varepsilon^2}
	\I \lcb{ 
		x \in [0,\varepsilon] \times \lsb{\frac{3}{8},\frac{3}{8} + \varepsilon}
	}
	+
	\frac{1}{3 \varepsilon^2}
	\I \lcb{ 
		x \in \lsb{\frac{3}{8}-\varepsilon,\frac{3}{8}} \times [1 - \varepsilon,1]
	}
	\\
	&\quad +
	\frac{1}{3 \varepsilon^2}
	\I \lcb{ 
		x \in \lsb{\frac{5}{8}-\varepsilon,\frac{5}{8}} \times \lsb{\frac{5}{8},\frac{5}{8}+\varepsilon}
	}
	\end{align*}
	and if $j=2$, 
	$(S_t,B_t)$ has density at $x \in [0,1]^2$ given by
	\begin{align*}
	&\frac{1}{3 \varepsilon^2}
	\I \lcb{ 
		x \in [0,\varepsilon] \times \lsb{\frac{5}{8},\frac{5}{8} + \varepsilon}
	}
	+
	\frac{1}{3 \varepsilon^2}
	\I \lcb{ 
		x \in \lsb{\frac{3}{8}-\varepsilon,\frac{3}{8}} \times \lsb{\frac{3}{8},\frac{3}{8}+\varepsilon}
	}
	\\
	&\quad +
	\frac{1}{3 \varepsilon^2}
	\I \lcb{ 
		x \in \lsb{\frac{5}{8}-\varepsilon,\frac{5}{8}} \times [1 - \varepsilon,1]
	}.
	\end{align*}
	
	Use the notation, for $j \in \lcb{1,2}$,
	\[
	G_{\lambda}^j(p)
	=
	\E_{(S_t,B_t)_{t \in [T]} \sim \cL^{j}(T)} \bsb{
		\gft_{\lambda}(p,S_1,B_1)
	}.
	\]
	Write also $R_{\lambda}^{j}(\alpha,T)$ for the regret of the randomized pricing policy $\alpha$ when $(S_t,B_t)_{t \in [T]} \sim \cL^{j}(T)$, for $j=1,2$.

	Exactly as in \cite[Proof of Theorem 5]{cesa2024bilateralMOR}, one can check that for any fixed $p \in [0,1]$, for $u,v \in \{-1,1\}$,
	\begin{equation}
		\label{eq:unidentifiable}
		\mathbb{P}_{(S_t,B_t)_{t \in [T]} \sim \cL^{1}(T)}
		\lsb{
			uS_t \le up,
			vp \le vB_t 
		}
		=
		\mathbb{P}_{(S_t,B_t)_{t \in [T]} \sim \cL^{2}(T)}
		\lsb{
			uS_t \le up,
			vp \le vB_t 
		}.
	\end{equation}
	As a consequence,
	assume there is a fixed $A \subset [0,1]$ and
	a constant
	$b >0$  such that
	\begin{equation} \label{eq:unidentifiable:gap}
		\sup_{p \in A}
		G_{\lambda}^1(p) 
		\ge 
		\sup_{p \in [0,1] \backslash A}
		G_{\lambda}^1(p) 
		+ b
		~ ~ ~ ~
		\text{and}
		~ ~ ~ ~
		\sup_{p \in [0,1] \backslash A}
		G_{\lambda}^2(p) 
		\ge 
		\sup_{p \in A}
		G_{\lambda}^2(p) 
		+ b.
	\end{equation}
	Note that we have
	\begin{align*}
		& \E_{(S_t,B_t)_{t \in [T]} \sim \cL^{1}(T)} 
		\lsb{
			\sum_{t=1}^T
			\I \lcb{
				\alpha_t \lrb{ 
					(\xi_s)_{s \in [t]},
					\brb{\I\lcb{S_s \le P_s }, \I\lcb{P_s \le B_s }}_{s \in [t-1]}
				}
				\in [0,1] \backslash A
			}
		}
		\\
		= &
		\E_{(S_t,B_t)_{t \in [T]} \sim \cL^{2}(T)} 
		\lsb{
			\sum_{t=1}^T
			\I \lcb{
				\alpha_t \lrb{ 
					(\xi_s)_{s \in [t]},
					\brb{\I\lcb{S_s \le P_s }, \I\lcb{P_s \le B_s }}_{s \in [t-1]}
				}
				\in [0,1] \backslash A
			}
		},
	\end{align*}
	that can be simply shown by induction as a consequence of \eqref{eq:unidentifiable}. Hence, if the above expectations are at least $\frac{T}{2}$, then \eqref{eq:unidentifiable:gap} yields
	\begin{align*}
	R_{\lambda}^{1}(\alpha,T)
	&\ge 
	b \E_{(S_t,B_t)_{t \in [T]} \sim \cL^{1}(T)} 
	\lsb{
		\sum_{t=1}^T
		\I \lcb{
			\alpha_t \lrb{ 
				(\xi_s)_{s \in [t]},
				\brb{\I\lcb{S_s \le P_s }, \I\lcb{P_s \le B_s }}_{s \in [t-1]}
			}
			\in [0,1] \backslash A
		}
	}
	\\
	&\ge 
	\frac{bT}{2},
	\end{align*}
	while if these expectations are smaller than $\frac{T}{2}$, then 
	\begin{align*}
	R_{\lambda}^{2}(\alpha,T)
	&\ge 
	b \E_{(S_t,B_t)_{t \in [T]} \sim \cL^{2}(T)} 
	\lsb{
		\sum_{t=1}^T
		\I \lcb{
			\alpha_t \lrb{ 
				(\xi_s)_{s \in [t]},
				\brb{\I\lcb{S_s \le P_s }, \I\lcb{P_s \le B_s }}_{s \in [t-1]}
			}
			\in A
		}
	}
	\\
	&\ge 
	\frac{bT}{2}.
	\end{align*}
	Hence \eqref{eq:unidentifiable:gap} implies
	\[
	\max \brb{
		R_{\lambda}^{1}(\alpha,T),
		R_{\lambda}^{2}(\alpha,T)
	}
	\ge \frac{bT}{2},
	\]
	and it remains to prove \eqref{eq:unidentifiable:gap}. For that, we will exploit the following. When $(S_t,B_t)_{t \in [T]} \sim \cL^{1}(T)$, then one can check simply that $(1-B_t,1-S_t)_{t \in [T]} \sim \cL^{2}(T)$. As a consequence, for $p \in [0,1]$,
	\begin{align} \label{eq:symmetry:Glundeux}
		G_{\lambda}^2(1-p)
		&=
		\E_{(S_t,B_t)_{t \in [T]} \sim \cL^{2}(T)}  
		\bsb{
			\gft_{\lambda}
			\lrb{1-p, S_1 , B_1} 
		}
		\notag
		\\
		&=
		\E_{(S_t,B_t)_{t \in [T]} \sim \cL^{1}(T)}  
		\bsb{
			\gft_{\lambda}
			\lrb{1-p, 1-B_1 , 1-S_1} 
		}
		\notag
		\\
		&=
		\E_{(S_t,B_t)_{t \in [T]} \sim \cL^{1}(T)}  
		\bsb{
			M_{\lambda}
			\Brb{
				\brb{1-p - \lrb{ 1-B_1 }}_+
				,
				\brb{1-S_1 -\lrb{1-p}}_+
			}
		}
		\notag
		\\
		&=
		\E_{(S_t,B_t)_{t \in [T]} \sim \cL^{1}(T)}  
		\bsb{
			M_{\lambda}
			\Brb{
				\brb{B_1 - p}_+ 
				,
				\brb{p -S_1}_+
			}
		}
		\notag
		\\
		&= 
		G_{\lambda}^1(p).
	\end{align}

Let us set $A = \lsb{\frac{1}{2},1}$.
	Then we have, 
	\begin{align*}
		G_{\lambda}^1\lrb{\frac{11}{16}} 
		&\ge 
		\frac{1}{3} 
		\inf_{ \substack{ s \in [0 , \varepsilon] \\ 
				b \in [\frac{3}{8}, \frac{3}{8} + \varepsilon]
		}}
		\gft_{\lambda}
		\lrb{\frac{11}{16} , s , b}
		+
		\frac{1}{3} 
		\inf_{ \substack{ s \in [\frac{3}{8}- \varepsilon, \frac{3}{8} ] \\ 
				b \in [1- \varepsilon,1]
		}}
		\gft_{\lambda}
		\lrb{\frac{11}{16} , s , b
		}
		+ 
		\frac{1}{3} 
		\inf_{ \substack{ s \in [\frac{5}{8}-\varepsilon, \frac{5}{8} ] \\ 
				b \in [\frac{5}{8}, \frac{5}{8} + \varepsilon]
		}}
		\gft_{\lambda}
		\lrb{\frac{11}{16} , s , b} 
		\\
		&\ge
		\frac{1}{3} 
		M_{\lambda} 
		\lrb{
			\frac{5}{16} 
			,
			\frac{5}{16} - \varepsilon
		} \ge
		\frac{1}{3}
		\lrb{
			\frac{5}{16} - \varepsilon
		}.
	\end{align*}
	Next, for $p \in \lsb{0,\frac{3}{8} - \varepsilon}$,
	\begin{align*}
		G_{\lambda}^1\lrb{p}
		& \le 
		\frac{1}{3} 
		\sup_{ \substack{ s \in [0 , \varepsilon] \\ 
				b \in [\frac{3}{8}, \frac{3}{8} + \varepsilon]
		}}
		\gft_{\lambda}
		\lrb{p , s , b}
		+
		\frac{1}{3} 
		\sup_{ \substack{ s \in (\frac{3}{8}-\varepsilon, \frac{3}{8} ] \\ 
				b \in [1- \varepsilon,1]
		}}
		\gft_{\lambda}
		\lrb{p , s , b
		}
		+ 
		\frac{1}{3} 
		\sup_{ \substack{ s \in [\frac{5}{8}-\varepsilon, \frac{5}{8} ] \\ 
				b \in [\frac{5}{8}, \frac{5}{8} + \varepsilon]
		}}
		\gft_{\lambda}
		\lrb{p , s , b} 
		\\
		& \le 
		\frac{1}{3}
		\gft_{\lambda}
		\lrb{p, 0 , 
			\frac{3}{8}+ \varepsilon}
		 \le 
		\frac{1}{3}
		\lrb{
			\frac{3}{16}+ \frac{\varepsilon}{2}
		}.
	\end{align*}
	Also, for $p \in \lsb{\frac{3}{8} - \varepsilon,\frac{3}{8} + \varepsilon}$,
	\begin{align*}
		G_{\lambda}^1\lrb{p}
		& \le 
		\frac{1}{3} 
		\sup_{ \substack{ s \in [0 , \varepsilon] \\ 
				b \in [\frac{3}{8}, \frac{3}{8} + \varepsilon]
		}}
		\gft_{\lambda}
		\lrb{p , s , b}
		+
		\frac{1}{3} 
		\sup_{ \substack{ s \in [\frac{3}{8}-\varepsilon, \frac{3}{8} ] \\ 
				b \in [1- \varepsilon,1]
		}}
		\gft_{\lambda}
		\lrb{p , s , b
		}
		+ 
		\frac{1}{3} 
		\sup_{ \substack{ s \in [\frac{5}{8}-\varepsilon, \frac{5}{8} ] \\ 
				b \in [\frac{5}{8}, \frac{5}{8} + \varepsilon]
		}}
		\gft_{\lambda}
		\lrb{p , s , b} 
		\\
		& \le 
		\frac{1}{3}
		\gft_{\lambda}
		\lrb{p, 0 , 
			\frac{3}{8}+ \varepsilon}
		+
		\frac{1}{3}
		\gft_{\lambda}
		\lrb{p, \frac{3}{8}- \varepsilon , 
			1}
		\\
		& \le 
		\frac{1}{3}
		M_{\lambda}\lrb{2 \varepsilon,\frac{3}{8} - \varepsilon}
		+
		\frac{1}{3}
		M_{\lambda}\lrb{2 \varepsilon,\frac{5}{8} - \varepsilon}.
	\end{align*}
	Also, for $p \in \lsb{ \frac{3}{8} + \varepsilon , \frac{1}{2}}$,
	\begin{align*}
		G_{\lambda}^1\lrb{p}
		& \le 
		\frac{1}{3} 
		\sup_{ \substack{ s \in [0 , \varepsilon] \\ 
				b \in [\frac{3}{8}, \frac{3}{8} + \varepsilon)
		}}
		\gft_{\lambda}
		\lrb{p , s , b}
		+
		\frac{1}{3} 
		\sup_{ \substack{ s \in [\frac{3}{8}-\varepsilon, \frac{3}{8} ] \\ 
				b \in [1- \varepsilon,1]
		}}
		\gft_{\lambda}
		\lrb{p , s , b
		}
		+ 
		\frac{1}{3} 
		\sup_{ \substack{ s \in [\frac{5}{8}-\varepsilon, \frac{5}{8} ] \\ 
				b \in [\frac{5}{8}, \frac{5}{8} + \varepsilon]
		}}
		\gft_{\lambda}
		\lrb{p , s , b} 
		\\
		& \le 
		\frac{1}{3}
		\gft_{\lambda}
		\lrb{p , \frac{3}{8} - \varepsilon ,1}
		 \le 
		\frac{1}{3}
		M_{\lambda}\lrb{\frac{1}{8}+\varepsilon,\frac{1}{2}}.
	\end{align*}
	In the end, we have shown 
	\begin{align*}
		& G_{\lambda}^1
		\lrb{
			\frac{11}{16}
		}
		-
		\sup_{p \in \lsb{0,\frac{1}{2}}}
		G_{\lambda}^1
		\lrb{
			p
		}
		\\
		& \ge 
		\min \Bigg(
		\frac{5}{48}
		-
		\frac{3}{48}
		-
		\frac{\varepsilon}{2}
		,
		\frac{5}{48}
		- \frac{\varepsilon}{3}
		-
		\frac{1}{3}
		M_{\lambda}\lrb{2 \varepsilon,\frac{3}{8} - \varepsilon}
		-
		\frac{1}{3}
		M_{\lambda}\lrb{2 \varepsilon,\frac{5}{8} - \varepsilon},
		\frac{5}{48}
		- \frac{\varepsilon}{3}
		-
		\frac{1}{3}
		M_{\lambda}\lrb{\frac{1}{8}+\varepsilon,\frac{1}{2}}
		\Bigg)
		\\
		& \ge 
		\min \Bigg(
		\frac{2}{48}
		,
		\frac{5}{48}
		-
		\frac{1}{3}
		M_{\lambda}\lrb{0,\frac{3}{8}}
		-
		\frac{1}{3}
		M_{\lambda}\lrb{0,\frac{5}{8}},
		\frac{5}{48}
		-
		\frac{1}{3}
		M_{\lambda}\lrb{\frac{1}{8},\frac{1}{2}}
		\Bigg)
		\\
		& -
		\max \Bigg(
		\frac{\varepsilon}{2}
		,
		\frac{\varepsilon}{3}
		+ 
		\frac{1}{3}
		\lrb{
			M_{\lambda}\lrb{2 \varepsilon,\frac{3}{8} - \varepsilon}
			-
			M_{\lambda}\lrb{0,\frac{3}{8}}}
		+
		\frac{1}{3}
		\lrb{
			M_{\lambda}\lrb{2 \varepsilon,\frac{5}{8} - \varepsilon}
			-
			M_{\lambda}\lrb{0,\frac{5}{8}}
		},
		\\
		& \quad 
		\frac{\varepsilon}{3}
		+
		\frac{1}{3}
		\lrb{
			M_{\lambda}\lrb{\frac{1}{8}+\varepsilon,\frac{1}{2}}
			-
			M_{\lambda}\lrb{\frac{1}{8},\frac{1}{2}}
		}
		\Bigg).
	\end{align*}

	Hence, from \eqref{eq:symmetry:Glundeux}, we have
	\begin{align*}
		& G_{\lambda}^2
		\lrb{
			\frac{5}{16}
		}
		-
		\sup_{p \in \lsb{\frac{1}{2},1}}
		G_{\lambda}^2
		\lrb{
			p
		}
		\\
		& \ge 
		\min \Bigg(
		\frac{2}{48}
		,
		\frac{5}{48}
		-
		\frac{1}{3}
		M_{\lambda}\lrb{0,\frac{3}{8}}
		-
		\frac{1}{3}
		M_{\lambda}\lrb{0,\frac{5}{8}}
		,
		\frac{5}{48}
		-
		\frac{1}{3}
		M_{\lambda}\lrb{\frac{1}{8},\frac{1}{2}}
		\Bigg)
		\notag
		\\
		& -
		\max \Bigg(
		\frac{\varepsilon}{2}
		,
		\frac{\varepsilon}{3}
		+ 
		\frac{1}{3}
		\lrb{
			M_{\lambda}\lrb{2 \varepsilon,\frac{3}{8} - \varepsilon}
			-
			M_{\lambda}\lrb{0,\frac{3}{8}}}
		+
		\frac{1}{3}
		\lrb{
			M_{\lambda}\lrb{2 \varepsilon,\frac{5}{8} - \varepsilon}
			-
			M_{\lambda}\lrb{0,\frac{5}{8}}
		},
		\notag
		\\
		& \quad 
		\frac{\varepsilon}{3}
		+
		\frac{1}{3}
		\lrb{
			M_{\lambda}\lrb{\frac{1}{8}+\varepsilon,\frac{1}{2}}
			-
			M_{\lambda}\lrb{\frac{1}{8},\frac{1}{2}}
		}
		\Bigg).
	\end{align*}
	
Set
\[
\begin{aligned}
m_\lambda
&\ceq
\min \Bigg(
\frac{2}{48},
\frac{5}{48}
-
\frac{1}{3}
M_{\lambda}\lrb{0,\frac{3}{8}}
-
\frac{1}{3}
M_{\lambda}\lrb{0,\frac{5}{8}},
\frac{5}{48}
-
\frac{1}{3}
M_{\lambda}\lrb{\frac{1}{8},\frac{1}{2}}
\Bigg),
\\
\psi_\lambda(\varepsilon)
&\ceq
\max \Bigg(
\frac{\varepsilon}{2},
\\
&\qquad
\frac{\varepsilon}{3}
+
\frac{1}{3}
\lrb{
M_{\lambda}\lrb{2\varepsilon,\frac{3}{8}-\varepsilon}
-
M_{\lambda}\lrb{0,\frac{3}{8}}
}
+
\frac{1}{3}
\lrb{
M_{\lambda}\lrb{2\varepsilon,\frac{5}{8}-\varepsilon}
-
M_{\lambda}\lrb{0,\frac{5}{8}}
},
\\
&\qquad
\frac{\varepsilon}{3}
+
\frac{1}{3}
\lrb{
M_{\lambda}\lrb{\frac{1}{8}+\varepsilon,\frac{1}{2}}
-
M_{\lambda}\lrb{\frac{1}{8},\frac{1}{2}}
}
\Bigg).
\end{aligned}
\]
With this notation, the two previous displays imply
\[
G_{\lambda}^1\lrb{\frac{11}{16}}
-
\sup_{p\in\lsb{0,\frac12}}G_{\lambda}^1(p)
\ge
m_\lambda-\psi_\lambda(\varepsilon),
\]
and
\[
G_{\lambda}^2\lrb{\frac{5}{16}}
-
\sup_{p\in\lsb{\frac12,1}}G_{\lambda}^2(p)
\ge
m_\lambda-\psi_\lambda(\varepsilon).
\]

Observe first that, for every $\lambda\in[-\infty,0]$ and every $a\ge 0$,
\[
    M_\lambda(0,a)=0,
\]
and, by \Cref{lem:holder_basic_bounds},
\[
    M_\lambda(x,y)\le M_0(x,y)=\sqrt{xy}.
\]
Hence
\[
\begin{aligned}
m_\lambda
&\ge
\min\lrb{
\frac{1}{24},
\frac{5}{48},
\frac{5}{48}
-
\frac{1}{3}\sqrt{\frac{1}{8}\cdot\frac{1}{2}}
}
=
\frac{1}{48}.
\end{aligned}
\]

We now choose explicitly
\[
    \bar\eps\ceq 10^{-4}.
\]
Clearly $\bar\eps\le 1/20$.
We claim that, for every $\lambda\in[-\infty,0]$,
\[
    \psi_\lambda(\bar\eps)\le \frac{1}{96}.
\]
Indeed, the first term in the definition of $\psi_\lambda$ is bounded by
\[
    \frac{\bar\eps}{2}
    <
    \frac{1}{96}.
\]
For the second term, using $M_\lambda\le M_0$ and $M_\lambda(0,a)=0$, we get
\[
\begin{aligned}
\frac{\bar\eps}{3}
+
\frac{1}{3}
M_{\lambda}\lrb{2\bar\eps,\frac{3}{8}-\bar\eps}
+
\frac{1}{3}
M_{\lambda}\lrb{2\bar\eps,\frac{5}{8}-\bar\eps}
&\le
\frac{\bar\eps}{3}
+
\frac{1}{3}\sqrt{2\bar\eps\cdot\frac{3}{8}}
+
\frac{1}{3}\sqrt{2\bar\eps\cdot\frac{5}{8}}
\\
&<
\frac{\bar\eps}{3}
+
\frac{2}{3}\sqrt{\bar\eps}
=
\frac{1}{30000}
+
\frac{1}{150}
<
\frac{1}{96}.
\end{aligned}
\]
For the third term, we use the following elementary bound.
For every $\lambda\in[-\infty,0]$, every $x\in[1/8,1/8+\bar\eps]$, and $y=1/2$,
\[
    \partial_x M_\lambda(x,y)\le 2
\]
whenever the derivative exists, and the same increment bound is immediate in the Rawls case $\lambda=-\infty$.
Indeed, writing $\lambda=-\mu$ with $\mu>0$ and $r=y/x\in[1,4]$, \Cref{lem:holder_differential_identities} gives
\[
    \partial_x M_{-\mu}(x,y)
    =
    \frac{1}{2}
    \lrb{
        \frac{1+r^{-\mu}}{2}
    }^{-(1+\mu)/\mu}.
\]
If $\mu\ge 1$, then $\lrb{1+r^{-\mu}}/2\ge 1/2$ and $(1+\mu)/\mu\le 2$, so the derivative is at most $2$.
If $0<\mu\le 1$, then by the arithmetic-geometric mean inequality and $r\le 4$,
\[
    \frac{1+r^{-\mu}}{2}
    \ge
    r^{-\mu/2}
    \ge
    2^{-\mu},
\]
and therefore the derivative is again at most $2$.
The case $\lambda=0$ follows by taking the limit, or directly from $M_0(x,y)=\sqrt{xy}$.
Consequently,
\[
    M_\lambda\lrb{\frac{1}{8}+\bar\eps,\frac{1}{2}}
    -
    M_\lambda\lrb{\frac{1}{8},\frac{1}{2}}
    \le
    2\bar\eps,
\]
and hence the third term in $\psi_\lambda(\bar\eps)$ is bounded by
\[
    \frac{\bar\eps}{3}
    +
    \frac{2\bar\eps}{3}
    =
    \bar\eps
    <
    \frac{1}{96}.
\]
Thus $\psi_\lambda(\bar\eps)\le 1/96$.

We now fix $\eps=\bar\eps$ in the construction.
Combining the previous bounds gives
\[
G_{\lambda}^1
\lrb{
    \frac{11}{16}
}
-
\sup_{p \in \lsb{0,\frac{1}{2}}}
G_{\lambda}^1(p)
\ge
m_\lambda-\psi_\lambda(\bar\eps)
\ge
\frac{1}{48}-\frac{1}{96}
=
\frac{1}{96},
\]
and, by the symmetry argument,
\[
G_{\lambda}^2
\lrb{
    \frac{5}{16}
}
-
\sup_{p \in \lsb{\frac{1}{2},1}}
G_{\lambda}^2(p)
\ge
\frac{1}{96}.
\]
Therefore \eqref{eq:unidentifiable:gap} holds with
\[
    b=\frac{1}{96}.
\]
The indistinguishability argument above then yields
\[
    \max\brb{
        R_\lambda^1(\alpha,T),
        R_\lambda^2(\alpha,T)
    }
    \ge
    \frac{bT}{2}
    =
    \frac{T}{192}.
\]
Finally, since the density of each constructed instance is bounded by
\[
    \frac{1}{3\bar\eps^2}
    =
    \frac{10^8}{3},
\]
the theorem holds with
\[
    L\coloneqq \frac{10^8}{3},
    \qquad
    c_{\rm lin} \ceq \frac{1}{192}.
\]
\end{proof}


\section{Useful Probabilistic Facts}
\label{app:useful_results}

\begin{lemma}[McDiarmid's Bounded-Differences Inequality]
\label{lem:mcdiarmid_boundeddifferences}
Let $Z_1,\dots,Z_K$ be independent random variables taking values in measurable spaces $E_1,\dots,E_K$.
Let
$
    f\colon E_1\times\cdots\times E_K\to\bbR
$
be measurable.
Assume that there exist constants $c_1,\dots,c_K\ge0$ such that, for every $r\in\lcb{1,\dots,K}$ and every pair of points $z,z'\in E_1\times\cdots\times E_K$ that differ only in the $r$-th coordinate,
\[
    \labs{f(z)-f(z')}
    \le
    c_r.
\]
Then, for every $t>0$,
\[
    \Pb\Bsb{
        \labs{
            f(Z_1,\dots,Z_K)
            -
            \E\bsb{f(Z_1,\dots,Z_K)}
        }
        \ge t
    }
    \le
    2\exp\lrb{
        -\frac{2t^2}{\sum_{r=1}^K c_r^2}
    },
\]
with the convention that $\exp(-t^2/0)=0$.
\end{lemma}

This is the classical bounded-differences inequality of McDiarmid \citep{mcdiarmid1989bounded}; see also \citet[Chapter~6]{boucheron2013concentration}.


\bibliographystyle{plainnat}
\bibliography{biblio}

\end{document}